\documentclass{article}

\PassOptionsToPackage{compress}{natbib}
\usepackage[accepted]{icml2018}

\usepackage{tikz}
\usetikzlibrary{fit}					
\usetikzlibrary{backgrounds}	

\usepackage{url}
\usepackage{amsmath}
\usepackage{amssymb}
\usepackage{amsbsy}
\usepackage{siunitx}
\usepackage{mathtools}
\usepackage{array}
\usepackage{setspace}
\usepackage{caption}
\usepackage{subcaption}
\usepackage{booktabs}
\usepackage{adjustbox}
\usepackage{microtype}

\usepackage{todonotes}

\usepackage{wrapfig}
\usepackage{tabularx}

\usepackage{times}
\usepackage[bottom]{footmisc}
\usepackage{listings}
\usepackage{color}
\usepackage{xcolor}
\usepackage{textcomp}
\usepackage{xspace}

\usepackage{amsthm}

\usepackage{bbm}

\usepackage{thmtools}
\usepackage{thm-restate}
\usepackage{cancel}

\usepackage{todonotes}

\theoremstyle{plain}

\theoremstyle{plain}
\newtheorem{theoremApp}{Theorem}

\theoremstyle{remark}

\theoremstyle{lemma}
\newtheorem{lemma}{Lemma}

\theoremstyle{corollary}

\usepackage{paralist}

\usepackage{thmtools}
\usepackage{thm-restate}

\usepackage{multicol}

\usepackage{wrapfig}

\newcommand{\E}{\mathbb{E}}

\definecolor{darkgreenClj}{rgb}{0.25,.5,0.25}
\definecolor{blueClj}{rgb}{0,0.33,0.66}
\definecolor{redClj}{rgb}{0.66,0.0,0.0}
\definecolor{purpleClj}{rgb}{0.33,0,0.66}
\definecolor{cyanClj}{rgb}{0.0,0.5,0.5}
\definecolor{orangeClj}{rgb}{0.75,0.35,0.0}
\definecolor{grayClj}{rgb}{0.4,0.4,0.4}
\lstnewenvironment{code}[2]{\lstset{caption=#1,label=#2}}{}

\usepackage{abbreviations}

\usepackage{titlesec}
\titlespacing\section{0pt}{4pt plus 2pt minus 2pt}{0pt plus 2pt minus 0pt}
\titlespacing\subsection{0pt}{4pt plus 2pt minus 2pt}{0pt plus 2pt minus 0pt}
\titlespacing\subsubsection{0pt}{4pt plus 2pt minus 2pt}{0pt plus 2pt minus 0pt}

\usepackage[acronym,smallcaps,nowarn,section,nogroupskip,nonumberlist]{glossaries}
\newacronym{VAE}{vae}{variational auto-encoder}
\newacronym{AESMC}{aesmc}{auto-encoding sequential Monte Carlo}
\newacronym{IS}{is}{importance sampling}
\newacronym{IWAE}{iwae}{importance-weighted auto-encoder}
\newacronym{PIWAE}{piwae}{partially importance-weighted auto-encoder}
\newacronym{MIWAE}{miwae}{multiply importance-weighted auto-encoder}
\newacronym{CIWAE}{ciwae}{combination importance-weighted auto-encoder}
\newacronym{SMC}{smc}{sequential Monte Carlo}
\newacronym{SSM}{ssm}{state-space model}
\newacronym{SGA}{sga}{stochastic gradient ascent}
\newacronym{SGD}{sgd}{stochastic gradient descent}
\newacronym{ELBO}{elbo}{evidence lower bound}
\newacronym{KL}{kl}{Kullback-Leibler}
\newacronym{LSTM}{lstm}{long short-term memory}
\newacronym{AD}{ad}{automatic differentiation}
\newacronym{SPSA}{spsa}{simultaneous perturbation stochastic approximation}
\newacronym{CG-SPSA}{cg-spsa}{computational graph SPSA}
\newacronym{MML}{mml}{maximum marginal likelihood}
\newacronym{REINFORCE}{reinforce}{REINFORCE}
\glsunset{REINFORCE}
\newacronym{ADAM}{adam}{ADAM}
\glsunset{ADAM}
\newacronym{GRU}{gru}{gated recurrent unit}
\newacronym{MLP}{mlp}{multilayer perceptron}
\newacronym{MAP}{map}{maximum a-posteriori}
\newacronym{KDE}{kde}{kernel density estimation}
\newacronym{EM}{em}{expectation maximization}
\newacronym{MC}{mc}{Monte Carlo}
\newacronym{ALT}{alt}{alternating \textsc{elbo}s}
\newacronym{SNR}{snr}{signal-to-noise ratio}
\newacronym{VRNN}{vrnn}{Variational Recurrent Neural Network}
\newacronym{LGSSM}{lgssm}{linear Gaussian state space model}
\newacronym[firstplural=recurrent neural networks, plural=RNNs]{RNN}{rnn}{recurrent neural network}
\newacronym{MCM}{mcmc}{Markov Chain Monte Carlo}
\newacronym{RMSE}{rmse}{root mean squared error}

\newcommand{\given}{\lvert}

\DeclareMathOperator{\ELBO}{\acrshort{ELBO}}
\DeclareMathOperator{\SNR}{\acrshort{SNR}}

\newcommand{\KL}[2]{\acrshort{KL}\left(#1 \middle| \middle| #2\right)}

\usepackage{fancyhdr}

\icmltitlerunning{Tighter Variational Bounds are Not Necessarily Better}

\begin{document}
%

\twocolumn[
\icmltitle{Tighter Variational Bounds are Not Necessarily Better}



\icmlsetsymbol{equal}{*}

\begin{icmlauthorlist}
	\icmlauthor{Tom Rainforth}{stats}
\icmlauthor{Adam R. Kosiorek}{stats,eng}
\icmlauthor{Tuan Anh Le}{eng}
\icmlauthor{Chris J. Maddison}{stats}
\icmlauthor{Maximilian Igl}{eng}
\icmlauthor{Frank Wood}{ucb}
\icmlauthor{Yee Whye Teh}{stats}
\end{icmlauthorlist}

\icmlaffiliation{stats}{Department of Statistics, University of Oxford}
\icmlaffiliation{eng}{Department of Engineering, University of Oxford}
\icmlaffiliation{ucb}{Department of Computer Science, University of British Columbia}

\icmlcorrespondingauthor{Tom Rainforth}{rainforth@stats.ox.ac.uk}

\icmlkeywords{Variational inference, variational autoencoders, importance weighted autoencoders}

\vskip 0.3in
]
\printAffiliationsAndNotice{}	

\setlength{\abovedisplayskip}{2.5pt}
\setlength{\belowdisplayskip}{2.5pt}
\setlength{\abovedisplayshortskip}{2.5pt}
\setlength{\belowdisplayshortskip}{2.5pt}
%

\begin{abstract}
\vspace{2pt}
We provide theoretical and empirical evidence that using tighter \glspl{ELBO}
can be
detrimental to the process of learning an inference network by reducing the 
signal-to-noise ratio of the gradient estimator.  Our results call into question common 
implicit assumptions that tighter \glspl{ELBO} are better variational objectives for 
simultaneous model learning and inference amortization schemes.
Based on our insights, we introduce three new algorithms:  the partially importance
weighted auto-encoder (\textsc{piwae}), the multiply
importance weighted auto-encoder (\textsc{miwae}), and the combination importance weighted
auto-encoder (\textsc{ciwae}), each of which includes the standard importance
weighted auto-encoder (\textsc{iwae}) as a special case.  We show that each
can deliver improvements over \textsc{iwae}, even when performance is measured
by the \textsc{iwae} target itself. Furthermore, our results suggest that \textsc{piwae} 
may be able to deliver simultaneous improvements in the training of both
the inference and generative networks.
\end{abstract}



\section{Introduction}
\label{sec:intro}
Variational bounds provide tractable and state-of-the-art objectives for training deep generative models \citep{kingma2014auto,rezende2014stochastic}.  Typically taking
the form of a lower bound on the intractable model evidence, 
they provide surrogate targets that are more amenable to optimization.
In general, this optimization
requires the generation of approximate posterior samples during the model
training and so a number of methods simultaneously learn an \emph{inference
network} alongside the target \emph{generative network}. 

As well as assisting the training process, this inference network is often also of
direct interest itself.  For example, variational bounds are often used to train 
auto-encoders~\citep{bourlard1988auto,hinton1994autoencoders,gregor2016towards,chen2016variational},
for which the inference network forms the encoder.
Variational bounds are also used in amortized and traditional Bayesian inference 
 contexts~\cite{hoffman2013stochastic,ranganath2014black,
 	paige2016inference,le2017inference}, for which the generative model
 is fixed and the inference network is the primary target for the training.

The performance of variational approaches depends upon the choice of
evidence lower bound (\gls{ELBO})
and the formulation of the inference network, with the two often intricately linked to one another;
if the inference network formulation is not sufficiently expressive, this can have 
a knock-on effect on the generative network~\citep{burda2016importance}.  
In choosing the \gls{ELBO}, it is often implicitly
assumed that using tighter \glspl{ELBO} is universally beneficial,
at least whenever this does not in turn lead to higher variance 
gradient estimates.

In this work we question this implicit assumption
by demonstrating that,
although using a tighter \gls{ELBO} is typically beneficial to gradient 
updates of the 
generative network, it can be detrimental to updates of
 the inference network.
Remarkably, we find that it is possible to simultaneously tighten the bound,
reduce the variance of the gradient updates, and \emph{arbitrarily deteriorate} the
training of the inference network.

Specifically, we present theoretical and empirical evidence that increasing the
number of importance sampling particles, $K$, to tighten the bound in the \gls{IWAE}~\citep{burda2016importance}, degrades the \gls{SNR} of the
gradient estimates for the inference network, inevitably deteriorating the overall learning process.  In short, this behavior manifests because even though increasing
$K$ decreases the standard deviation of the gradient estimates, it decreases
the magnitude of the true gradient faster, such that the \emph{relative variance increases}.
%

Our results suggest that it may be best to use distinct objectives for learning
the generative and inference networks, or that when using the same target, it should
take into account the needs of both networks.
Namely, while tighter bounds are typically better for training
the generative network, looser bounds are often preferable for
training the inference network.
Based on these insights, we introduce three new algorithms: the \gls{PIWAE}, the \gls{MIWAE}, and the \gls{CIWAE}. Each of these include \gls{IWAE} as a special case and are based on the same set of importance weights, 
but use these weights in different ways to ensure a higher SNR for the
inference network.  

We demonstrate that
our new algorithms can produce inference networks more closely representing the true posterior
than~\gls{IWAE}, while matching the
training of the generative network, or potentially even improving it in the case of~\gls{PIWAE}. 
Even when treating the \gls{IWAE} 
objective itself as the measure of performance, all our algorithms are able to demonstrate clear
improvements over \gls{IWAE}.

%

\section{Background and Notation}

Let $x$ be an $\mathcal{X}$-valued random variable defined via a process involving an unobserved $\mathcal{Z}$-valued random variable $z$ with joint density $p_{\theta}(x, z)$. Direct maximum likelihood estimation of $\theta$ is generally intractable if  $p_{\theta}(x, z)$ is a deep generative model due to the marginalization of $z$. A common strategy is to instead optimize a variational lower bound on $\log p_{\theta}(x)$, defined via an auxiliary inference model $q_{\phi}(z \given x)$:
\begin{align}
	\ELBO_{\text{VAE}} &(\theta, \phi, x) \coloneqq \int q_{\phi}(z \given x) \log \frac{p_{\theta}(x, z)} {q_{\phi}(z \given x)} \,\mathrm dz \nonumber \\
	&= \log p_{\theta}(x) - \mathrm{KL}(q_{\phi}(z \given x) || p_{\theta}(z \given x)). \label{eqn:intro/elbo_vae}
\end{align}
Typically, $q_{\phi}$ is parameterized by a neural network, for which the approach is known as the \gls{VAE}~\citep{kingma2014auto,rezende2014stochastic}. Optimization
is performed with \gls{SGA} using unbiased estimates of $\nabla_{\theta, \phi} \ELBO_{\text{VAE}}(\theta, \phi, x)$. If $q_{\phi}$ is reparameterizable, 
then given a reparameterized sample  $z \sim q_{\phi}(z  \given x)$, the gradients $\nabla_{\theta, \phi} (\log p_{\theta}(x, z) - \log q_{\phi}(z \given x))$ can be used for the 
optimization.

The \gls{VAE} objective places a harsh penalty on mismatch between $q_{\phi}(z \given x)$ and $p_{\theta}(z \given x)$; optimizing jointly in $\theta, \phi$ can confound improvements in $\log p_{\theta}(x)$ with reductions in the KL \citep{turner2011two}. Thus, research has looked to develop bounds that separate the tightness of the bound from the expressiveness of the class of $q_{\phi}$. For example, the \gls{IWAE} objectives~\citep{burda2016importance}, which we
denote as $\ELBO_{\text{IS}}(\theta, \phi, x)$, are a family of bounds defined by
\begin{align}
	Q_{\text{IS}}(z_{1:K} \given x) &\coloneqq \prod\nolimits_{k = 1}^K q_{\phi}(z_k \given x), \nonumber
	\\
	\hat Z_{\text{IS}}(z_{1:K}, x) &\coloneqq \frac{1}{K}\sum\nolimits_{k = 1}^K \frac{p_{\theta}(x, z_k)}{q_{\phi}(z_k \given x)}, \label{eq:background/q_is_z_is}\\
	\ELBO_{\text{IS}}(\theta, \phi, x) &\coloneqq \int Q_{\text{IS}}(z_{1:K} \given x) \log \hat Z_{\text{IS}}(z_{1:K}, x) \,\mathrm dz_{1:K}
	\nonumber
\end{align}
$\leq \log p_{\theta}(x)$.
The \gls{IWAE} objectives generalize the \gls{VAE} objective ($K=1$ corresponds to the \gls{VAE}) and the bounds become strictly tighter as $K$ increases \cite{burda2016importance}. When the family of $q_{\phi}$ contains the true posteriors, the global optimum parameters $\{\theta^*,\phi^*\}$ are independent of $K$, see e.g.~\cite{le2017auto}. Nonetheless, except for the most trivial models, it is not usually the case that $q_{\phi}$ contains the true posteriors, and \citet{burda2016importance} provide strong empirical evidence that setting
$K>1$ leads to significant empirical gains over the \gls{VAE} in terms of learning the
generative model. 

Optimizing tighter bounds is usually empirically associated with 
better models $p_{\theta}$ in 
terms of marginal likelihood on held out data.
Other related approaches extend this to \gls{SMC} \citep{maddison2017filtering, le2017auto,naesseth2017variational} or change the lower bound that is optimized to reduce the bias \citep{li2016renyi,bamler2017perturbative}.
A second, unrelated, approach is to tighten the bound by improving the expressiveness of $q_{\phi}$ \citep{salimans_markov_2015, tran_variational_2015, rezende_variational_2015, kingma2016improving, maaloe_auxiliary_2016, ranganath2016hierarchical}.
In this work, we focus on the former, algorithmic, approaches to tightening bounds.

\section{Assessing the Signal-to-Noise Ratio of the Gradient Estimators}
\label{sec:snr}
Because it is not feasible to analytically optimize any \gls{ELBO} in complex models,  the effectiveness of any particular choice of \gls{ELBO} is linked
to our ability to numerically solve the resulting optimization problem. This motivates us to examine 
the effect $K$ has on the variance and magnitude of the gradient estimates of \gls{IWAE} for the two networks. More generally, we study \gls{IWAE} gradient estimators constructed as the average of $M$ estimates, each built from $K$ independent particles. We present a result characterizing the asymptotic signal-to-noise ratio in $M$ and $K$. For the standard case of $M=1$, our result shows that the signal-to-noise ratio of the reparameterization gradients of the inference network for the \gls{IWAE} decreases with rate $O(1/\sqrt{K})$.

As estimating the $\ELBO$ requires a Monte Carlo estimation of an expectation over $z$, 
we have two sample sizes to tune for the estimate:
the number of samples $M$ used for Monte Carlo estimation of the \gls{ELBO} and the number of importance samples $K$ used
in the bound construction.  Here $M$ does not change the true value of
$\nabla_{\theta, \phi} \ELBO$, only our variance in estimating
it, while changing $K$ changes the \gls{ELBO} itself, with larger $K$ leading to tighter 
bounds~\citep{burda2016importance}.  Presuming that reparameterization
is possible, we can express our gradient estimate in the general form
\begin{align}
\label{eq:grad-est}
\Delta_{M,K} :=  &\frac{1}{M} \sum\nolimits_{m=1}^{M}
\nabla_{\theta, \phi} \log \frac{1}{K} \sum\nolimits_{k=1}^{K} w_{m,k}^{}, \\
\text{where} \, \, w_{m,k}^{} &= \frac{p_{\theta}(z_{m,k},x^{})}{q_{\phi}(z_{m,k} \given x^{})} \, \,
\text{and} \, \,z_{m,k} \iid q_{\phi}(z_{m,k} \given x^{}). \nonumber
\end{align}
Thus, for a fixed budget of $T = MK$ samples, we have a family of estimators with the cases $K=1$ and $M=1$
corresponding respectively to the \gls{VAE} and \gls{IWAE} objectives.
We will use ${\Delta}_{M,K} \left(\theta\right)$ to refer to gradient estimates with respect
to $\theta$ and ${\Delta}_{M,K} \left(\phi\right)$ for those with respect to
$\phi$. 

Variance is not always a good barometer for the effectiveness
of a gradient estimation scheme; estimators with small expected values need proportionally smaller variances to 
be estimated accurately. In the case of \gls{IWAE}, when changes in $K$ simultaneously affect both the variance and
expected value, the quality of the estimator for learning can actually \emph{worsen} as the variance decreases.
To see why, consider
the marginal likelihood estimates $\hat{Z}_{m,K}^{} =\sum_{k=1}^{K} w_{m,k}^{}$. Because these become exact (and thus independent of the proposal) as $K\to\infty$, it must be the case that 
$\lim_{K\rightarrow\infty} {\Delta}_{M,K} (\phi) = 0$. 
Thus as $K$ becomes large, the expected value of the gradient
must decrease along with its variance, such that the variance relative to the problem scaling
need not actually improve.

To investigate this formally, we introduce the signal-to-noise-ratio (\textsc{snr}),
defining it to be the absolute value of the expected estimate scaled by its standard deviation:
\begin{align}
\SNR_{M,K} (\theta) &= \left|\E \left[\Delta_{M,K} (\theta)\right]/
\sigma\left[\Delta_{M,K} (\theta)\right]\right|
\end{align}
where $\sigma [\cdot]$ denotes the standard deviation of a random variable.
The \gls{SNR} is defined separately on each dimension of the parameter vector and
similarly for $\SNR_{M,K} (\phi)$. It provides a measure of the relative accuracy of the gradient estimates. Though a high \gls{SNR}
does not always indicate a good \gls{SGA} scheme (as the target objective itself might be
poorly chosen), a low \gls{SNR} is always problematic as it indicates that
the gradient estimates are dominated by noise: if $\SNR\to0$ then the estimates
become completely random.  We are now ready to state our main
theoretical result:  
$\SNR_{M,K} (\theta) = O(\sqrt{MK})$ and
$\SNR_{M,K} (\phi) = O(\sqrt{M/K})$.

\begin{restatable}{theorem}{snrproof}
	\label{the:snr}
Assume that when $M=K=1$, the expected gradients; the variances of the gradients; and the 
first four moments of  $w_{1,1}$, $\nabla_{\theta} w_{1,1}$, and 
$\nabla_{\phi} w_{1,1}$ are all finite and the variances are
also non-zero.
Then the signal-to-noise ratios of the gradient estimates converge at the following rates
\begin{align}
&\SNR_{M,K} (\theta) = 
	\label{eq:snr_theta} \\
&\sqrt{M}\left|\frac{ \sqrt{K} \; 
	\nabla_{\theta} Z -\frac{1}{2Z\sqrt{K}}\nabla_{\theta} \left(\frac{\textnormal{Var} \left[w_{1,1}\right]}{Z^2}\right)+ O\left(\frac{1}{K^{3/2}}\right) }
{\sqrt{\E \left[w_{1,1}^2\left(\nabla_{\theta} \log w_{1,1}-\nabla_{\theta} \log Z\right)^2\right]} + O\left(\frac{1}{K}\right)}\right| \nonumber \\
\label{eq:snr_phi}
& \SNR_{M,K} (\phi) =\sqrt{M} \left|\frac{
	\nabla_{\phi} \textnormal{Var} \left[w_{1,1}\right] + O\left(\frac{1}{K}\right) }
{2 Z\sqrt{K} \; \sigma\left[\nabla_{\phi} w_{1,1}\right] +O\left(\frac{1}{\sqrt{K}}\right)}\right|
\end{align}
where $Z := p_{\theta}(x)$ is the true marginal likelihood.
\end{restatable}
\begin{proof}
    We give an intuitive demonstration
	of the result here and provide a formal proof in Appendix~\ref{sec:proof}.
	The effect of $M$ on the \gls{SNR} follows from
	using the law of large
	numbers on the random variable $\nabla_{\theta, \phi} \log \hat{Z}_{m,K}$.  Namely, the 
	overall expectation is independent of $M$ and
	the variance reduces at a rate $O(1/M)$.  
	The effect of $K$ is more complicated but
	is perhaps most easily seen by noting that~\cite{burda2016importance}
	\begin{align*}
	\nabla_{\theta, \phi} \log \hat{Z}_{m,K}
	= \sum_{k=1}^{K} \frac{w_{m,k}^{}}{\sum_{\ell=1}^{K} w_{m,k}^{}} 
	\nabla_{\theta, \phi} \log \left(w_{m,k}^{}\right),
	\end{align*}
	such that $\nabla_{\theta, \phi} \log \hat{Z}_{m,K}$ can be interpreted as a self-normalized importance sampling
	estimate.  We can, therefore, invoke the known result (see e.g.~\citet{hesterberg1988advances}) 
	that the bias of a self-normalized
	importance sampler converges at a rate $O(1/K)$ and the standard deviation at
	a rate $O(1/\sqrt{K})$.  We thus see that the \gls{SNR} converges at a rate $O((1/K)/(1/\sqrt{K}))
	=O(1/\sqrt{K})$ if the asymptotic gradient is $0$ and $O((1)/(1/\sqrt{K}))=O(\sqrt{K})$
	otherwise, giving the convergence rates in the $\phi$ and $\theta$ cases respectively.
	\vspace{-8pt}
\end{proof}

The implication of these rates is that increasing $M$ is monotonically beneficial to
the \gls{SNR} for both $\theta$ and $\phi$, but that increasing $K$ is beneficial to the
former and detrimental to the latter.  We emphasize that this means the \gls{SNR} for the \gls{IWAE} inference network
gets worse as we increase $K$: this is not just an opportunity cost
from the fact that we could have increased $M$ instead, increasing the total number of samples
used in the estimator actually worsens the \gls{SNR}!

\subsection{Asymptotic Direction}
\label{sec:dir}

An important point of note is that the dependence of the true inference
network gradients becomes independent of $K$ as $K$ becomes large.
Namely, because we have as an intermediary result from
deriving the \glspl{SNR} that
\begin{align}
\label{eq:expt}
\hspace{-3pt}
\E \left[\Delta_{M,K} (\phi)\right] = -\frac{\nabla_{\phi} \text{Var}\left[w_{1,1}\right]}{2K Z^2} 
+O\left(\frac{1}{K^2}\right),
\end{align}
we see that expected gradient points in the direction of
$-\nabla_{\phi} \text{Var}\left[w_{1,1}\right]$ as $K \to \infty$.  
This direction is rather interesting: it implies that as $K\to\infty$, the optimal $\phi$
is that which minimizes the variance of the weights.  This is well known to be 
the optimal importance sampling
distribution in terms of estimating the marginal likelihood~\citep{mcbook}.
Given that the role of the inference network during training is to estimate the
marginal likelihood, this is thus arguably exactly what we want to optimize for.  
As such, this result, which complements those
of~\cite{cremer2017reinterpreting}, suggests that increasing 
$K$ provides a preferable target in terms of the direction of the
true inference network gradients.  We thus see that there is a trade-off with the fact that increasing $K$ also diminishes
the \gls{SNR}, reducing the estimates to pure noise if $K$ is
set too high.   In the absence of other factors, there may thus be a ``sweet-spot'' for 
setting $K$.

\begin{figure*}[t]
	\centering
	\begin{subfigure}[b]{0.4\textwidth}
		\centering
		\includegraphics[width=\textwidth]{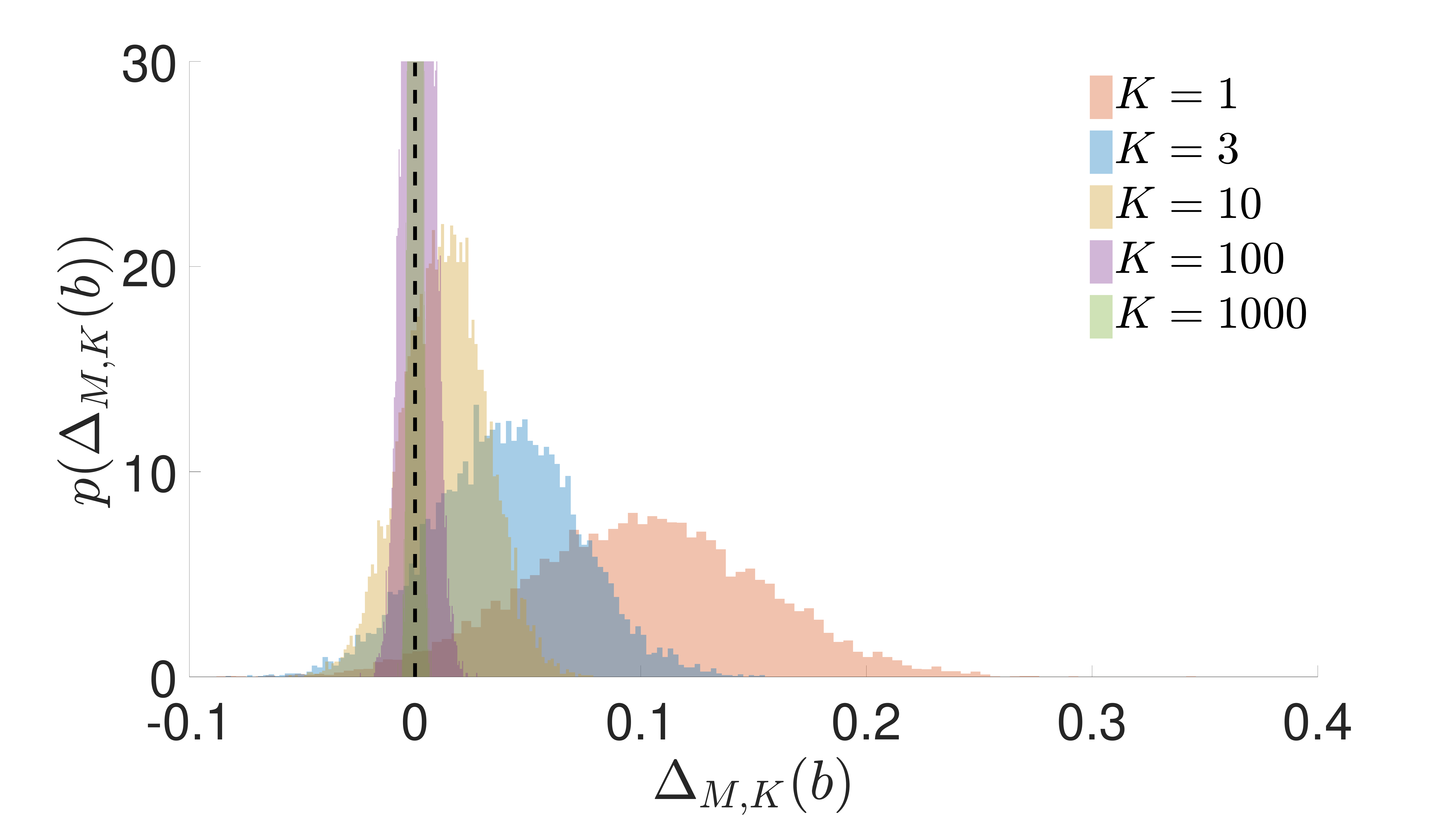} 
		\caption{ \gls{IWAE} inference network gradient estimates \label{fig:snr/b_hist_iwae}}
	\end{subfigure} ~~~~~~~~~~~~~
	\begin{subfigure}[b]{0.4\textwidth}
		\centering
		\includegraphics[width=\textwidth]{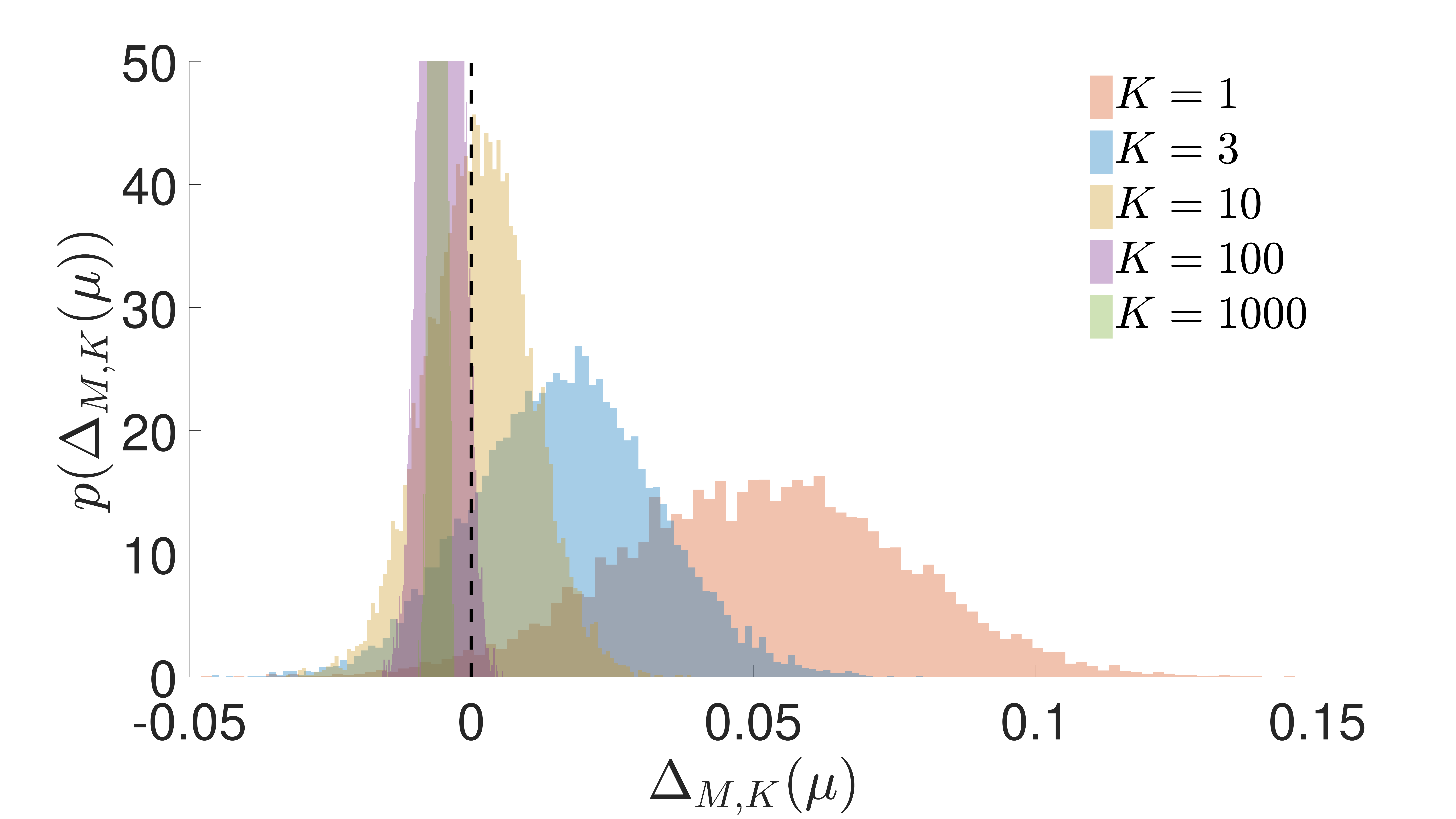} 
		\caption{ \gls{IWAE} generative network gradient estimates \label{fig:snr/mu_hist_iwae}}
	\end{subfigure}\vspace{-6pt}
	\caption{Histograms of gradient estimates $\Delta_{M,K}$ for the generative network and 
		the inference network using the \gls{IWAE} ($M=1$)
		objective with different values of $K$.
			\vspace{-14pt}
		\label{fig:snr/hists}}
\end{figure*}

\subsection{Multiple Data Points}
\label{sec:multi}


	Typically when training deep generative models, one does not optimize a single \gls{ELBO}
	but instead its average over multiple data points, i.e.
	\begin{align}
	\label{eq:J}
	\mathcal J(\theta, \phi) &:= 
	\frac{1}{N} \sum\nolimits_{n = 1}^N \ELBO_{\text{IS}} (\theta, \phi, x^{(n)}).
	\end{align}
	Our results extend to this setting because the $z$ are drawn independently for each $x^{(n)}$, so
	\begin{align}
	\hspace{-4pt}\E\left[\frac{1}{N} \sum\nolimits_{n=1}^{N} \Delta_{M,K}^{(n)}\right]\hspace{-2pt} &=\hspace{-2pt}\frac{1}{N}\sum\nolimits_{n=1}^{N}
	\E\left[\Delta_{M,K}^{(n)}\right], \displaybreak[0] \\
	\hspace{-4pt}\text{Var}\left[\frac{1}{N} \sum\nolimits_{n=1}^{N} \Delta_{M,K}^{(n)}\right]\hspace{-2pt} &=\hspace{-2pt}\frac{1}{N^2}\sum\nolimits_{n=1}^{N}
	\text{Var}\left[\Delta_{M,K}^{(n)}\right]\hspace{-2pt}.
	\end{align}
	We thus also see that if we are using mini-batches such that $N$ is a chosen parameter and the $x^{(n)}$ are
	drawn from the empirical data distribution, then 
	the \glspl{SNR} of $\bar{\Delta}_{N,M,K} := \frac{1}{N} \sum_{n=1}^{N} \Delta_{M,K}^{(n)}$ scales as $\sqrt{N}$, i.e.
$\SNR_{N,M,K} (\theta) = O(\sqrt{NMK})$ and $\SNR_{N,M,K} (\phi) 
= O(\sqrt{NM/K})$.  Therefore increasing $N$ has the same ubiquitous benefit as increasing $M$.
In the rest of the paper, we will implicitly be considering the \glspl{SNR} for $\bar{\Delta}_{N,M,K}$, but
will omit the dependency on $N$ to simplify the notation.

\section{Empirical Confirmation}
\label{sec:emp}

Our convergence results hold exactly in relation to $M$ (and $N$) but are only 
asymptotic in $K$ due to the higher order terms.  Therefore their applicability should be viewed with a 
healthy degree of skepticism in the small $K$ regime.
 With this in mind, we now present empirical support for our theoretical results 
and test how well they hold in the small $K$ regime
 using a simple Gaussian model, for which we can analytically calculate the ground truth.
 
 Consider a family of generative models with $\mathbb R^D$--valued latent variables $z$ and observed variables $x$:
\begin{align}
    z \sim \mathcal{N}(z;\mu, I), &&
    x \given z \sim \mathcal{N}(x; z, I),
\end{align}
which is parameterized by $\theta := \mu$.
Let the inference network be parameterized by $\phi = (A, b), \; A \in \mathbb R^{D \times D}, \; b \in \mathbb R^D$ where $q_{\phi}(z \given x) = \mathcal{N}(z; Ax + b, \frac{2}{3}I )$.
Given a dataset $(x^{(n)})_{n = 1}^N$, we can analytically calculate the optimum of our target $\mathcal J(\theta, \phi)$ as explained in Appendix~\ref{sec:optGauss},
giving $\theta^* := \mu^* = \frac{1}{N} \sum_{n = 1}^N x^{(n)}$ and 
$\phi^* := (A^*, b^*)$, where $A^* = I / 2$ and $b^* = \mu^* / 2$. Though this will not be the
case in general, for
this particular problem, the optimal proposal is independent of $K$. 
However, the expected gradients for the inference network still change with $K$.

\begin{figure*}[t]
	\centering
	\begin{subfigure}[b]{0.4\textwidth}
		\centering
		\includegraphics[width=\textwidth]{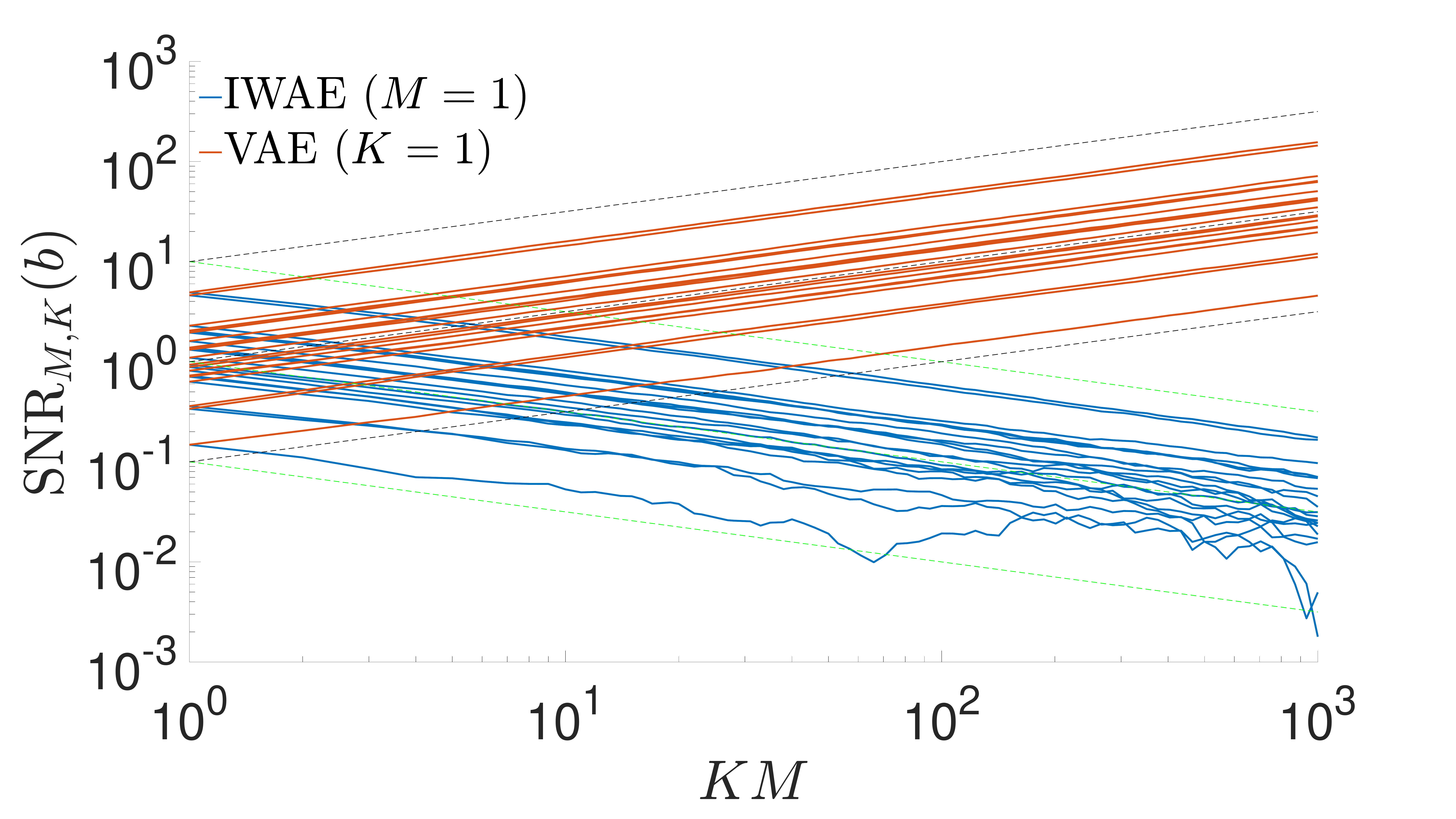}\vspace{-4pt}
		\caption{\vspace{-6pt} Convergence of \gls{SNR} for inference network \label{fig:snr/b}}
	\end{subfigure}
~~~~~~~~~~~~~~
	\begin{subfigure}[b]{0.4\textwidth}
		\centering
		\includegraphics[width=\textwidth]{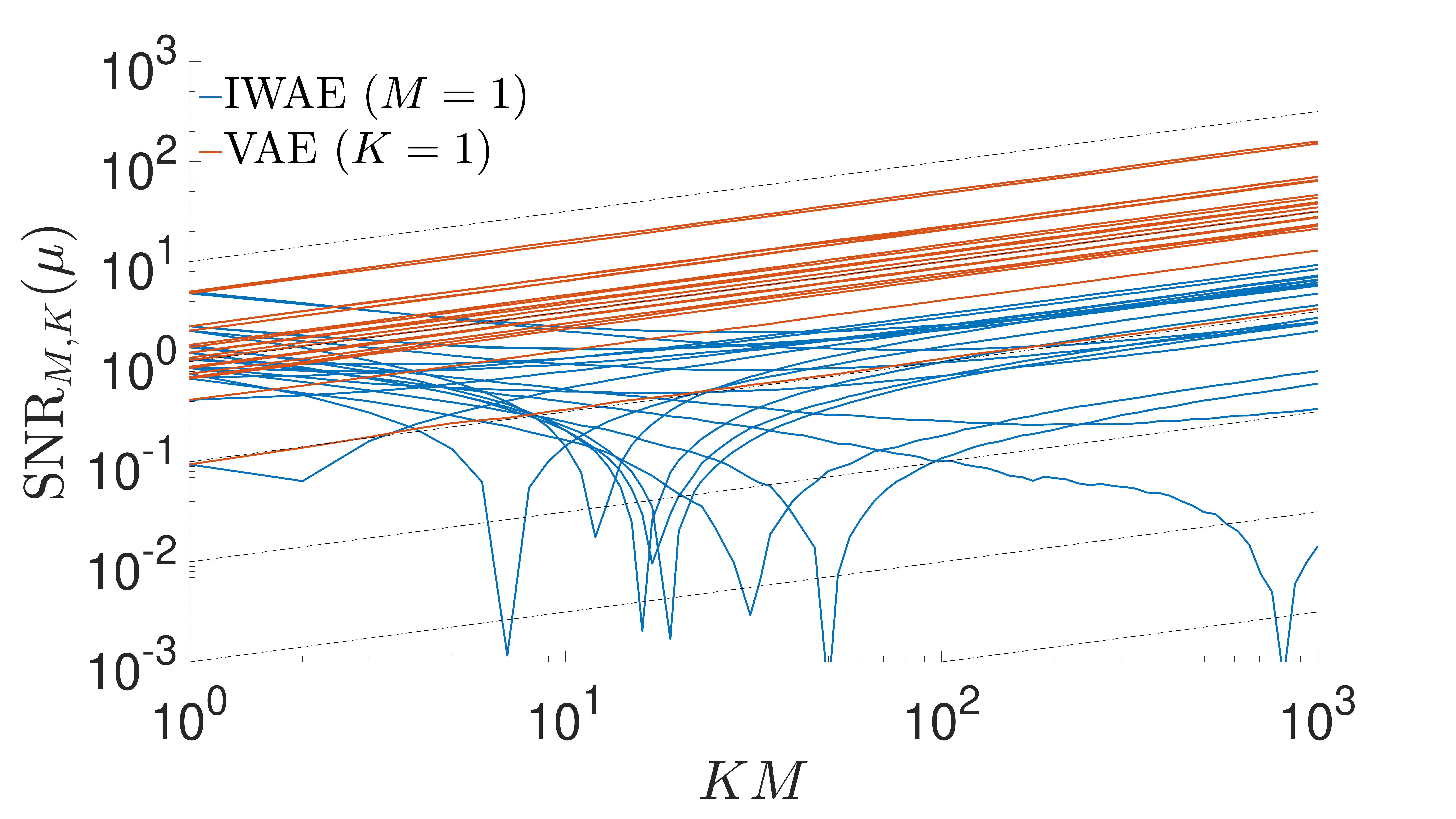}\vspace{-4pt}
		\caption{\vspace{-6pt} Convergence of \gls{SNR} for generative network\label{fig:snr/mu}}
	\end{subfigure}
	\caption{Convergence of signal-to-noise ratios of gradient estimates with increasing $M$ and $K$.
		Different lines correspond to different
		dimensions of the parameter vectors.
		Shown in blue is the \gls{IWAE} where we keep $M=1$ fixed and increase $K$.  
		Shown in red is the \gls{VAE} where $K=1$ is fixed and we increase $M$. 
		The black and green dashed lines show the expected convergence rates from our theoretical results, 
		representing gradients of $1/2$ and $-1/2$ respectively.  
		\vspace{-12pt}
		\label{fig:snr/K_conv}}
\end{figure*}

To conduct our investigation, we randomly generated a synthetic dataset from the model with $D=20$
dimensions, $N=1024$ data points, and a true model parameter value $\mu_{\text{true}}$ that was itself 
randomly generated from a unit Gaussian, i.e. $\mu_{\text{true}} \sim \mathcal{N}(\mu_{\text{true}} ;0,I)$.
We then considered the gradient at a random point in the parameter space close to optimum (we also 
	consider a point far from the optimum in 
	Appendix~\ref{sec:hv}). Namely
each dimension of each parameter was randomly offset from its optimum value using a zero-mean
Gaussian with standard deviation $0.01$.  We then calculated empirical estimates of the \gls{ELBO}
gradients for \gls{IWAE}, where $M=1$ is held fixed and we increase $K$, and
for \gls{VAE}, where $K=1$ is held fixed and we increase $M$.  In all cases we 
calculated $10^4$ such estimates and used these samples to provide empirical estimates for, amongst other things, the
mean and standard deviation of the estimator, and thereby an empirical estimate for the \gls{SNR}.

We start by examining the qualitative behavior of the different gradient estimators as $K$ increases as
shown in Figure~\ref{fig:snr/hists}.  This shows histograms of the \gls{IWAE}  gradient estimators
for a single parameter of the inference network (left) and generative network (right).
We first see in Figure~\ref{fig:snr/b_hist_iwae} that as $K$ increases, both the magnitude and the standard
deviation of the estimator decrease for the inference network, 
with the former decreasing faster.  This
matches the qualitative behavior of our theoretical result, with the \gls{SNR} ratio 
diminishing as $K$ increases.
In particular, the probability that the gradient is positive or negative 
becomes roughly equal for larger
values of $K$, meaning the optimizer is equally likely to increase as decrease the 
inference network parameters at the next iteration.
By contrast, for the generative network, \gls{IWAE} converges towards a 
non-zero gradient, such that,
even though the \gls{SNR} initially decreases with $K$, it then rises again, with a very
clear gradient signal for $K=1000$.

To provide a more rigorous analysis, we next directly examine the convergence of 
the \gls{SNR}.
Figure~\ref{fig:snr/K_conv} shows the convergence of the estimators with increasing $M$ and $K$.
The observed rates for the inference network (Figure~\ref{fig:snr/b})
correspond to our theoretical results, with
the suggested rates observed all the way back to $K=M=1$.  As expected, we see 
that as $M$ increases, so does $\SNR_{M,K}(b)$, but as $K$ increases, $\SNR_{M,K}(b)$ reduces.

In Figure~\ref{fig:snr/mu}, we see that
the theoretical convergence for $\SNR_{M,K}(\mu)$
is again observed exactly for variations in $M$, 
but a more unusual behavior is seen for variations in $K$,
where the \gls{SNR} initially decreases before starting to increase again for large enough $K$, 
eventually exhibiting behavior consistent with the theoretical result for large enough $K$.
The driving factor for this is that here
$\E [\Delta_{M,\infty}(\mu)]$  has a smaller magnitude than (and opposite sign to)
$\E [\Delta_{M,1}(\mu)]$  (see Figure~\ref{fig:snr/mu_hist_iwae}).  If we think of the estimators
for all values of $K$ as biased estimates for
$\E [\Delta_{M,\infty}(\mu)]$, we see from our theoretical results that this bias decreases faster 
than the standard deviation.  Consequently, while the magnitude of this bias remains large compared to
$\E [\Delta_{M,\infty}(\mu)]$, it is the predominant component in the true gradient and
we see similar \gls{SNR} behavior as in the inference network.

Note that this does not mean that the 
estimates are getting worse for the generative network.
As we increase $K$ our bound is getting tighter and our estimates closer to
the true gradient for the target that we 
actually want to optimize $\nabla_{\mu} \log Z$.  See 
Appendix~\ref{sec:app:rmse} for more details.
As we previously discussed, it is also the case that increasing $K$ could be beneficial for the inference network
even if it reduces the \gls{SNR} by improving the direction of the expected gradient.
However, as we will now show, the \gls{SNR} is, for this problem,
the dominant effect for the inference network.

\begin{figure*}[t]
	\centering
	\begin{subfigure}[b]{0.4\textwidth}
		\centering
		\includegraphics[width=\textwidth]{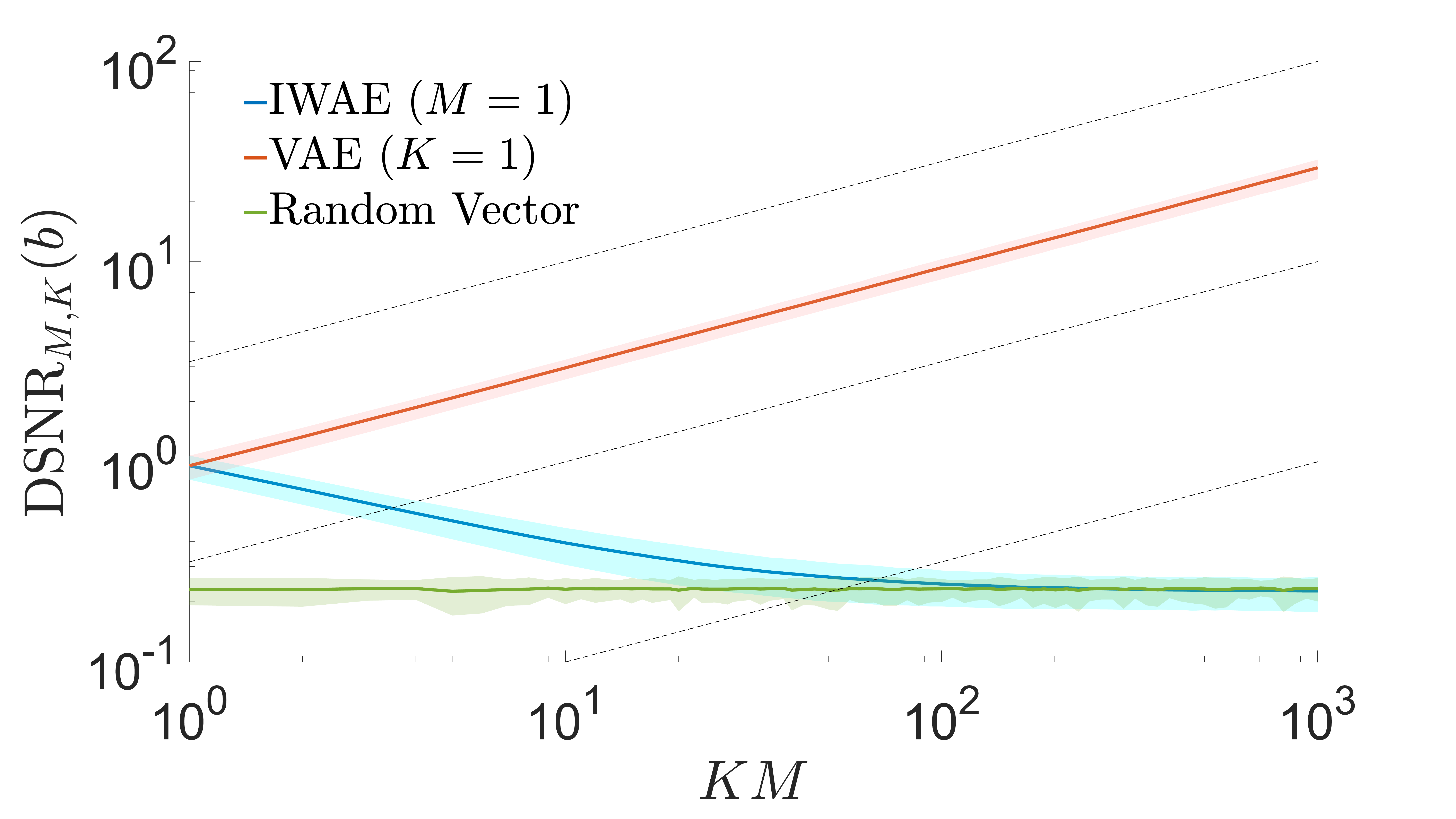}\vspace{-2pt}
		\caption{Convergence of \textsc{dsnr} for inference network\label{fig:snr/snr_dir}}
	\end{subfigure}
~~~~~~~~~~~~~~
	\begin{subfigure}[b]{0.4\textwidth}
		\centering
		\includegraphics[width=\textwidth]{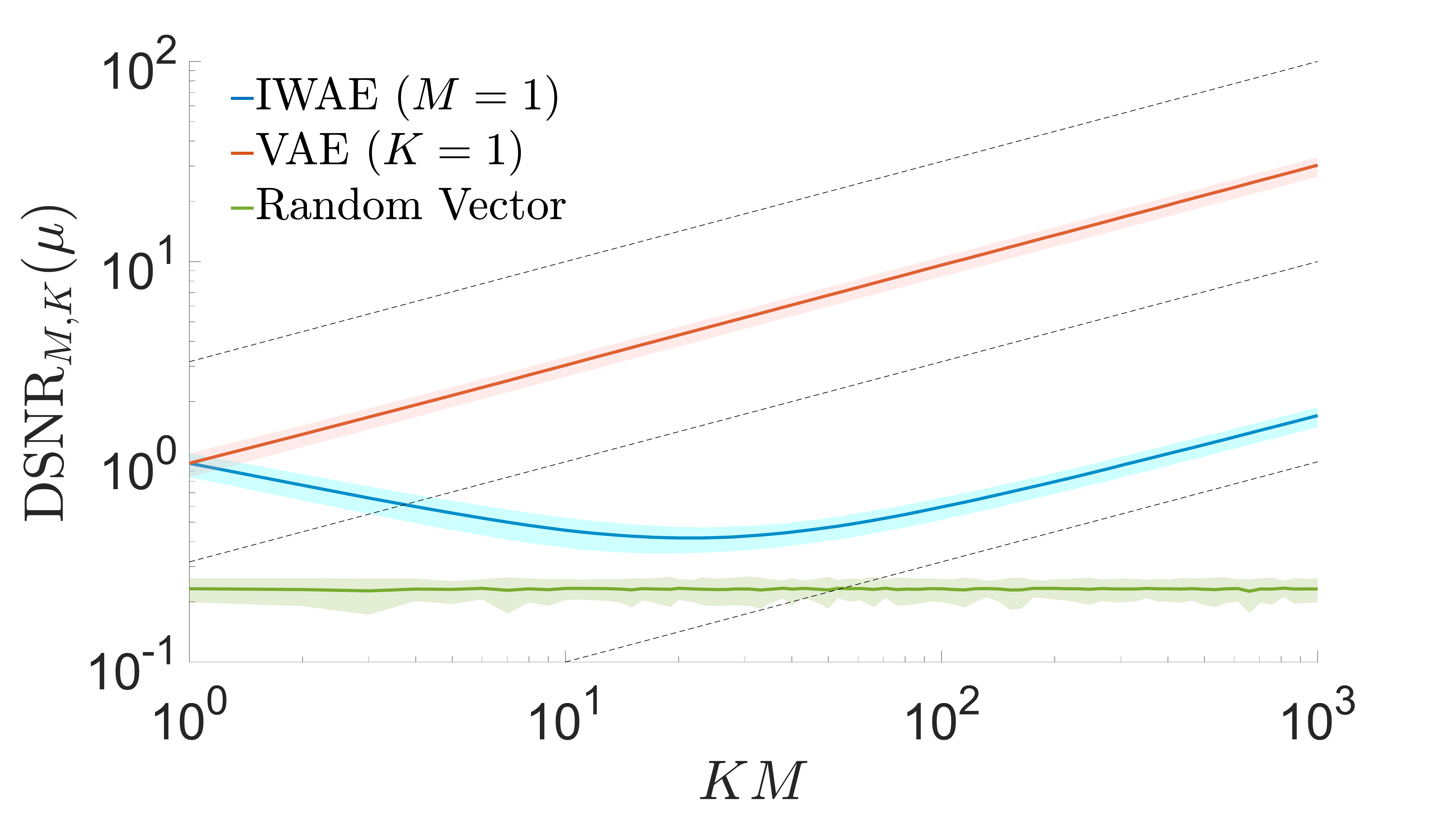}\vspace{-2pt}
		\caption{Convergence of \textsc{dsnr} for generative network\label{fig:snr/snr_dir_mu}}
	\end{subfigure}\vspace{-6pt}
	\caption{Convergence of the directional \gls{SNR} of gradients estimates with increasing $M$ and $K$.
		The solid lines show the estimated \textsc{dsnr} and the shaded regions the interquartile range of
		the individual ratios.  Also shown for reference is the \textsc{dsnr} for a randomly generated
		vector where each component is drawn from a unit Gaussian.
		\vspace{-10pt}
		\label{fig:snr/extra}}
\end{figure*}

\begin{figure*}[t]
	\centering
	\begin{subfigure}[b]{0.4\textwidth}
		\centering
		\includegraphics[width=\textwidth]{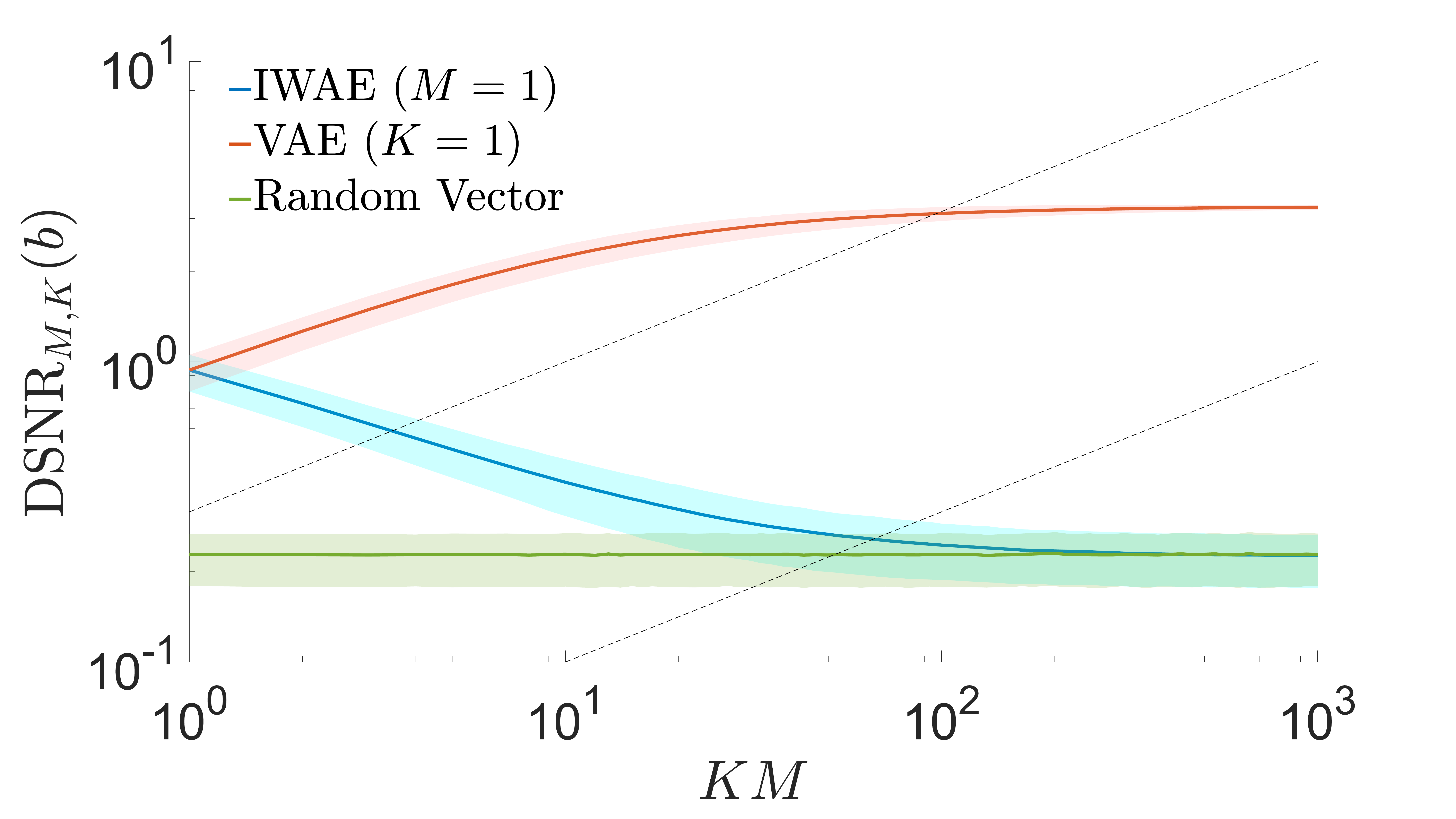}\vspace{-2pt}
		\caption{Convergence of \textsc{dsnr} for inference network\label{fig:snr/snr_dir_end}}
	\end{subfigure} ~~~~~~~~~~~~~~
	\begin{subfigure}[b]{0.4\textwidth}
		\centering
		\includegraphics[width=\textwidth]{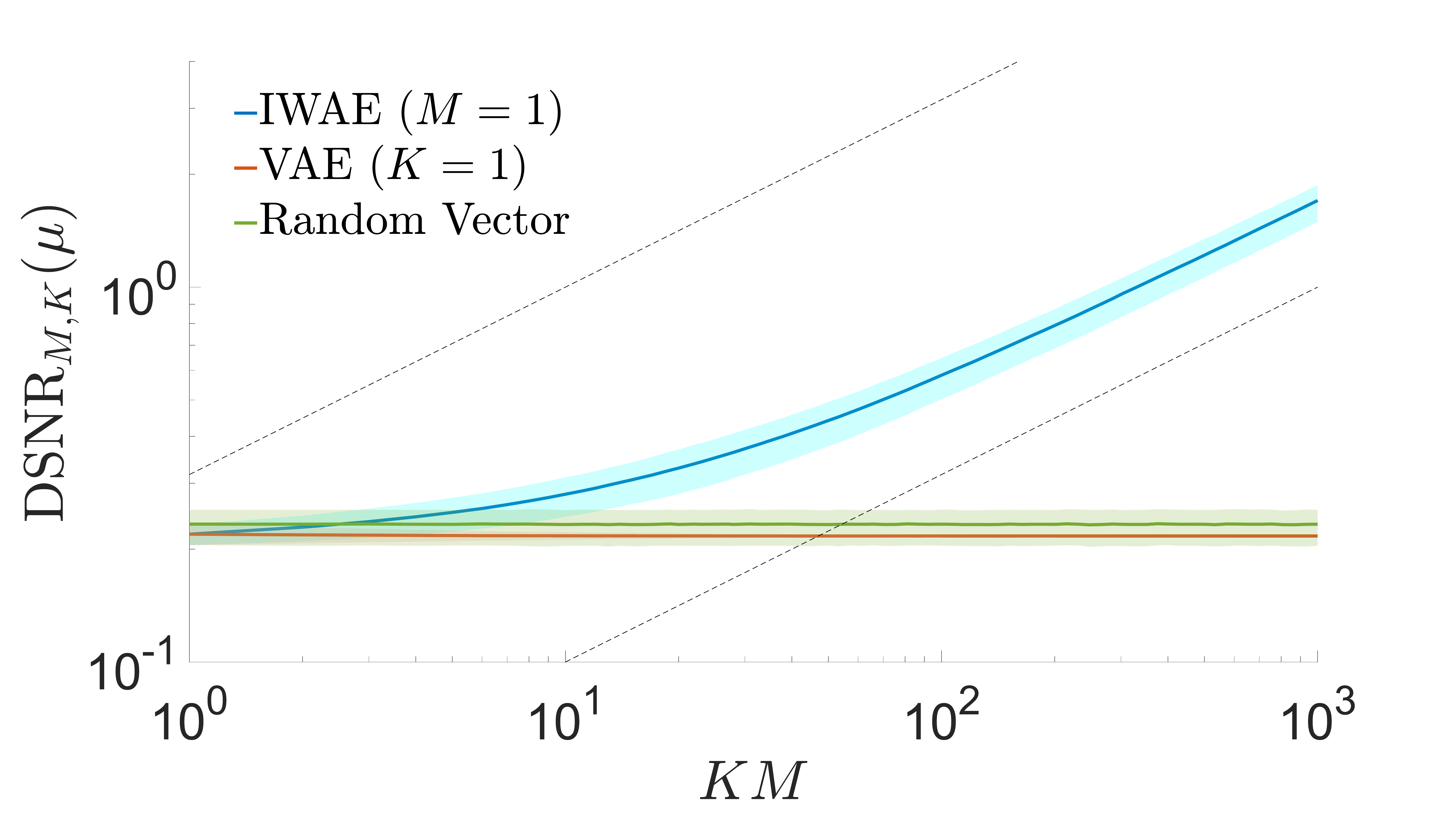}\vspace{-2pt}
		\caption{Convergence of \textsc{dsnr} for generative network\label{fig:snr/snr_dir_mu_end}}
	\end{subfigure}\vspace{-6pt}
	\caption{Convergence of the \textsc{dsnr} when the
		target gradient is taken as $u = \E \left[\Delta_{1,1000}\right]$.  Conventions as
		per Figure~\ref{fig:snr/extra}.
		\vspace{-14pt}
		\label{fig:snr/extra_end}}
\end{figure*}

\subsection{Directional Signal-to-Noise Ratio}

As a reassurance that our chosen definition of the \gls{SNR} is appropriate
for the problem at hand and to examine the effect of multiple dimensions explicitly,
 we now also consider an alternative definition of the \gls{SNR} that is similar (though distinct)
to that used in~\cite{roberts2009signal}.  We refer to this as the ``directional'' \gls{SNR} (\textsc{dsnr}).
At a high-level, we define the \textsc{dsnr} by splitting each gradient estimate into two component vectors, one parallel
to the true gradient and one perpendicular, then taking the expectation of ratio of their magnitudes.  More precisely,
we define $u =\E \left[\Delta_{M,K}\right]/\lVert\E \left[\Delta_{M,K}\right]\rVert_2$ as being
the true normalized gradient direction and then the \textsc{dsnr} as
\begin{align}
\textsc{dsnr}_{M,K} = & \,\,\E \left[\frac{\lVert\Delta_{\|}\rVert_2}{\lVert\Delta_{\bot}\rVert_2} \right]
\quad \text{where}  \displaybreak[0]\\
\Delta_{\|} = \left(\Delta_{M,K}^T u\right)u &\quad \text{and} \quad
\Delta_{\bot} = \Delta_{M,K}- \Delta_{\|}. \nonumber
\end{align}
The \textsc{dsnr} thus provides a measure of the expected proportion of the gradient that will point in the
true direction.  For perfect estimates of the gradients, then $\textsc{dsnr}\to\infty$, but unlike the
\gls{SNR}, arbitrarily bad estimates do not have $\textsc{dsnr}=0$ because even random vectors will have
a component of their gradient in the true direction.

The convergence of the \textsc{dsnr} is shown in Figure~\ref{fig:snr/extra}, for which the true normalized
gradient $u$ has been estimated empirically, noting that this varies with $K$.
We see a similar qualitative behavior
to the \gls{SNR}, with the gradients of \gls{IWAE} for the inference network degrading to having the
same directional accuracy as drawing a random vector.  Interestingly, the \textsc{dsnr} seems to be
following the same asymptotic convergence behavior as the \gls{SNR} for both
networks in $M$ (as shown by the dashed lines), even though we have no theoretical result to suggest this should occur.

As our theoretical results suggest that the direction of the true gradients correspond to
targeting an improved objective 
as $K$ increases, we now examine whether this or the changes in
the \gls{SNR} is the dominant effect.  To this end, we repeat our calculations for
the \textsc{dsnr} but take $u$ as the target direction of the gradient for $K=1000$. 
This provides a measure of how varying $M$ and $K$ affects the quality of the gradient directions
as biased estimators for $\E\left[\Delta_{1,1000}\right]
/\lVert\E\left[\Delta_{1,1000}\right]\rVert_2$.  As shown in Figure~\ref{fig:snr/extra_end}, increasing $K$ is still detrimental for the inference network
by this metric, even though
it brings the expected gradient estimate closer to the target gradient.  By contrast, increasing
$K$ is now monotonically beneficial for the generative network.  Increasing $M$
leads to initial improvements for the inference network before plateauing due to the bias 
of the estimator.  For the generative network, increasing $M$ has little impact, with the bias being
the dominant factor throughout.  Though this metric is not an absolute measure of
performance of the~\gls{SGA} scheme, e.g. because high bias may be more detrimental than high variance,
it is nonetheless a powerful result in suggesting that increasing $K$ can be detrimental
to learning the inference network.


\section{New Estimators}
\label{sec:algs}

Based on our theoretical results, we now introduce three new
algorithms that address the issue of diminishing \gls{SNR} for the inference
network.  Our first, \gls{MIWAE}, is exactly equivalent to the 
general formulation given in~\eqref{eq:grad-est}, the distinction from previous
approaches coming from the fact that it takes both $M>1$ and $K>1$.  
The motivation for this is that because our inference network \gls{SNR} increases as
$O(\sqrt{M/K})$, we should be able to mitigate the issues increasing
$K$ has on the \gls{SNR} by also increasing $M$.  For fairness, we will
keep our overall budget $T=MK$ fixed, but we will show that given this
budget, the optimal value for $M$ is often not $1$.  In practice, we expect that
it will often be beneficial to increase the mini-batch size $N$ rather than $M$ for
\gls{MIWAE};
as we showed in Section~\ref{sec:multi} this has the same effect on the~\gls{SNR}.  Nonetheless, \gls{MIWAE} forms an interesting reference method
for testing our theoretical results and, as we will show, it can offer improvements
over \gls{IWAE} for a given $N$.

Our second algorithm, \gls{CIWAE} uses a convex combination of the
\gls{IWAE} and \gls{VAE} bounds, namely
\begin{align}
\ELBO_{\text{CIWAE}} = \beta \ELBO_{\text{VAE}} + (1-\beta) \ELBO_{\text{IWAE}}
\end{align}
where $\beta \in [0,1]$ is a combination parameter.  It is trivial
to see that $\ELBO_{\text{CIWAE}}$ is a lower bound on the log marginal
that is tighter than the \gls{VAE} bound but looser than the \gls{IWAE}
bound.  We then employ the following estimator
\begin{align}
\label{eq:CIWAE-est}
\Delta_{K,\beta}^{\text{C}}
=& \\\nabla_{\theta, \phi} &\left(
\beta \frac{1}{K} \sum_{k=1}^{K} 
\log w_k + (1-\beta) \log \left(\frac{1}{K} \sum_{k=1}^{K}  w_k\right)\right) \nonumber
\end{align}
where we use the same $w_k$ for both terms.  The motivation
for  \gls{CIWAE} is that, if we set $\beta$ to a relatively small value, the
objective will behave mostly like \gls{IWAE}, except when the
expected \gls{IWAE} gradient becomes very small.  When this happens,
the \gls{VAE} component should ``take-over'' and alleviate SNR issues:
the asymptotic SNR of $\Delta_{K,\beta}^{\text{C}}$ for $\phi$ is
$O(\sqrt{MK})$ because the \gls{VAE} component has non-zero expectation
in the limit $K\rightarrow\infty$.

Our results suggest that what is good for the generative
network, in terms of setting $K$, is often detrimental for the inference 
network.  It is therefore natural to question whether it is sensible
to always use the same target for both the inference and generative networks.
Motivated by this,
our third method, \gls{PIWAE}, uses the \gls{IWAE} target when
training the generative network, but the \gls{MIWAE} target for training
the inference network.  We thus have
\begin{subequations}
\begin{align}
\label{eq:PIWAE-est}
\Delta_{K,\beta}^{\text{C}} (\theta) &= \nabla_{\theta} \log \frac{1}{K}
\sum\nolimits_{k=1}^{K} w_k \\
\Delta_{M,K,\beta}^{\text{C}} (\phi) &= \frac{1}{M} \sum\nolimits_{m=1}^{M}
\nabla_{\phi} \log \frac{1}{L} \sum\nolimits_{\ell=1}^{L} w_{m,\ell}
\end{align}
\end{subequations}
where we will generally set $K=ML$ so that the same weights can be
used for both gradients.  

\begin{figure*}[t!]
	\centering
	\begin{subfigure}[b]{0.33\textwidth}
		\centering
		\includegraphics[width=\textwidth]{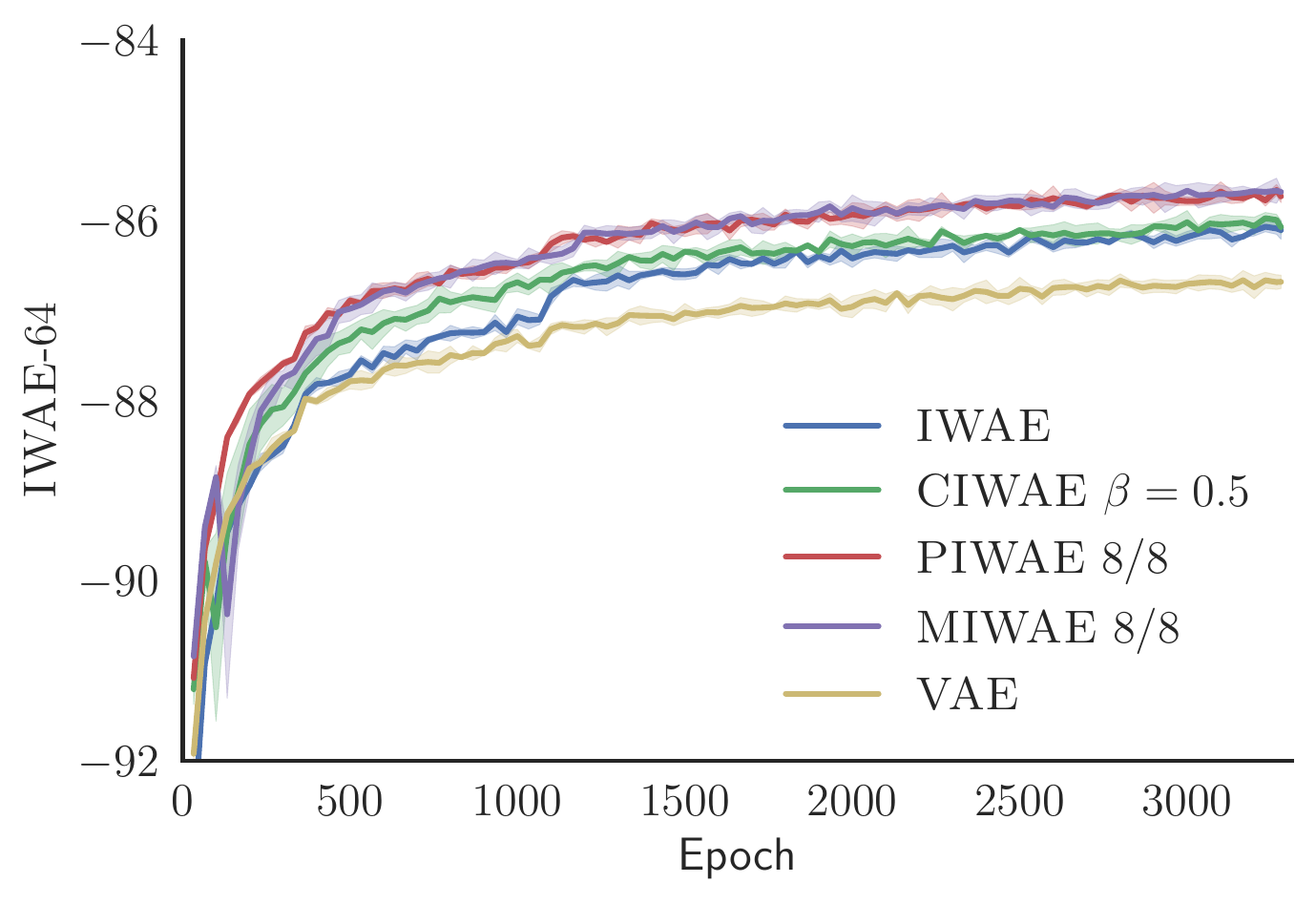}\vspace{-2pt}
		\caption{\textsc{IWAE}$_{64}$ \label{fig:mnistexpt/convergence/iwae64}}
	\end{subfigure}
	\begin{subfigure}[b]{0.33\textwidth}
		\centering
		\includegraphics[width=\textwidth]{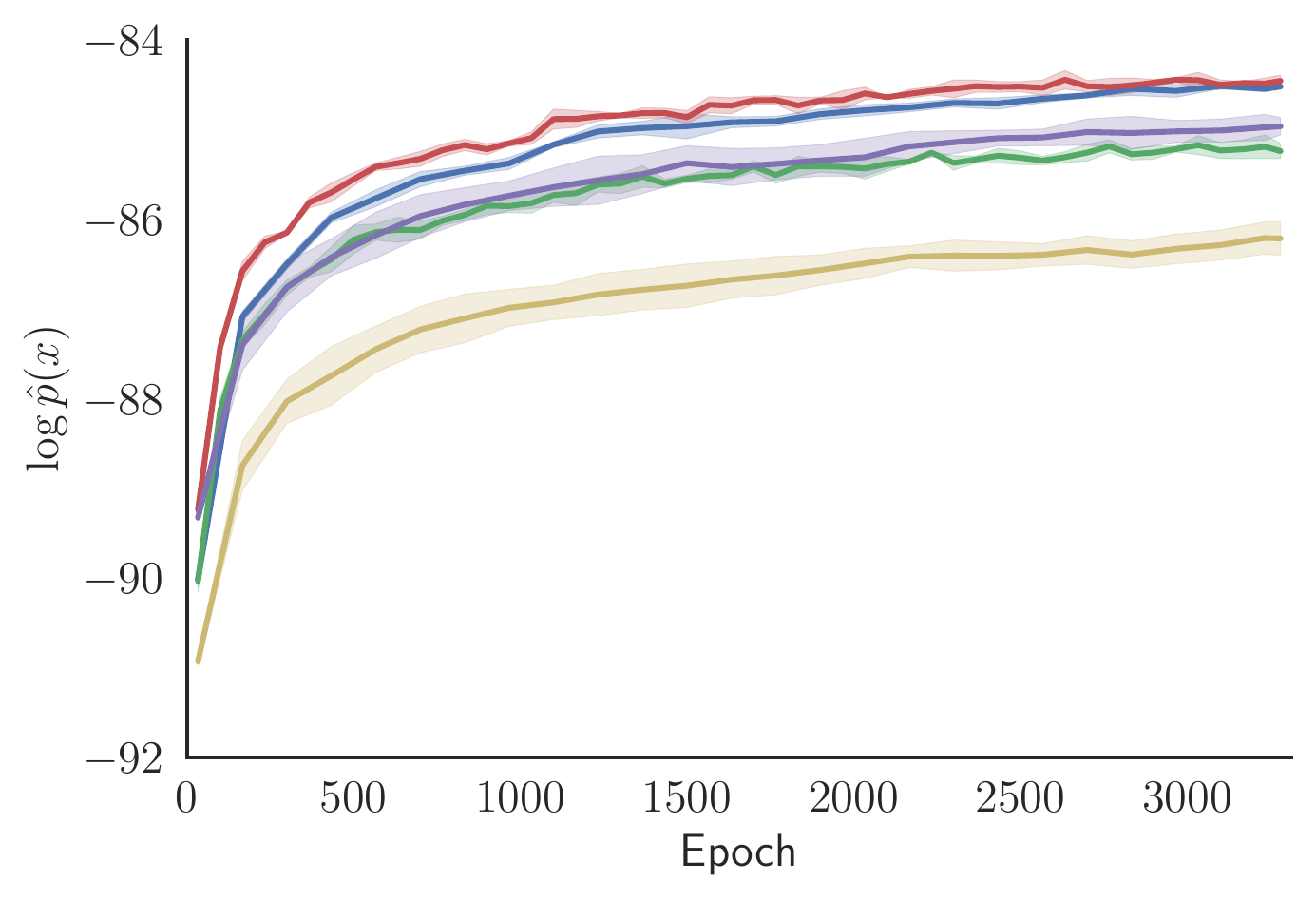}\vspace{-2pt}
		\caption{$\log \hat{p}(x)$ \label{fig:mnistexpt/convergence/logpx}}
	\end{subfigure}
	\begin{subfigure}[b]{0.33\textwidth}
		\centering
		\includegraphics[width=\textwidth]{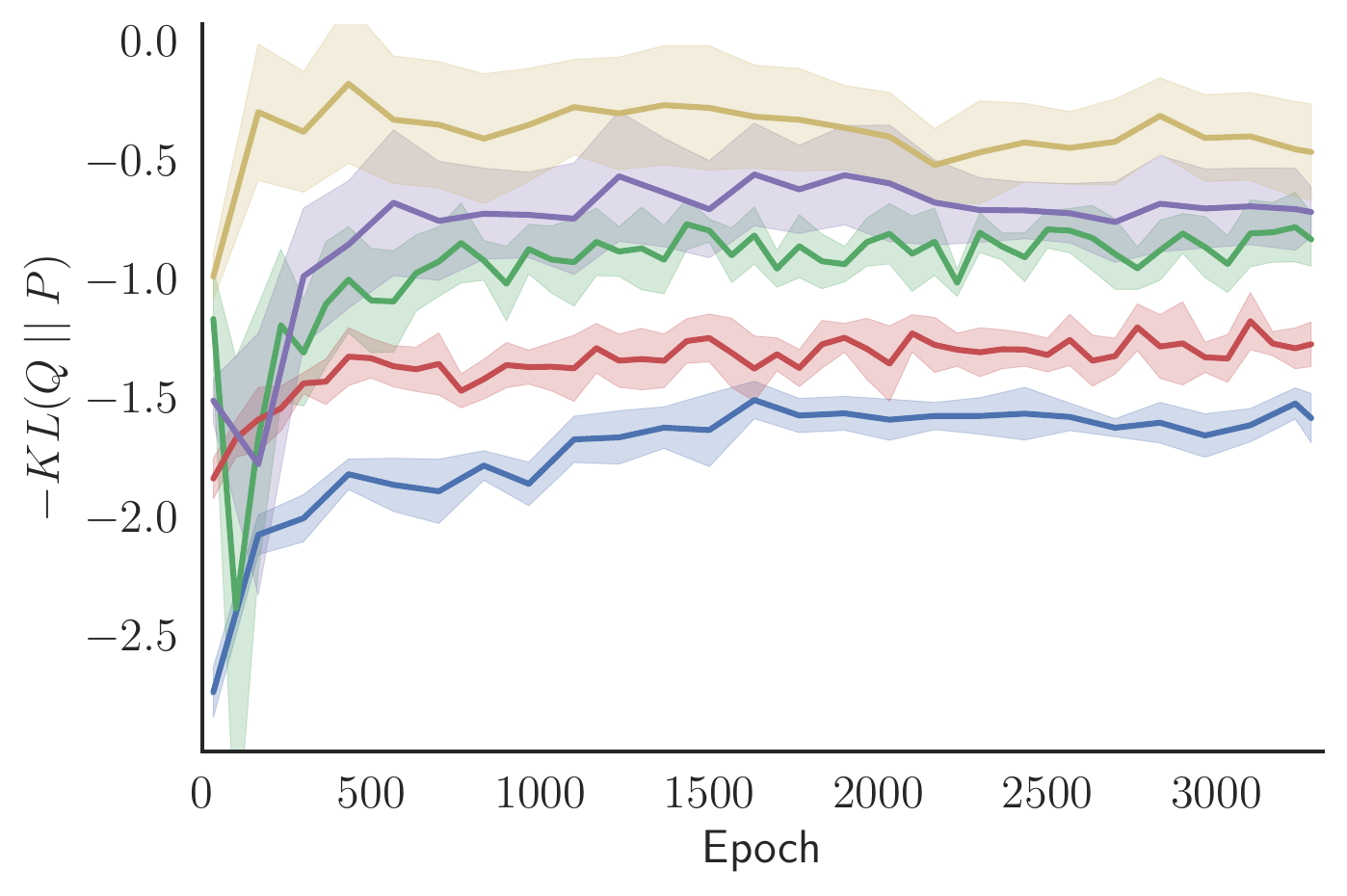}\vspace{-2pt}
		\caption{$-\mathrm{KL}(Q_{\phi}(z \given x) || P_{\theta}(z \given x))$ \label{fig:mnistexpt/convergence/kl}}
	\end{subfigure}\vspace{-6pt}
	\caption{Convergence of evaluation metrics on the test set with increased training time. 
		All lines show mean $\pm$ standard deviation 
		over 4 runs with different random
		initializations. Larger values are preferable for each plot.
		\vspace{-8pt}  \label{fig:mnistexpt/convergence}}
\end{figure*}

\begin{figure*}[t!]
	\centering
	\begin{subfigure}[b]{0.33\textwidth}
		\centering
		\includegraphics[width=\textwidth]{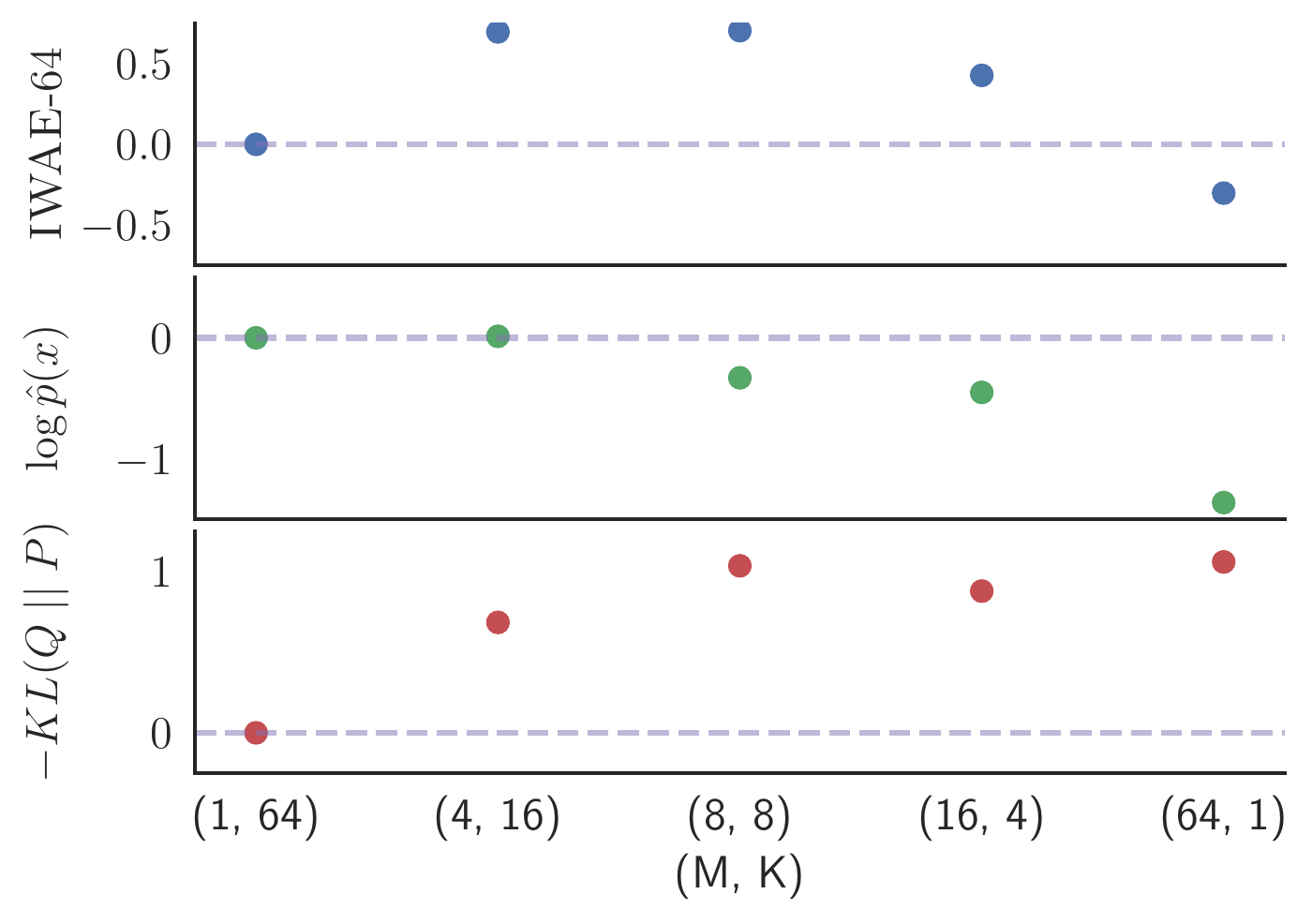}\vspace{-2pt}
		\caption{Comparing \gls{MIWAE} and \gls{IWAE} \label{fig:mnistexpt/dotplot/miwae}}
	\end{subfigure}
	\begin{subfigure}[b]{0.33\textwidth}
		\centering
		\includegraphics[width=\textwidth]{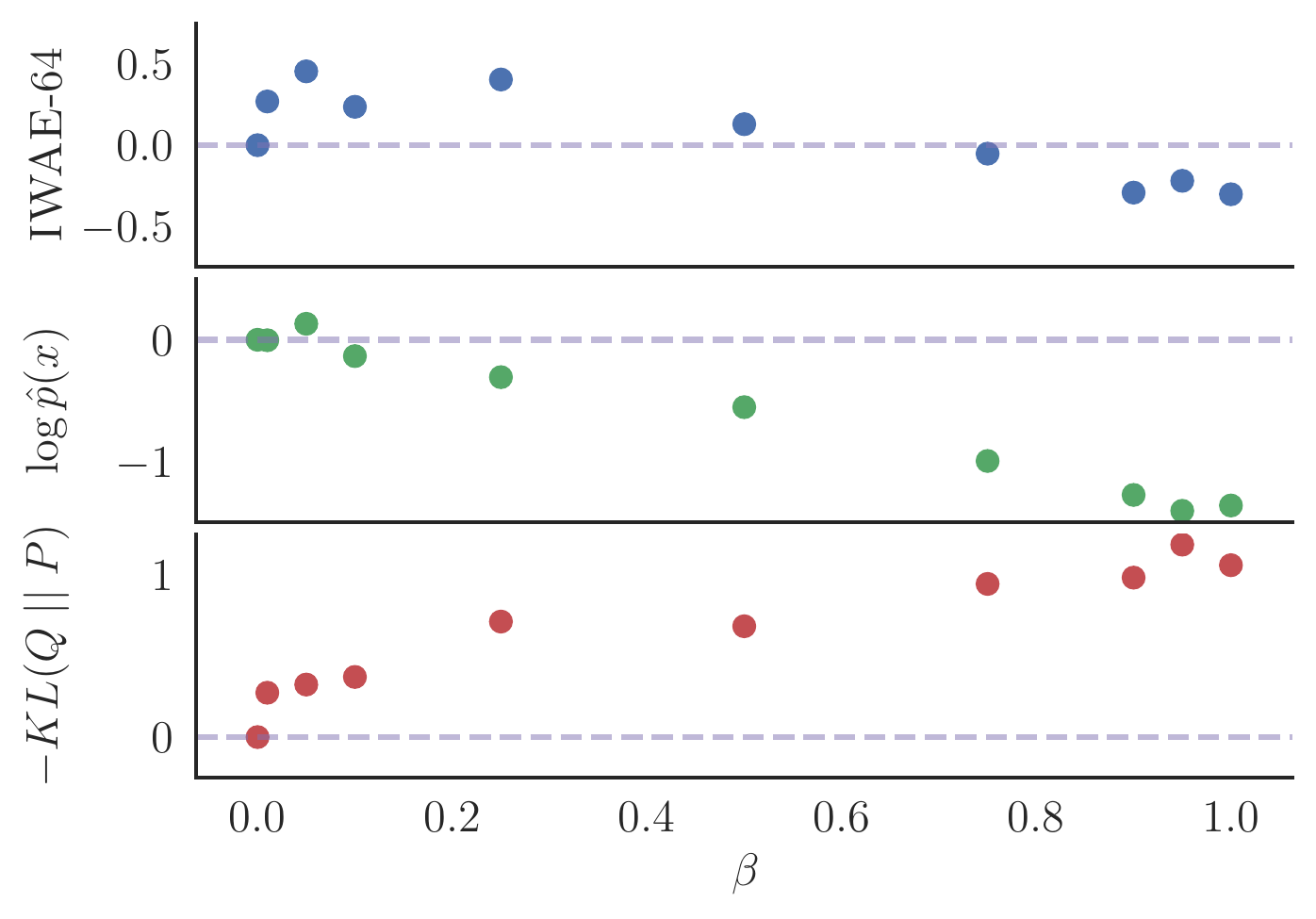}\vspace{-2pt}
		\caption{Comparing \gls{CIWAE} and \gls{IWAE} \label{fig:mnistexpt/dotplot/ciwae}}
	\end{subfigure}
	\begin{subfigure}[b]{0.33\textwidth}
		\centering
		\includegraphics[width=\textwidth]{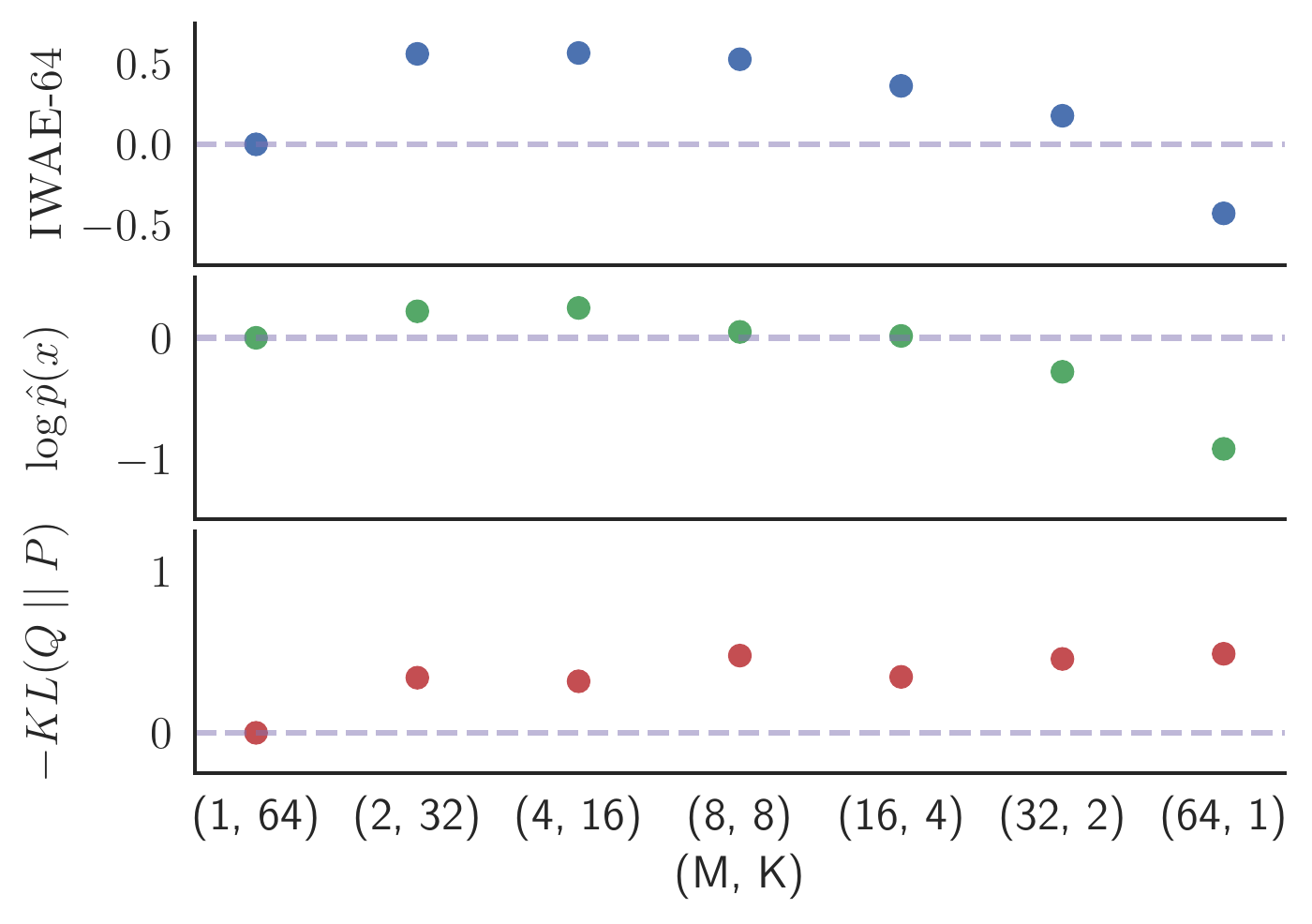}\vspace{-2pt}
		\caption{Comparing \gls{PIWAE} and \gls{IWAE} \label{fig:mnistexpt/dotplot/piwae}}
	\end{subfigure} \vspace{-12pt}
	\caption{Test set performance of \gls{MIWAE}, \gls{CIWAE}, and \gls{PIWAE} relative to \gls{IWAE} in terms of the \gls{IWAE}-64 (top), $\log \hat p(x)$ (middle), and $-\mathrm{KL}(Q_{\phi}(z \given x) || P_{\theta}(z \given x))$ (bottom) metrics.  All dots are the difference in the metric to that of~\gls{IWAE}. Dotted line is the \gls{IWAE} baseline.
		Note that in all cases, the far left of the plot correspond to
		settings equivalent to the \gls{IWAE}.\vspace{-8pt} 	 \label{fig:mnistexpt/dotplot}}
\end{figure*}

\begin{table*}[t!]                        
	\centering    
	\scriptsize             
	\setlength\tabcolsep{2.5pt}	  
	\renewcommand{\arraystretch}{1.2}          
	\caption{Mean final test set performance $\pm$ standard deviation over $4$ runs. Numbers in brackets in indicate $(M,K)$.
		The best result is shown in red, while bold results are not
		statistically significant to best result at the 5\% level of a Welch's t-test. \vspace{-8pt} \label{table:final-pef}}                    
	\begin{tabular}{|c|cccccccc|}                    
		\hline                                                       
		Metric & \gls{IWAE} & \gls{PIWAE} $(4,16)$ & \gls{PIWAE} $(8,8)$ & 
		\gls{MIWAE} $(4,16)$ & \gls{MIWAE} $(8,8)$ & \gls{CIWAE} $\beta=0.05$ & \gls{CIWAE} $\beta=0.5$ & \gls{VAE}     \\
		\hline    
		\gls{IWAE}-$64$ 
		& --$86.11\pm 0.10$ 
		& {\bf--}$\mathbf{85.68\pm 0.06}$ & {\bf--}$\mathbf{85.74\pm 0.07}$ 
		& {\color{red} \bf --$\mathbf{85.60\pm 0.07}$} & {\bf--}$\mathbf{85.69\pm 0.04}$ 
		& --$85.91\pm 0.11$ & --$86.08\pm 0.08$  
		& --$86.69\pm 0.08$ \\ 
		$\log \hat{p}(x)$ 
		& {\bf--}$\mathbf{84.52\pm 0.02}$  
		& {\color{red} \bf --$\mathbf{84.40\pm 0.17}$} 	& {\bf--}$\mathbf{84.46\pm 0.06}$ 
		& {\bf--}$\mathbf{84.56\pm 0.05}$ & --$84.97\pm 0.10$	 
		& 	{\bf--}$\mathbf{84.57\pm 0.09}$ & --$85.24\pm 0.08$  
		& --$86.21\pm 0.19$ \\ 
		$-\mathrm{KL}(Q|| P)$ 
		& --$1.59\pm 0.10$ 
		& --$1.27\pm 0.18$ & --$1.28\pm 0.09$ 
		& {\bf--}$\mathbf{1.04\pm 0.08}$ & {\bf--}$\mathbf{0.72\pm 0.11}$  
		& --$1.34\pm 0.14$ & {\bf--}$\mathbf{0.84\pm 0.11}$ 
		& {\color{red} \bf --$\mathbf{0.47\pm 0.20}$} \\ 
		\hline                                
	\end{tabular}                 
	\vspace{-12pt}                                           
\end{table*} 

\subsection{Experiments}
\label{sec:exp-algs}

We now use our new estimators to train deep generative models for the MNIST digits 
dataset~\citep{lecun1998gradient}.  For this, we duplicated the 
architecture and
training schedule outlined in~\citet{burda2016importance}. In particular, all networks were trained and evaluated using their
 stochastic binarization.
For all methods we set a budget of $T=64$ weights in the target estimate
for each datapoint in the minibatch.  

To assess different aspects of the training performance, we consider three
different metrics: $\ELBO_{\text{IWAE}}$ with $K=64$, 
$\ELBO_{\text{IWAE}}$ with $K=5000$, and the latter of these minus
 the former.  All reported metrics are evaluated on the test data.

The motivation for the  $\ELBO_{\text{IWAE}}$ with $K=64$ metric,
denoted as \gls{IWAE}-64, is that this is the target
used for training the \gls{IWAE} and so if another method does better on
this metric than the \gls{IWAE}, this is a clear indicator that \gls{SNR} issues
of the \gls{IWAE} estimator have degraded its performance.   In fact, 
this would demonstrate that, from a practical perspective, 
using the \gls{IWAE} estimator is sub-optimal, even if our explicit aim is to
optimize the \gls{IWAE} bound.  
The second metric, $\ELBO_{\text{IWAE}}$ with $K=5000$, denoted $\log \hat{p}(x)$, is used as a surrogate
for estimating the log marginal likelihood and thus provides an indicator
for fidelity of the learned generative model.  
The third metric is an
estimator for the divergence implicitly targeted by the \gls{IWAE}.  Namely, as shown
by~\citet{le2017auto}, the $\ELBO_{\text{IWAE}}$ can be interpreted as
\begin{align}
\ELBO_{\text{IWAE}} = 
\log p_{\theta}(x) - \mathrm{KL}(Q_{\phi}(z \given x) || P_{\theta}(z \given x))
\end{align}
\begin{align}
&\text{where} \quad Q_{\phi}(z \given x) := \prod\nolimits_{k = 1}^K q_{\phi}(z_k \given x),
\quad \text{and} \\
&P_{\theta}(z \given x) := \frac{1}{K} \sum\nolimits_{k = 1}^K \frac{\prod_{\ell = 1}^K q_{\phi}(z_{\ell} \given x)}{q_{\phi}(z_k \given x)} p_{\theta}(z_k\given x).
\end{align}
Thus we can estimate $\mathrm{KL}(Q_{\phi}(z \given x) || P_{\theta}(z \given x))$
using $\log \hat{p}(x)-{\text{IWAE-64}}$, to provide a
metric for divergence between the inference network and the
 proposal network.  
We use this instead of
$\mathrm{KL}(q_{\phi}(z \given x) || p_{\theta}(z \given x))$
because the latter can be deceptive metric for the inference network
fidelity.  For example, it tends to 
prefer $q_{\phi}(z \given x)$ that cover only one of the posterior modes, rather than
encompassing all of them.  As we showed in Section~\ref{sec:dir}, the implied target 
of the true gradients for the inference network improves as $K$ increases and so 
$\mathrm{KL}(Q_{\phi}(z \given x) || P_{\theta}(z \given x))$ should be a more reliable
metric of inference network performance.

Figure~\ref{fig:mnistexpt/convergence} shows the convergence of these metrics
for each algorithm.  Here we have considered
the middle value for each of the parameters, namely $K=M=8$
for \gls{PIWAE} and \gls{MIWAE}, and $\beta=0.5$ for \gls{CIWAE}.  
We see that \gls{PIWAE} and \gls{MIWAE} both
comfortably outperformed, and \gls{CIWAE} slightly outperformed,
\gls{IWAE} in terms of \gls{IWAE}-64 metric, despite
\gls{IWAE} being directly trained on this target.  In terms of 
$\log \hat{p}(x)$, \gls{PIWAE} gave the best performance, 
followed by \gls{IWAE}.  For the \textsc{KL}, we see that
the \gls{VAE} performed best followed by \gls{MIWAE}, 
with \gls{IWAE} performing the worst.  
We note here that the \textsc{KL} is not an exact
measure of the inference network performance as it also depends on the generative model.
As such, the apparent superior performance of the \gls{VAE} may be because it produces a
simpler model, as per the observations of~\citet{burda2016importance}, which in turn is easier
to learn an inference network for.  Critically though,
\gls{PIWAE} improves this metric whilst also improving generative network performance, such
that this reasoning no longer applies.  Similar behavior is observed for~\gls{MIWAE} and~\gls{CIWAE}
for different parameter settings (see Appendix~\ref{sec:app:exp-algs}).


\begin{figure}[t!]
	\centering
	\includegraphics[width=0.39\textwidth]{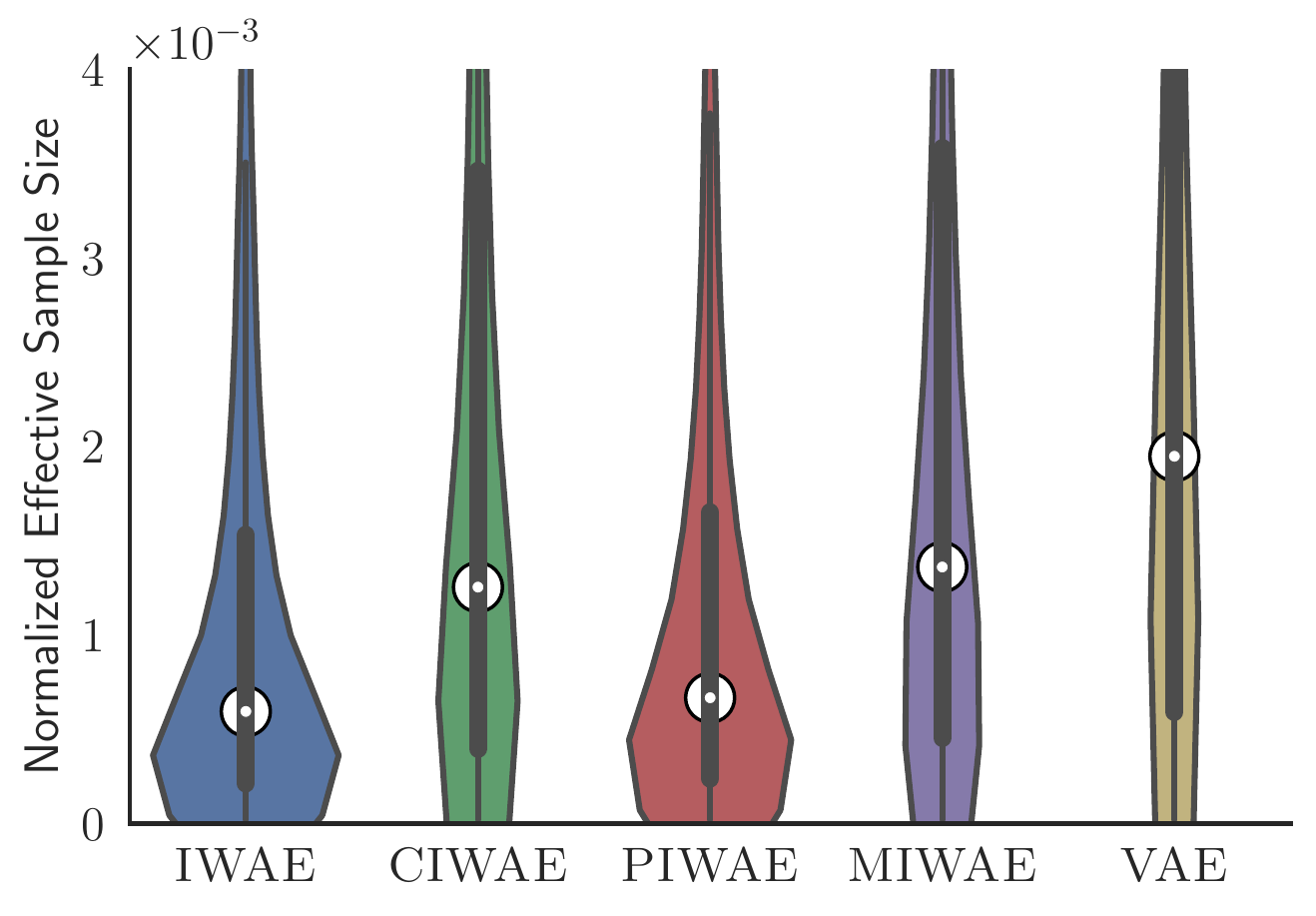}\vspace{-10pt}
	\caption{Violin plots of ESS estimates for each image of MNIST,
		normalized by the number of samples drawn. A violin plot uses a kernel density plot on each side -- thicker means more MNIST images whose $q_{\phi}$ achieves that ESS. 
		\vspace{-17pt}  \label{fig:violiness}}
\end{figure}

We next considered tuning the parameters for each of
our algorithms as shown in Figure~\ref{fig:mnistexpt/dotplot}, for
which we look at the final metric values after training.
Table~\ref{table:final-pef} further summarizes the performance
for certain selected parameter settings. For \gls{MIWAE} we see that as we increase $M$, 
the $\log \hat{p}(x)$ metric gets worse, while the \textsc{KL}
gets better.  The \gls{IWAE}-64  metric initially increases with $M$, before reducing
again from $M=16$ to $M=64$, suggesting that intermediate
values for $M$ (i.e. $M\neq1$, $K\neq1$) give a better trade-off.  For \gls{PIWAE}, similar
behavior to \gls{MIWAE} is seen for the \gls{IWAE}-64 
and \textsc{KL} metrics.  However, unlike for \gls{MIWAE}, we see that 
$\log \hat{p}(x)$ initially increases with $M$, such that 
\gls{PIWAE} provides uniform improvement over \gls{IWAE}
for the $M=2,4,8,$ and $16$ cases.  
\gls{CIWAE} exhibits similar behavior in increasing $\beta$
as increasing $M$ for \gls{MIWAE}, but there appears to
be a larger degree of noise in the evaluations, while the optimal
value of $\beta$, though non-zero, seems to be closer
to \gls{IWAE} than for the other algorithms.

As an additional measure of the performance of the inference
network that is distinct to any of the training targets, 
we also considered the effective sample size (ESS)~\cite{mcbook}
for the fully trained networks, defined as
\begin{align}
\text{ESS} = (\textstyle \sum_{k=1}^K w_k)^2 / \textstyle \sum_{k=1}^K w_k^2.
\end{align}
The ESS is a measure of how many unweighted samples would be equivalent
to the weighted sample set.  
A low $\text{ESS}$ indicates that the inference network is struggling to
perform effective inference for the generative network.
The results, given in Figure~\ref{fig:violiness}, show that the
ESSs for \gls{CIWAE}, \gls{MIWAE}, and the \gls{VAE} were
all significantly larger than for \gls{IWAE} and \gls{PIWAE},
with \gls{IWAE} giving a particularly poor ESS.

Our final experiment looks at the \gls{SNR} values for the 
inference networks during training.  Here we took a number of different neural network gradient weights
at different layers of the network and calculated empirical estimates
for their \gls{SNR}s at various points during the training.  We then
averaged these estimates over the different network weights, the 
results of which are given in Figure~\ref{fig:inferencesnr}.  This
clearly shows the low \gls{SNR} exhibited by the \gls{IWAE}
inference network, suggesting that our results from the simple
Gaussian experiments carry over to the more complex neural network
domain.

\begin{figure}[t!]
	\centering
	\includegraphics[width=0.395\textwidth]{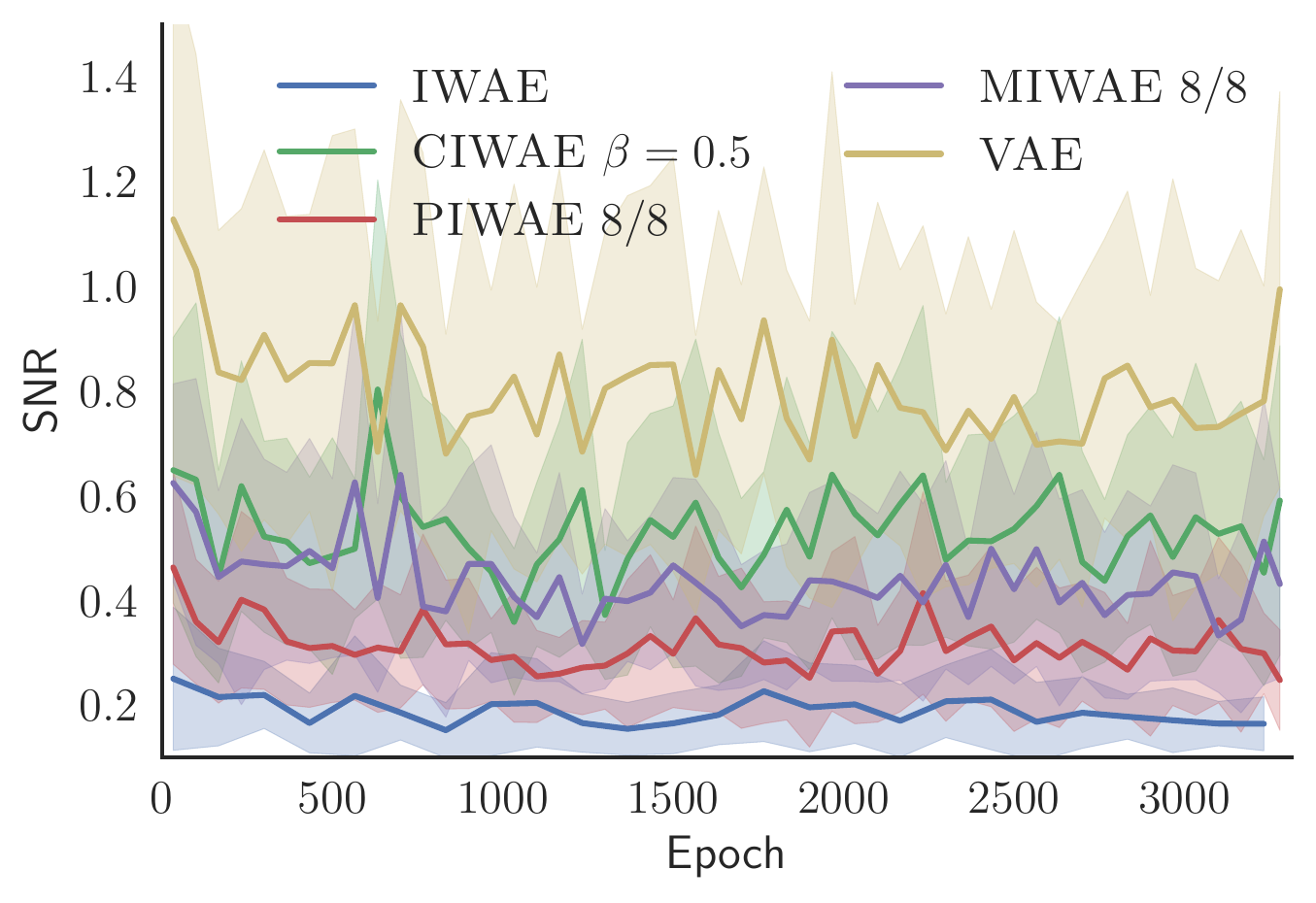}\vspace{-9pt}
	\caption{\gls{SNR} of inference network weights during training. All lines are mean $\pm$ standard deviation over 20 randomly chosen weights per layer.
		\vspace{-12pt} \label{fig:inferencesnr}}
\end{figure}


\section{Conclusions}
We have provided theoretical and empirical evidence that algorithmic approaches to increasing the tightness of the \gls{ELBO} independently to the expressiveness of the inference network can be detrimental to learning by reducing the
signal-to-noise ratio of the inference network gradients.
Experiments on a simple latent variable model confirmed our theoretical findings.
We then exploited these insights to introduce three estimators,
\gls{PIWAE},~\gls{MIWAE}, and~\gls{CIWAE} and showed that each
can deliver improvements over~\gls{IWAE}, even when the metric
used for this assessment is the~\gls{IWAE} target itself.  In particular, 
each was able to deliver improvement in the training of the inference network,
without any reduction in the quality of the learned generative network.

Whereas \gls{MIWAE} and \gls{CIWAE} mostly allow for balancing the requirements of
the inference and generative networks, \gls{PIWAE} appears to be able to offer
simultaneous improvements to both, with the improved training of the inference network
having a knock-on effect on the generative network.  Key to achieving this is, is its use
of separate targets for the two networks, opening up interesting avenues for
future work.
%
%

%


\clearpage

\appendix
	\onecolumn
\setlength{\abovedisplayskip}{5pt}
\setlength{\belowdisplayskip}{5pt}
\setlength{\abovedisplayshortskip}{5pt}
\setlength{\belowdisplayshortskip}{5pt}

\titlespacing\section{0pt}{8pt plus 2pt minus 2pt}{2pt plus 2pt minus 0pt}
\titlespacing\subsection{0pt}{8pt plus 2pt minus 2pt}{2pt plus 2pt minus 0pt}
\titlespacing\subsubsection{0pt}{8pt plus 2pt minus 2pt}{2pt plus 2pt minus 0pt}

\thispagestyle{empty} 
\rule{\textwidth}{1pt}
\vspace{-6pt}
\begin{center}
	\textbf{ \Large  Appendices for Tighter Variational Bounds are Not Necessarily
		Better}
\end{center}
\rule{\textwidth}{1pt}

	\begin{minipage}{\textwidth}
		\centering
		\vspace{17pt}
		\textbf{Tom Rainforth \quad Adam R. Kosiorek \quad Tuan Anh Le \quad Chris J. Maddison \\
		 Maximilian Igl \quad Frank Wood \quad Yee Whye Teh}
		\vspace{6pt}
	\end{minipage}

\setlength{\parskip}{0.3em}

\section{Proof of SNR Convergence Rates}
\label{sec:proof}

Before proving Theorem~\ref{the:app:snr}, we first introduce the following lemma that will be helpful for demonstrating the result.  We note that the lemma can be interpreted as a generalization on the well-known results for the third and fourth moments of Monte Carlo estimators.
\begin{lemma}
\label{the:thirdOrder}
Let $a_1,\dots,a_K$, $b_1,\dots,b_K$, and $c_1,\dots,c_K$ be sets of random variables such that
\begin{itemize}
\item The $a_k$ are independent and identically distributed (i.i.d.).  Similarly, the $b_k$ are i.i.d. and the $c_k$ are i.i.d.
\item $\E [a_k] = \E [b_k] = \E [c_k] = 0, \quad \forall k\in\{1,\dots,K\}$
\item $a_i$, $b_j$, $c_k$ are mutually independent if $i\neq j \neq k \neq i$, but potentially dependent otherwise
\end{itemize}
Then the following results hold
\begin{align}
&\E\left[\left(\frac{1}{K}\sum_{k=1}^K a_k\right)
\left(\frac{1}{K}\sum_{k=1}^K b_k\right)
\left(\frac{1}{K}\sum_{k=1}^K c_k\right)\right] =\frac{1}{K^2}\E[a_1 b_1 c_1] \\
&\textnormal{Var}\left[\left(\frac{1}{K}\sum_{k=1}^K a_k\right)
\left(\frac{1}{K}\sum_{k=1}^K b_k\right)\right]
=\frac{1}{K^3}\textnormal{Var}[a_1 b_1] +\frac{K-1}{K^3}\E [a_1^2] \E[b_1^2]
+ \frac{2K-2}{K^3}\left(\E [a_1b_1]\right)^2.
\end{align} 
\end{lemma}
\begin{proof}
For the first result we have 
\begin{align*}
\E\left[\left(\frac{1}{K}\sum_{k=1}^K a_k\right)
\left(\frac{1}{K}\sum_{k=1}^K b_k\right)
\left(\frac{1}{K}\sum_{k=1}^K c_k\right)\right] &=\frac{1}{K^3}\sum_{k=1}^K\E[a_k b_k c_k]
+\frac{1}{K^3}\sum_{i=1}^K\sum_{j=1,j\neq i}^K\sum_{k=1,k\neq i}^K
\E[a_i b_j c_k] \\
&=\frac{1}{K^2}\E[a_1 b_1 c_1]
+\frac{1}{K^3}\sum_{i=1}^K\sum_{j=1,j\neq i}^K\sum_{k=1,k\neq i}^K
\E[a_i] \E[b_j c_k] \\
&=\frac{1}{K^2}\E[a_1 b_1 c_1]
\end{align*}
as required by using the fact that $\E[a_i]=0$.\vspace{10pt}

For the second result, we again use the fact that any ``off-diagonal'' terms for $a_i a_j b_k b_{\ell}$, that is terms which have the value of one index distinct to any of the others, have expectation zero to give
\begin{align*}
\textnormal{Var}&\left[\left(\frac{1}{K}\sum_{k=1}^K a_k\right)
\left(\frac{1}{K}\sum_{k=1}^K b_k\right)\right]
= 
\E\left[\left(\frac{1}{K}\sum_{k=1}^K a_k\right)^2
\left(\frac{1}{K}\sum_{k=1}^K b_k\right)^2\right]
-\left(\E\left[\left(\frac{1}{K}\sum_{k=1}^K a_k\right)
\left(\frac{1}{K}\sum_{k=1}^K b_k\right)\right]\right)^2 \\
&=\frac{1}{K^4}\sum_{k=1}^K\E[a_k^2 b_k^2]
+\frac{1}{K^4}\sum_{i=1}^K\sum_{k=1,k\neq i}^K \E [a_i^2] \E[b_k^2]
+\frac{2}{K^4}\sum_{i=1}^K\sum_{k=1,k\neq i}^K \E [a_ib_i]\E [a_kb_k]
-\frac{1}{K^3}\left(\E[a_1 b_1]\right)^2 \\
&=\frac{1}{K^3}\textnormal{Var}[a_1 b_1] +\frac{K-1}{K^3}\E [a_1^2] \E[b_1^2]
+ \frac{2K-2}{K^3}\left(\E [a_1b_1]\right)^2
\end{align*}
as required.
\end{proof}

\begin{theoremApp}
	\label{the:app:snr}
	Assume that when $M=K=1$, the expected gradients; the variances of the gradients; and the 
	first four moments of  $w_{1,1}$, $\nabla_{\theta} w_{1,1}$, and 
	$\nabla_{\phi} w_{1,1}$ are all finite and the variances are
	also non-zero.
	Then the signal-to-noise ratios of the gradient estimates converge at the following rates
	\begin{align}
	\SNR_{M,K} (\theta) &= 
	\sqrt{M}\left|\frac{ \sqrt{K} \; 
		\nabla_{\theta} Z -\frac{1}{2Z\sqrt{K}}\nabla_{\theta} \left(\frac{\textnormal{Var} \left[w_{1,1}\right]}{Z^2}\right)+ O\left(\frac{1}{K^{3/2}}\right) }
	{\sqrt{\E \left[w_{1,1}^2\left(\nabla_{\theta} \log w_{1,1}-\nabla_{\theta} \log Z\right)^2\right]} + O\left(\frac{1}{K}\right)}\right| \\
	\SNR_{M,K} (\phi) & =\sqrt{M} \left|\frac{
		\nabla_{\phi} \textnormal{Var} \left[w_{1,1}\right] + O\left(\frac{1}{K}\right) }
	{2 Z\sqrt{K} \; \sigma\left[\nabla_{\phi} w_{1,1}\right] +O\left(\frac{1}{\sqrt{K}}\right)}\right|
	\end{align}
	where $Z := p_{\theta}(x)$ is the true marginal likelihood.
\end{theoremApp}

\begin{proof}
We start by considering the variance of the estimators.   We will first exploit the 
fact that each $\hat{Z}_{m,K}$ is independent and identically distributed and then apply
Taylor's theorem\footnote{This approach follows similar lines
	to the derivation of nested Monte Carlo convergence bounds in~\cite{rainforth2017thesis,rainforth2017opportunities,fort2017mcmc}
and the derivation of the mean squared error for self-normalized
importance sampling, see e.g.~\cite{hesterberg1988advances}.}
to 
$\log \hat{Z}_{m,K}$ about $Z$, using $R_1(\cdot)$ to indicate the remainder term, as
follows.
\begin{align}
M& \cdot \text{Var} \left[\Delta_{M,K}\right] = \text{Var} \left[\Delta_{1,K}\right]
=\text{Var} \left[\nabla_{\theta,\phi} \left( 
\log Z + \frac{\hat{Z}_{1,K}-Z}{Z} + R_1\left(\hat{Z}_{1,K}\right)
\right)\right] \displaybreak[0] \nonumber \\
=& \text{Var} \left[\nabla_{\theta,\phi} \left( 
\frac{\hat{Z}_{1,K}-Z}{Z} + R_1\left(\hat{Z}_{1,K}\right)
\right)\right] \displaybreak[0] \nonumber\\
=& \E \left[\left(\nabla_{\theta,\phi} \left( 
\frac{\hat{Z}_{1,K}-Z}{Z} + R_1\left(\hat{Z}_{1,K}\right)
\right)\right)^2\right] -  \left(\E \left[\nabla_{\theta,\phi} \left( 
\frac{\hat{Z}_{1,K}-Z}{Z} + R_1\left(\hat{Z}_{1,K}\right)
\right)\right]\right)^2 \displaybreak[0] \nonumber\\
=&\E \left[\left(\frac{1}{K} \sum_{k=1}^{K} \frac{Z \nabla_{\theta,\phi} w_{1,k}-w_{1,k}\nabla_{\theta,\phi} Z}{Z^2}
+ \nabla_{\theta, \phi} R_1\left(\hat{Z}_{1,K}\right)\right)^2\right]
-  \left(\nabla_{\theta,\phi} \cancelto{0}{\E \left[
	\frac{\hat{Z}_{1,K}-Z}{Z}\right]} + \E \left[ \nabla_{\theta, \phi} R_1\left(\hat{Z}_{1,K}\right)\right]\right)^2 \displaybreak[0] \nonumber\\
\begin{split}
\label{eq:full-var}
=&\frac{1}{KZ^4} \E \left[\left(Z \nabla_{\theta,\phi} w_{1,1}-w_{1,1}\nabla_{\theta,\phi} Z\right)^2\right] + \text{Var} \left[ \nabla_{\theta, \phi} R_1\left(\hat{Z}_{1,K}\right)\right] \\
&+ 2\E \left[\left(\nabla_{\theta, \phi} R_1\left(\hat{Z}_{1,K}\right)\right)
\left(\frac{1}{K} \sum_{k=1}^{K} \frac{Z \nabla_{\theta,\phi} w_{1,k}-w_{1,k}\nabla_{\theta,\phi} Z}{Z^2}\right)\right]
\end{split}
\end{align}
Now by the mean-value form of the remainder in Taylor's Theorem, we have that for some $\tilde{Z}$  between $Z$ and $\hat{Z}_{1,K}$
\begin{align*}
R_1\left(\hat{Z}_{1,K}\right) = -\frac{\left(\hat{Z}_{1,K}-Z\right)^2}{2 \tilde{Z}^2}
\end{align*}
and therefore
\begin{align}
\label{eq:nabla_R1}
\nabla_{\theta,\phi} R_1\left(\hat{Z}_{1,K}\right)
=-\frac{\tilde{Z}\left(\hat{Z}_{1,K}-Z\right)\nabla_{\theta,\phi} \left(\hat{Z}_{1,K}-Z\right)
	-\left(\hat{Z}_{1,K}-Z\right)^2 \nabla_{\theta,\phi}\tilde{Z}}{\tilde{Z}^3}.
\end{align}
By definition, $\tilde{Z} = Z+\alpha (\hat{Z}_{1,K}-Z)$ for some
$0<\alpha<1$.  Furthermore, the moments of $\nabla_{\theta,\phi} \alpha$ must be bounded as otherwise we would end up with non-finite moments of the overall gradients, which
would, in turn, violate our assumptions.  As moments of $(\hat{Z}_{1,K}-Z)$ and $\nabla_{\theta,\phi}(\hat{Z}_{1,K}-Z)$
decrease with $K$, we thus further have that $\tilde{Z}=Z+o(1)$ and $\nabla_{\theta,\phi} \tilde{Z}=\nabla_{\theta,\phi}  Z+o(1)$
where $o(1)$ indicates terms that tend to zero as $K\to\infty$.  By substituting these into~\eqref{eq:nabla_R1} and
doing a further Taylor expansion on $1/\tilde{Z}$ about $1/Z$ we now have
\begin{align*}
\nabla_{\theta,\phi} R_1\left(\hat{Z}_{1,K}\right)
&=-\left(\frac{1}{Z^3}+o(1)\right)\left(\left(Z+o(1)\right)\left(\hat{Z}_{1,K}-Z\right)\nabla_{\theta,\phi} \left(\hat{Z}_{1,K}-Z\right)
	-\left(\hat{Z}_{1,K}-Z\right)^2 \left(\nabla_{\theta,\phi}Z+o(1)\right)\right) \displaybreak[0]\\
&=-\frac{1}{Z^3}\left(\hat{Z}_{1,K}-Z\right) \left(Z\,\nabla_{\theta,\phi} \left(\hat{Z}_{1,K}-Z\right)
-\left(\hat{Z}_{1,K}-Z\right) \nabla_{\theta,\phi}Z\right) +o(\epsilon) \displaybreak[0]\\
&=\frac{1}{Z^3}\left(Z-\hat{Z}_{1,K}\right) \left(Z\,\nabla_{\theta,\phi}\hat{Z}_{1,K}
-\hat{Z}_{1,K}\nabla_{\theta,\phi}Z\right)+o(\epsilon)
\end{align*}
where $o(\epsilon)$ indicates (asymptotically dominated) higher order terms originating from the $o(1)$ terms.
If we now define
\begin{align}
a_k &= \frac{Z-w_{1,k}}{Z^3} \quad \text{and}\\
b_k &= Z\,\nabla_{\theta,\phi} w_{1,K}
-w_{1,K}\nabla_{\theta,\phi}Z
\end{align}
then we have $\E [a_k] = \E [b_k] = 0, \, \forall k\in\{1,\dots,K\}$ and
\begin{align*}
\nabla_{\theta,\phi} R_1\left(\hat{Z}_{1,K}\right) = \left(\frac{1}{K}\sum_{k=1}^{K} a_k\right)
\left(\frac{1}{K}\sum_{k=1}^{K} b_k\right)+o(\epsilon).
\end{align*}
This is in the form required by Lemma~\ref{the:thirdOrder} and satisfies the required assumptions and so we immediately have
\begin{align*}
\text{Var} \left[ \nabla_{\theta, \phi} R_1\left(\hat{Z}_{1,K}\right)\right] = 
\frac{1}{K^3}\textnormal{Var}[a_1 b_1] +\frac{K-1}{K^3}\E [a_1^2] \E[b_1^2]
+ \frac{2K-2}{K^3}\left(\E [a_1b_1]\right)^2+o(\epsilon) = O\left(\frac{1}{K^2}\right)
\end{align*}
by the second result in Lemma~\ref{the:thirdOrder}.  If we further define
\begin{align}
c_k = b_k \quad \forall k\in\{1,\dots,K\}
\end{align}
then we can also use the first result in Lemma~\ref{the:thirdOrder} to give
\begin{align*}
2\E \left[\left(\nabla_{\theta, \phi} R_1\left(\hat{Z}_{1,K}\right)\right)
\left(\frac{1}{K} \sum_{k=1}^{K} \frac{Z \nabla_{\theta,\phi} w_{1,k}-w_{1,k}\nabla_{\theta,\phi} Z}{Z^2}\right)\right]
= \frac{2}{Z^2 K^2}\E \left[a_1 b_1^2\right]+o(\epsilon)
= O\left(\frac{1}{K^2}\right).
\end{align*}

Substituting these results back into~\eqref{eq:full-var} now gives
\begin{align}
\label{eq:var}
\text{Var} \left[\Delta_{M,K}\right] = 
\frac{1}{MKZ^4} \E \left[\left(Z \nabla_{\theta,\phi} w_{1,1}-w_{1,1}\nabla_{\theta,\phi} Z\right)^2\right]
+\frac{1}{M}O\left(\frac{1}{K^2}\right).
\end{align}
\vspace{10pt}

Considering now the expected gradient estimate and again using Taylor's theorem, this
time to a higher number of terms,
\begin{align}
\E \left[\Delta_{M,K}\right] &= \E \left[\Delta_{1,K}\right]
= \E \left[\Delta_{1,K}-\nabla_{\theta,\phi} \log Z\right] +
\nabla_{\theta,\phi} \log Z\displaybreak[0] \nonumber \\ &= 
\nabla_{\theta,\phi} \E \left[\log Z + \frac{\hat{Z}_{1,K}-Z}{Z} 
- \frac{\left(\hat{Z}_{1,K}-Z\right)^2}{2Z^2} + R_2\left(\hat{Z}_{1,K}\right)-\log Z\right]
+
\nabla_{\theta,\phi} \log Z \displaybreak[0] \nonumber \\
&= 
-\frac{1}{2}\nabla_{\theta,\phi} \E \left[\left(\frac{\hat{Z}_{1,K}-Z}{Z}\right)^2\right]+
\nabla_{\theta,\phi} \E \left[R_2\left(\hat{Z}_{1,K}\right)\right] +\nabla_{\theta,\phi} \log Z \displaybreak[0] \nonumber \\
&= -\frac{1}{2}\nabla_{\theta,\phi} \left(\frac{\text{Var} \left[\hat{Z}_{1,K}\right]}{Z^2}\right)
+\nabla_{\theta,\phi} \E \left[R_2\left(\hat{Z}_{1,K}\right)\right] +\nabla_{\theta,\phi} \log Z \displaybreak[0] \nonumber \\
&= -\frac{1}{2K}\nabla_{\theta,\phi} \left(\frac{\text{Var} \left[w_{1,1}\right]}{Z^2}\right)
+\nabla_{\theta,\phi} \E \left[R_2\left(\hat{Z}_{1,K}\right)\right]+\nabla_{\theta,\phi} \log Z . \label{eq:mean}
\end{align}
Here we have
\begin{align*}
R_2\left(\hat{Z}_{1,K}\right) = \frac{\left(\hat{Z}_{1,K}-Z\right)^3}{3 \check{Z}^3}
\end{align*}
with $\check{Z}$ defined analogously to $\tilde{Z}$ but not necessarily taking on the same value.
Using $\check{Z} = Z+o(1)$ and $\nabla_{\theta,\phi}\check{Z} = \nabla_{\theta,\phi} Z+o(1)$ as before and again applying a Taylor expansion about $1/Z$ now yields
\begin{align*}
\nabla_{\theta,\phi} \E \left[R_2\left(\hat{Z}_{1,K}\right)\right]
&= \nabla_{\theta,\phi} \left(\frac{1}{3Z^3}\E \left[\left(\frac{1}{K}\sum_{k=1}^{K} w_{1,k}-Z\right)^3\right]\right)+o(\epsilon)
\end{align*}
and thus by applying Lemma~\ref{the:thirdOrder} with $a_k=b_k=c_k=w_{1,k}-Z$ we have
\begin{align*}
\nabla_{\theta,\phi} \E \left[R_2\left(\hat{Z}_{1,K}\right)\right]
&= \frac{1}{K^2}\nabla_{\theta,\phi} \left(\frac{1}{3 Z^3}\E[\left(w_{1,k}-Z\right)^3]\right)+o(\epsilon)
=O\left(\frac{1}{K^2}\right).
\end{align*}
%
Substituting this back into~\eqref{eq:mean} now yields
\begin{align}
\label{eq:bias-final}
\E \left[\Delta_{M,K}\right] = \nabla_{\theta,\phi} \log Z
-\frac{1}{2K}\nabla_{\theta,\phi} \left(\frac{\text{Var} \left[w_{1,1}\right]}{Z^2}\right)
+O\left(\frac{1}{K^2}\right)
\end{align}

Finally, by combing~\eqref{eq:var} and~\eqref{eq:bias-final}, and noting that
$\sqrt{\frac{A}{K}+\frac{B}{K^2}} = \frac{A}{\sqrt{K}}+\frac{B}{2AK^{3/2}}+O\left(\frac{1}{K^{(5/2)}}\right)$, we have
for $\theta$
\begin{align}
\SNR_{M,K} (\theta)
&= \left|\frac{\nabla_{\theta} \log Z-\frac{1}{2K}\nabla_{\theta} \left(\frac{\text{Var} \left[w_{1,1}\right]}{Z^2}\right)
	+O\left(\frac{1}{K^2}\right)}
{\sqrt{\frac{1}{MKZ^4} \E \left[\left(Z \nabla_{\theta} w_{1,1}-w_{1,1}\nabla_{\theta} Z\right)^2\right]
	+\frac{1}{M}O\left(\frac{1}{K^2}\right)}}\right| 
\\ &= \sqrt{M}\left|\frac{Z^2\sqrt{K}
	\left(\nabla_{\theta} \log Z -\frac{1}{2K}\nabla_{\theta} \left(\frac{\text{Var} \left[w_{1,1}\right]}{Z^2}\right)\right) 
	+ O\left(\frac{1}{K^{3/2}}\right) }
{\sqrt{\E \left[\left(Z \nabla_{\theta} w_{1,1}-w_{1,1}\nabla_{\theta} Z\right)^2\right]} +O\left(\frac{1}{K}\right)}\right| 
\displaybreak[0] \\
&= \sqrt{M}\left|\frac{ \sqrt{K} \; 
	\nabla_{\theta} Z -\frac{1}{2Z\sqrt{K}}\nabla_{\theta} \left(\frac{\text{Var} \left[w_{1,1}\right]}{Z^2}\right)+ O\left(\frac{1}{K^{3/2}}\right) }
{\sqrt{\E \left[w_{1,1}^2\left(\nabla_{\theta} \log w_{1,1}-\nabla_{\theta} \log Z\right)^2\right]} + O\left(\frac{1}{K}\right)}\right| =O\left(\sqrt{MK}\right).
\end{align}
For $\phi$, then because $\nabla_{\phi} Z = 0$, we instead have
\begin{align}
\SNR_{M,K} (\phi) = \sqrt{M} \left|\frac{
	\nabla_{\phi} \text{Var} \left[w_{1,1}\right] + O\left(\frac{1}{K}\right) }
{2 Z\sqrt{K} \; \sigma\left[\nabla_{\phi} w_{1,1}\right] +O\left(\frac{1}{\sqrt{K}}\right)}\right| = O\left(\sqrt{\frac{M}{K}}\right)
\end{align}
and we are done.
\end{proof}

\vspace{-4pt}

\section{Derivation of Optimal Parameters for Gaussian Experiment}
\label{sec:optGauss}

\vspace{-4pt}

To derive the optimal parameters for the Gaussian experiment we first note that
\begin{align*}
\mathcal J(\theta, \phi) 
&= \frac{1}{N}\log \prod_{n=1}^{N} p_{\theta}(x^{(n)}) - \frac{1}{N}\sum_{n=1}^{N} \KL{Q_{\phi}(z_{1:K}|x^{(n)})}{P_{\theta}(z_{1:K}|x^{(n)})} \quad \text{where}\\
P_{\theta}(z_{1:K}|x^{(n)}) &= \frac{1}{K} \sum_{k=1}^{K}
q_{\phi} (z_1 | x^{(n)}) \dots q_{\phi} (z_{k-1} | x^{(n)}) p_{\theta} (z_k | x^{(n)}) 
q_{\phi} (z_{k+1} | x^{(n)}) \dots q_{\phi} (z_{K} | x^{(n)}),
\end{align*}
$Q_{\phi}(z_{1:K}|x^{(n)})$ is as per~\eqref{eq:background/q_is_z_is} 
 and the form of the \gls{KL} is taken from~\cite{le2017auto}.
Next, we note that $\phi$ only controls the mean of the proposal so, while it is not possible to drive the
$\textsc{KL}$ to zero, it will be minimized for any particular $\theta$ when the means of $q_{\phi}(z|x^{(n)})$
and $p_{\theta}(z|x^{(n)})$ are the same.  
Furthermore, the corresponding minimum possible value of the \textsc{KL} is independent of
$\theta$ and so we can
calculate the optimum pair $(\theta^*,\phi^*)$ by first optimizing for $\theta$ and then choosing the matching $\phi$.
The optimal $\theta$ maximizes $\log \prod_{n=1}^{N} p_{\theta}(x^{(n)})$, giving $\theta^* := \mu^* = \frac{1}{N} \sum_{n = 1}^N x^{(n)}$.
As we straightforwardly have $p_{\theta} (z | x^{(n)}) = 
\mathcal{N}(z; \left(x^{(n)}+\mu\right)/2, I/2)$, the \text{KL} is then minimized
when $A=I/2$ and $b=\mu/2$, giving $\phi^* := (A^*, b^*)$, where $A^* = I / 2$ and $b^* = \mu^* / 2$.

\vspace{-4pt}

\section{Additional Empirical Analysis of SNR}
\label{sec:app:emp}

\subsection{Histograms for VAE}
\label{sec:hist-vae}

To complete the picture for the effect of $M$ and $K$ on the distribution
of the gradients, we generated histograms for the $K=1$ (i.e. \gls{VAE})
gradients as $M$ is varied.  As shown in Figure~\ref{fig:snr/b_hist_vae},
we see the expected effect from the law of large numbers that the 
variance of the estimates decreases with $M$, but not the expected value.

\begin{figure*}[h]
	\centering
		\begin{subfigure}[b]{0.45\textwidth}
			\centering
			\includegraphics[width=\textwidth]{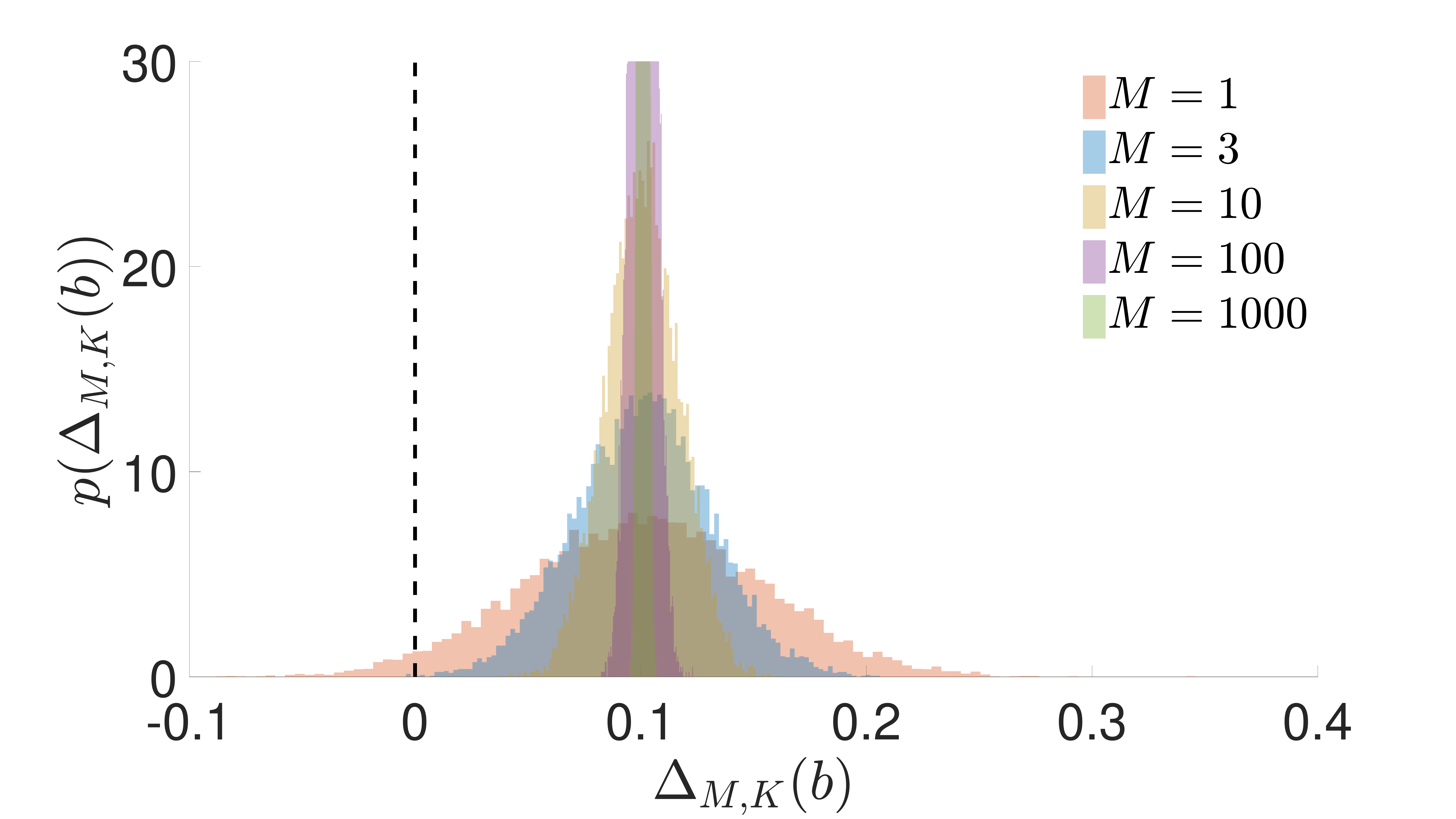}
			\caption{\gls{VAE} inference network gradient estimates \label{fig:snr/b_hist_vae}}
		\end{subfigure} ~~~~~~~~~~
		\begin{subfigure}[b]{0.45\textwidth}
			\centering
			\includegraphics[width=\textwidth]{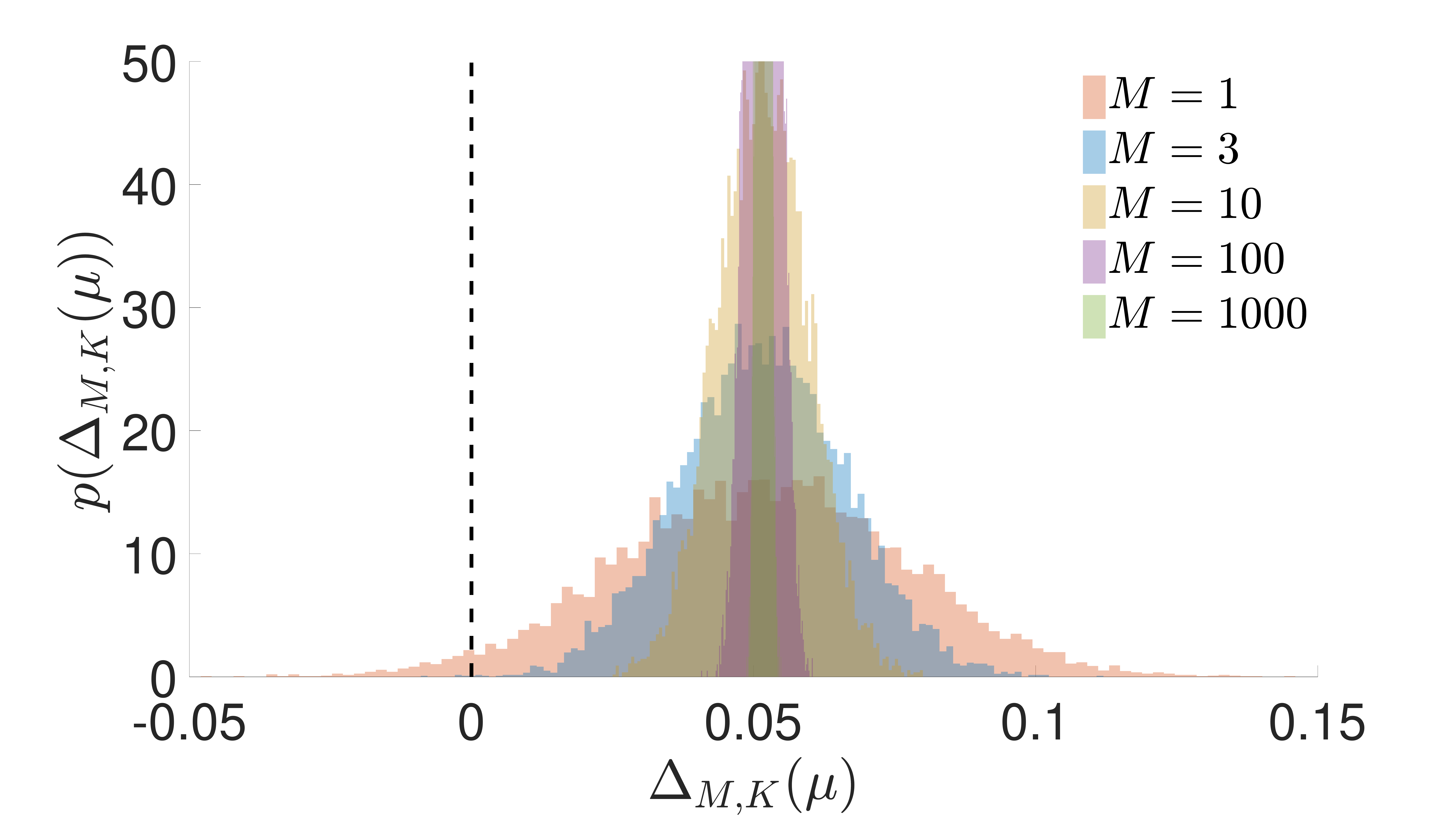}
			\caption{\gls{VAE} generative network gradient estimates \label{fig:snr/mu_hist_vae}}
		\end{subfigure}
	\caption{Histograms of gradient estimates $\Delta_{M,K}$ for the generative network and 
		the	inference network using the \gls{VAE} ($K=1$) objectives with different values of $M$. \vspace{-12pt}
		\label{fig:snr/hists-vae}}
\end{figure*}

\subsection{Convergence of RMSE for Generative Network}
\label{sec:app:rmse}

As explained in the main paper, the SNR is not an entirely appropriate metric for
the generative network -- a low SNR is still highly problematic, but a high SNR
does not indicate good performance.
It is thus perhaps better to measure
the quality of the gradient estimates for the generative network by looking at the \gls{RMSE}
to $\nabla_{\mu} \log Z$, i.e. $\sqrt{\E \left[\lVert\Delta_{M,K}-\nabla_{\mu} \log Z\rVert_2^2\right]}$.
The convergence of this \gls{RMSE}  is shown in
Figure~\ref{fig:snr/rmse} where the solid lines are the \gls{RMSE} estimates using $10^4$ runs 
and the shaded regions
show the interquartile range of the individual estimates. We see that increasing 
$M$ in the \gls{VAE} reduces the variance
of the estimates but has negligible effect on the \gls{RMSE} due to the fixed bias.  On the
other hand,
we see that increasing $K$ leads to a monotonic improvement, initially improving
at a rate $O(1/K)$ (because the bias is the dominating term in this region),
before settling to the standard Monte Carlo convergence rate of $O(1/\sqrt{K})$
(shown by the dashed lines).

\begin{figure}[h!]
	\centering
	\includegraphics[width=0.47\textwidth]{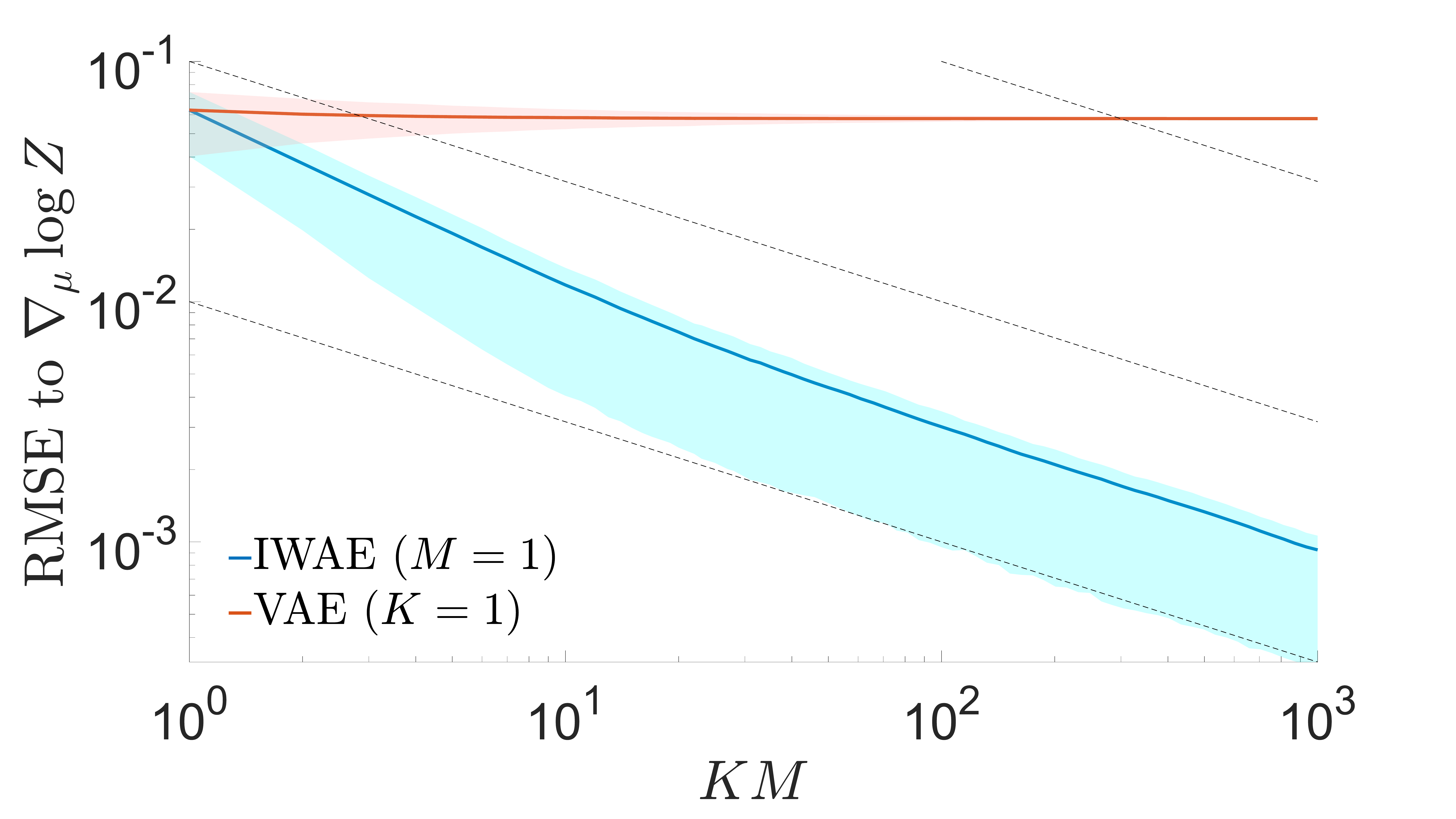}
	\caption{RMSE in $\mu$ gradient estimate to $\nabla_{\mu} \log Z$ 
		\label{fig:snr/rmse}}
\end{figure} 

\newpage

\subsection{Experimental Results for High Variance Regime}
\label{sec:hv}

\vspace{-4pt}

We now present empirical results for a case where our weights are higher variance. Instead
of choosing a point close to the optimum by offsetting parameters with a standard deviation of $0.01$, 
we instead offset using a standard deviation of $0.5$.  We further increased the proposal covariance to $I$
to make it more diffuse.  This is now a scenario where the model is far from its optimum and the proposal
is a very poor match for the model, giving very high variance weights.

We see that the behavior is the
same for variation in $M$, but somewhat distinct for variation in $K$.  
In particular, the \gls{SNR} and \textsc{dsnr} only decrease slowly with $K$ for the inference network, while increasing $K$ no longer has much benefit for the \gls{SNR} of the
inference network.
It is clear that, for this
setup, the problem is very far from the asymptotic regime in $K$ such that our theoretical results no
longer directly apply.  Nonetheless, the high-level effect observed is still that the \gls{SNR} of 
the inference
network deteriorates, albeit slowly, as $K$ increases.

\begin{figure}[h]
	\centering
	\begin{subfigure}[b]{0.45\textwidth}
		\centering
		\includegraphics[width=\textwidth]{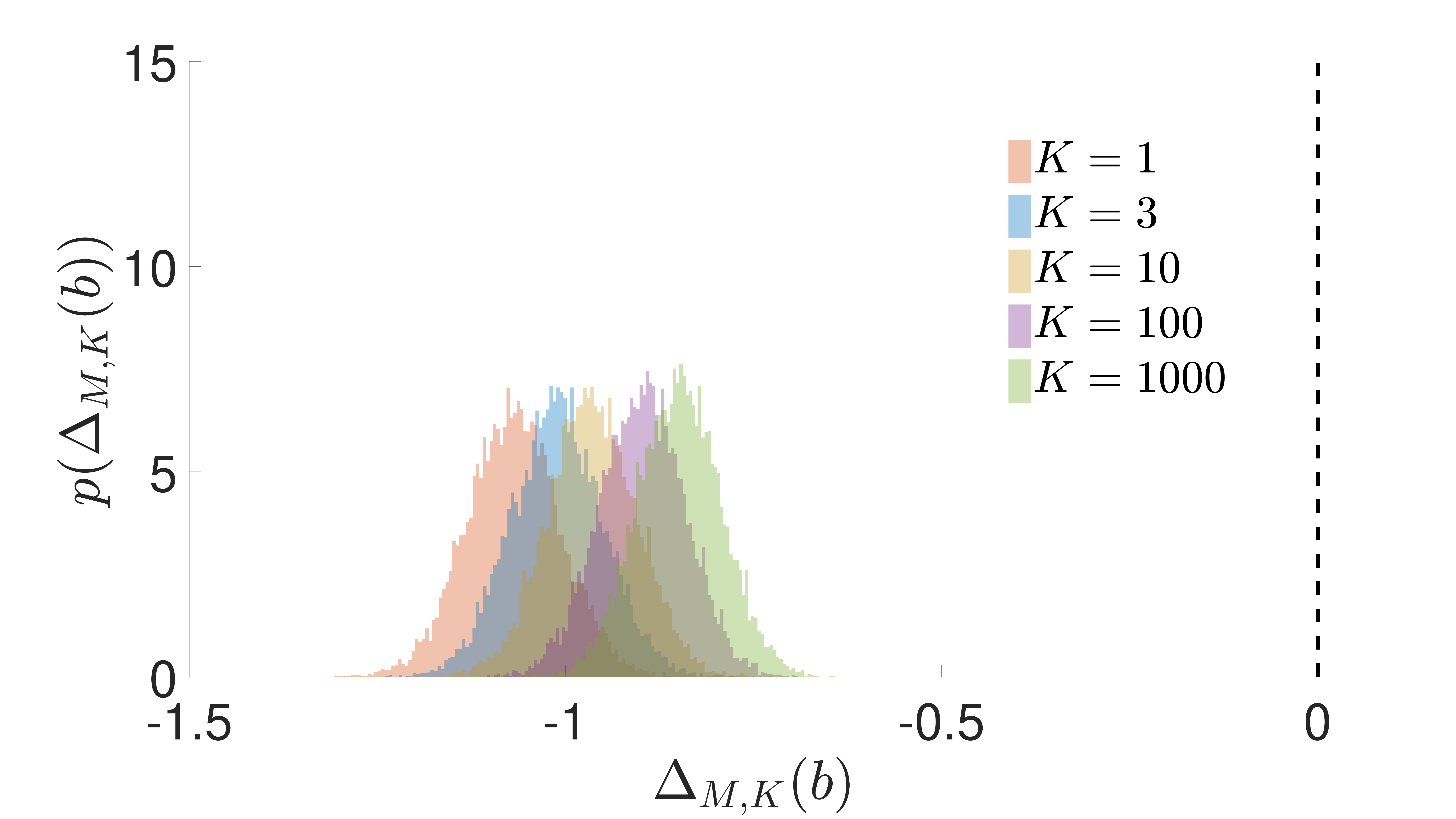}
		\caption{\gls{IWAE} inference network gradient estimates \label{fig:hv/b_hist_iwae}}
	\end{subfigure} ~~~~~~~~~~
	\begin{subfigure}[b]{0.45\textwidth}
		\centering
		\includegraphics[width=\textwidth]{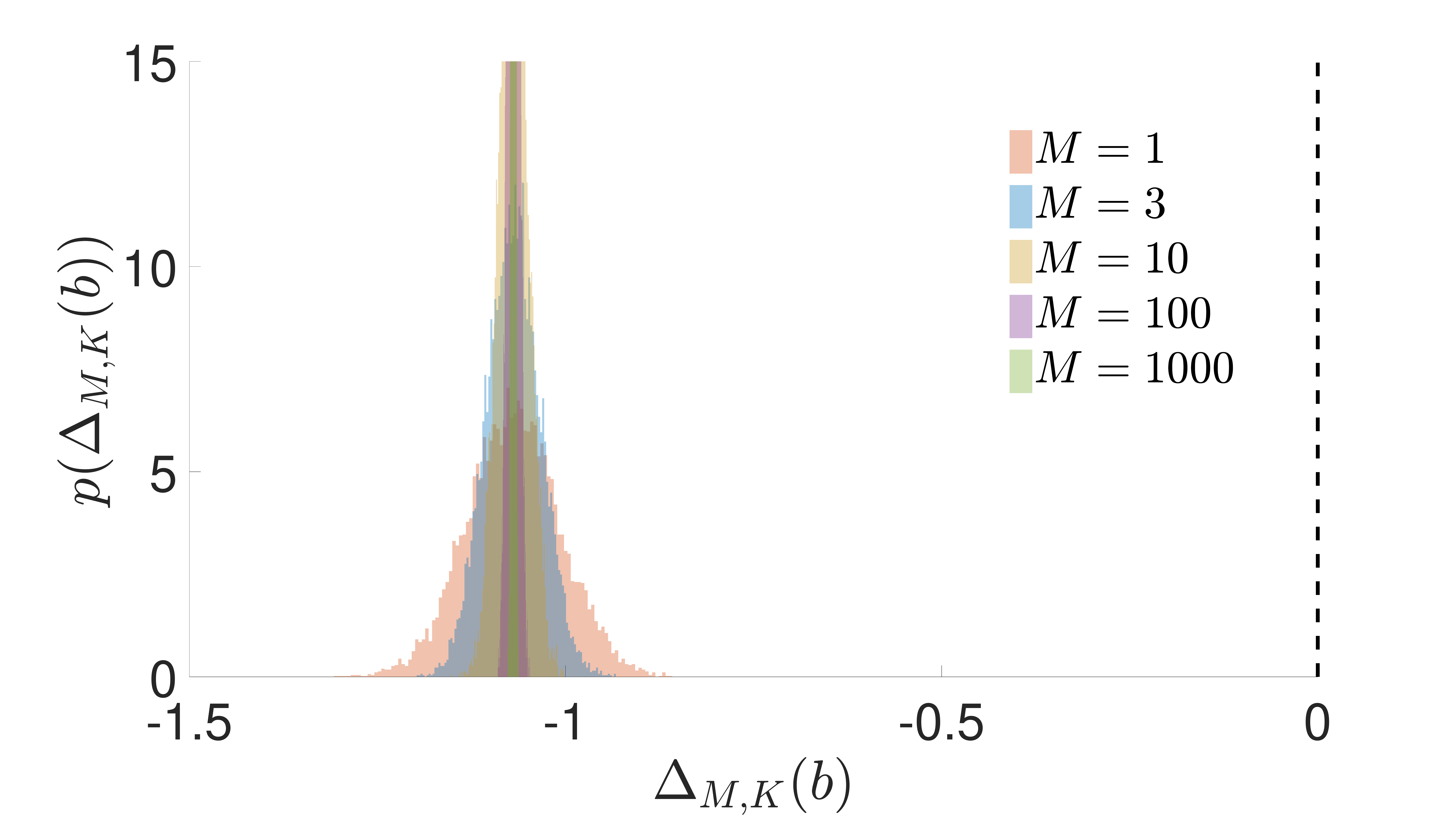}
		\caption{\gls{VAE} inference network gradient estimates \label{fig:hv/b_hist_vae}}
	\end{subfigure}\\
	\begin{subfigure}[b]{0.45\textwidth}
		\centering
		\includegraphics[width=\textwidth]{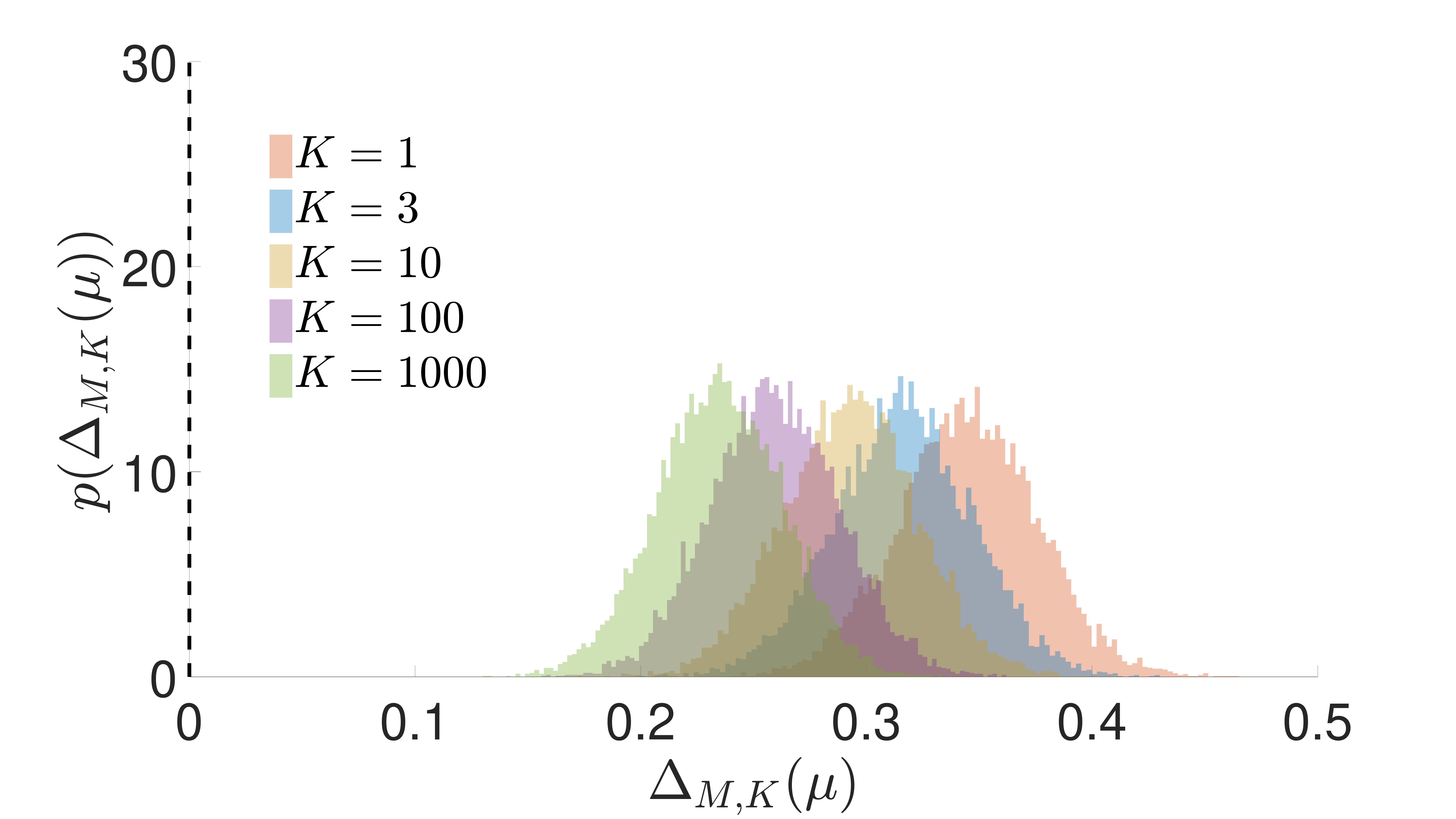}
		\caption{\gls{IWAE} generative network gradient estimates \label{fig:hv/mu_hist_iwae}}
	\end{subfigure} ~~~~~~~~~~
	\begin{subfigure}[b]{0.45\textwidth}
		\centering
		\includegraphics[width=\textwidth]{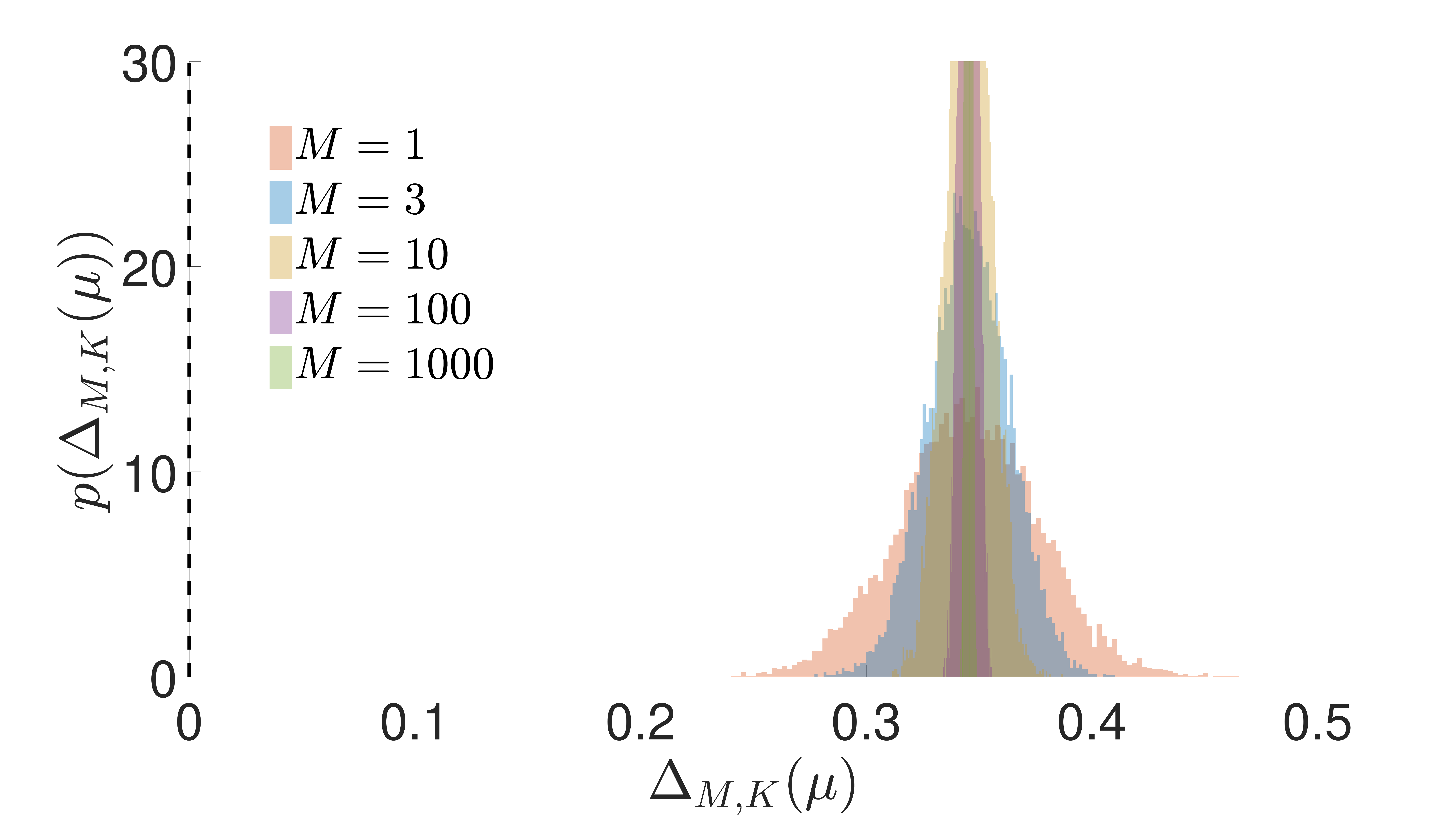}
		\caption{\gls{VAE} generative network gradient estimates \label{fig:hv/mu_hist_vae}}
	\end{subfigure}
	\caption{Histograms of gradient estimates as per Figure~\ref{fig:snr/hists}.
		\label{fig:hv/hists}
	\vspace{-4pt}}
\end{figure}

\begin{figure}[h]
	\centering
	\begin{subfigure}[b]{0.45\textwidth}
		\centering
		\includegraphics[width=\textwidth]{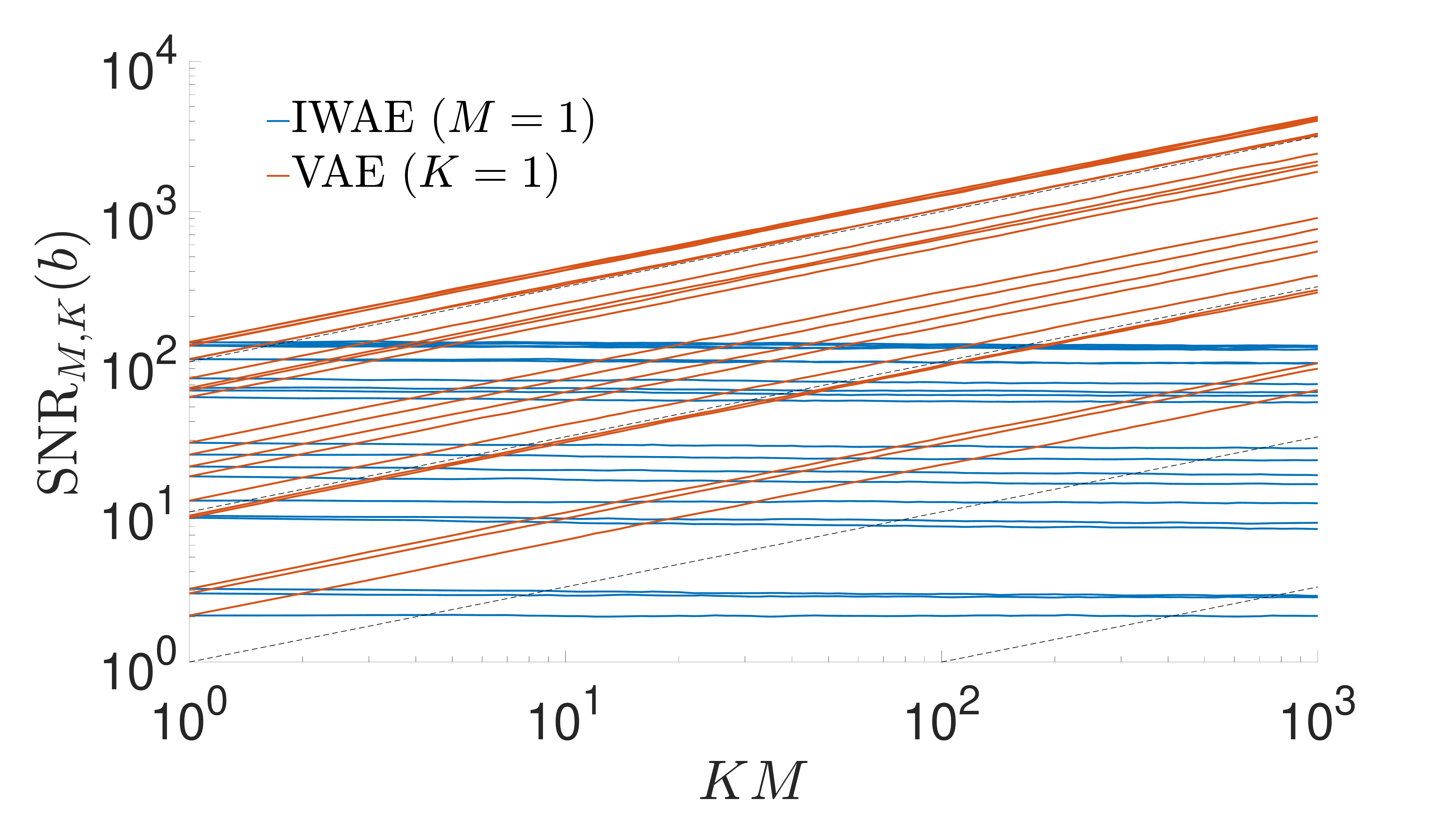}
		\caption{Convergence of \gls{SNR} for inference network \label{fig:hv/b}}
	\end{subfigure} ~~~~~~~~~~
	\begin{subfigure}[b]{0.45\textwidth}
		\centering
		\includegraphics[width=\textwidth]{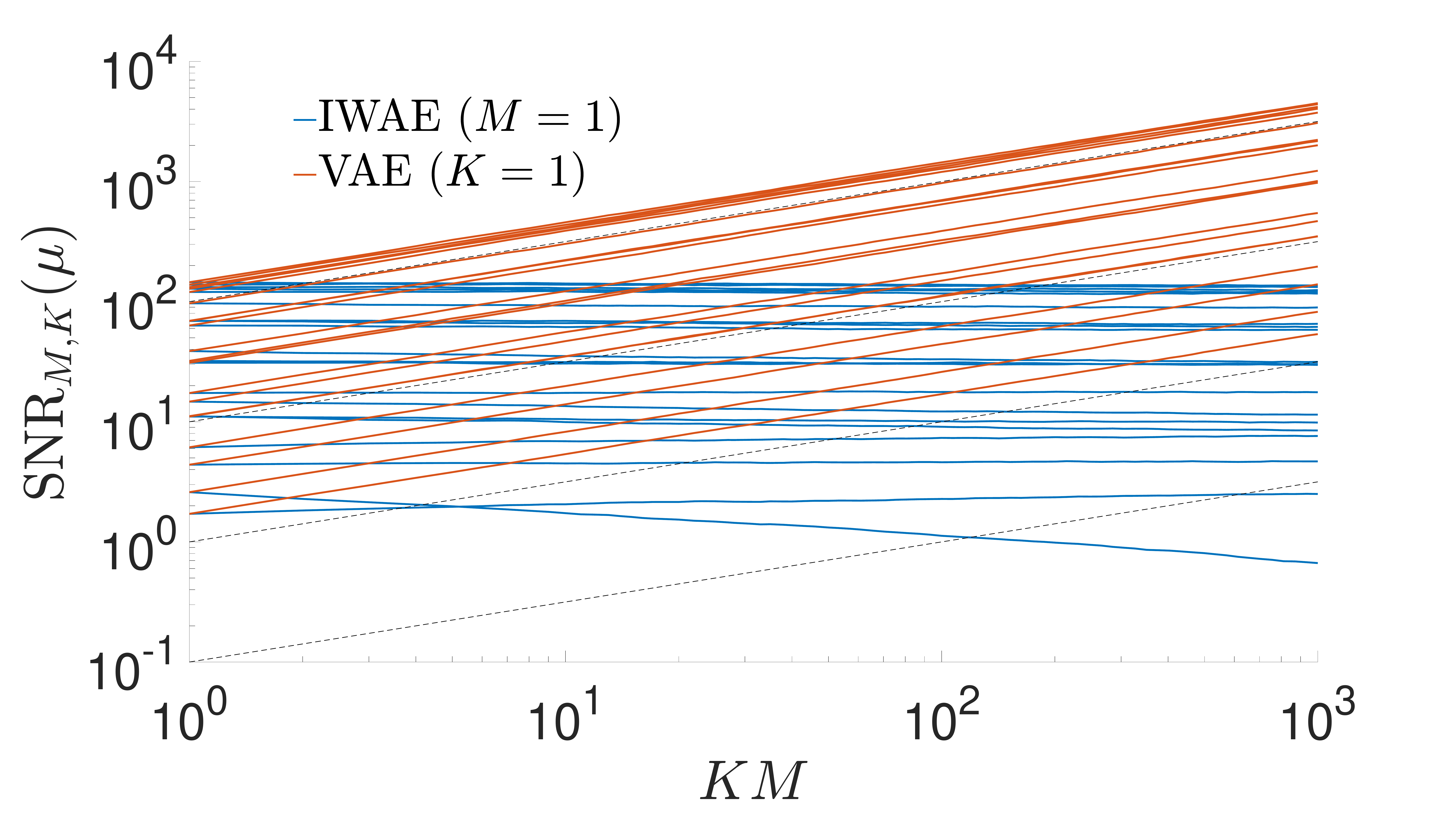}
		\caption{Convergence of \gls{SNR} for generative network\label{fig:hv/mu}}
	\end{subfigure}
	\caption{Convergence of signal-to-noise ratios of gradient estimates
		as per Figure~\ref{fig:snr/K_conv}.
		\label{fig:hv/K_conv}}
\end{figure}

\begin{figure}[h]
	\centering
	\begin{subfigure}[b]{0.45\textwidth}
		\centering
		\includegraphics[width=\textwidth]{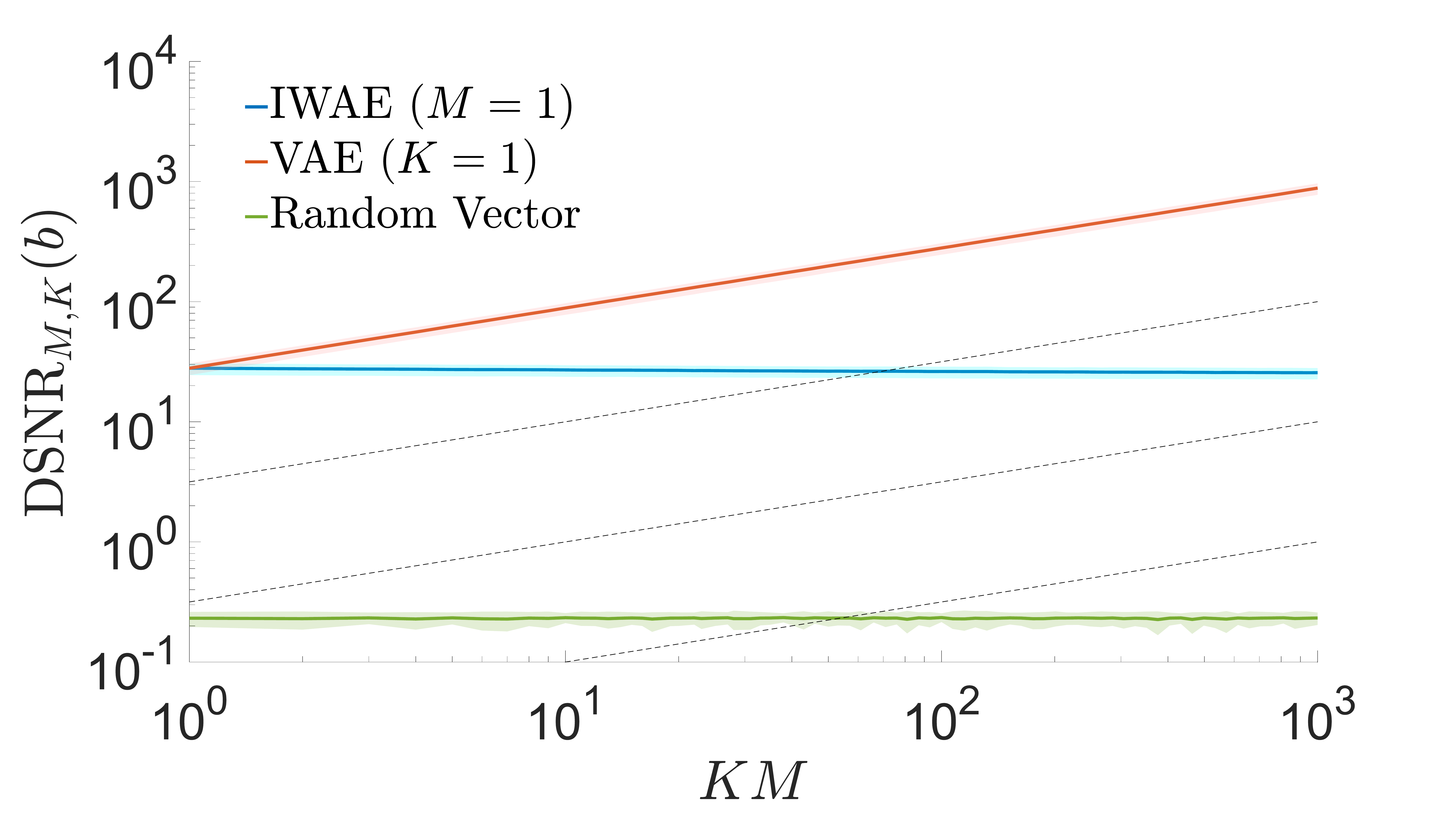}
		\caption{Convergence of \textsc{dsnr} for inference network\label{fig:hv/snr_dir}}
	\end{subfigure}~~~~~~~~~~
	\begin{subfigure}[b]{0.45\textwidth}
		\centering
		\includegraphics[width=\textwidth]{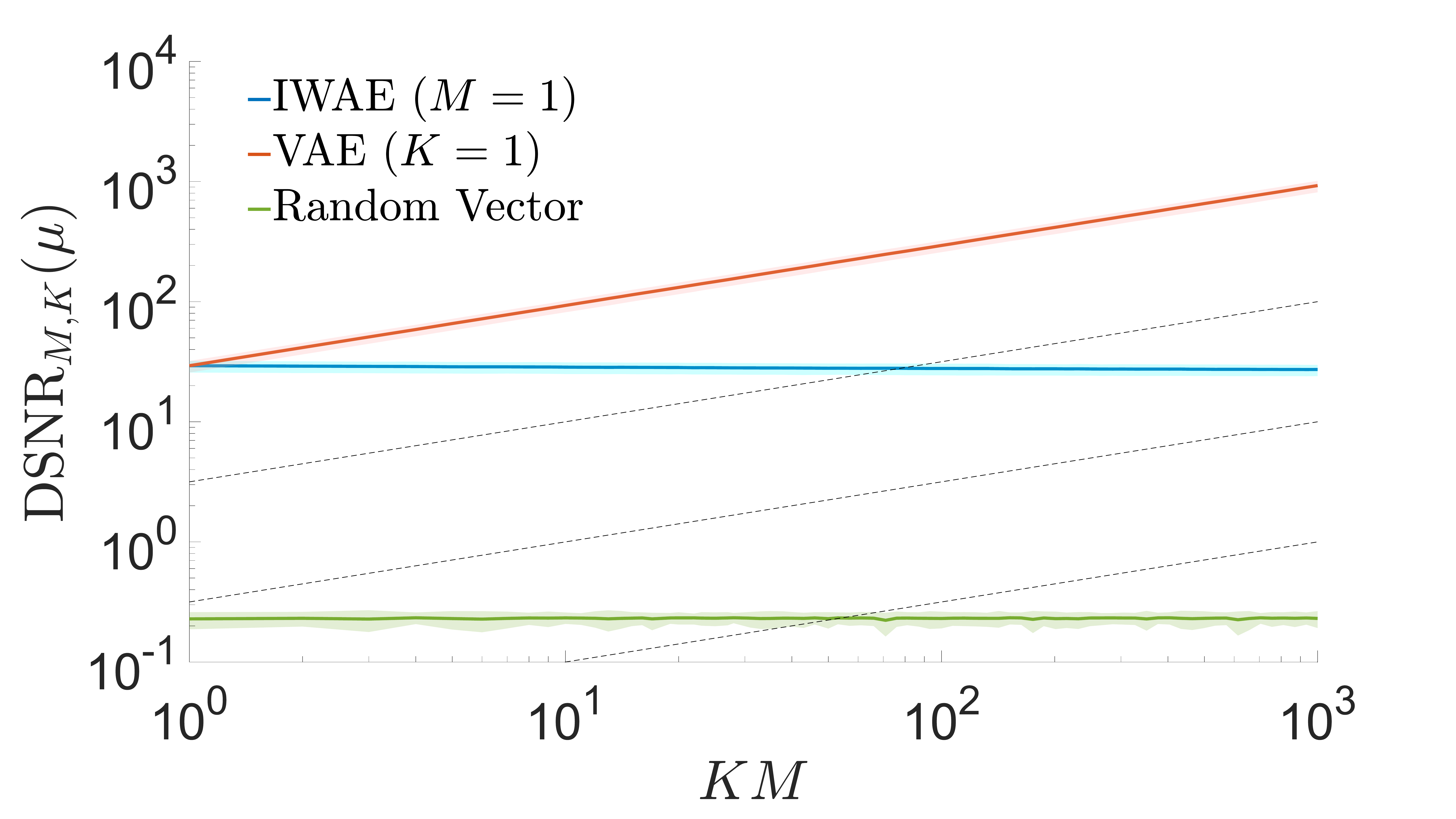}
		\caption{Convergence of \textsc{dsnr} for generative network\label{fig:hv/snr_dir_mu}}
	\end{subfigure}
	\caption{Convergence of directional signal-to-noise ratio of gradients estimates 
		as per Figure~\ref{fig:snr/extra}.
		\label{fig:hv/extra_end}}
\end{figure}

\begin{figure}[h]
	\centering
	\begin{subfigure}[b]{0.45\textwidth}
		\centering
		\includegraphics[width=\textwidth]{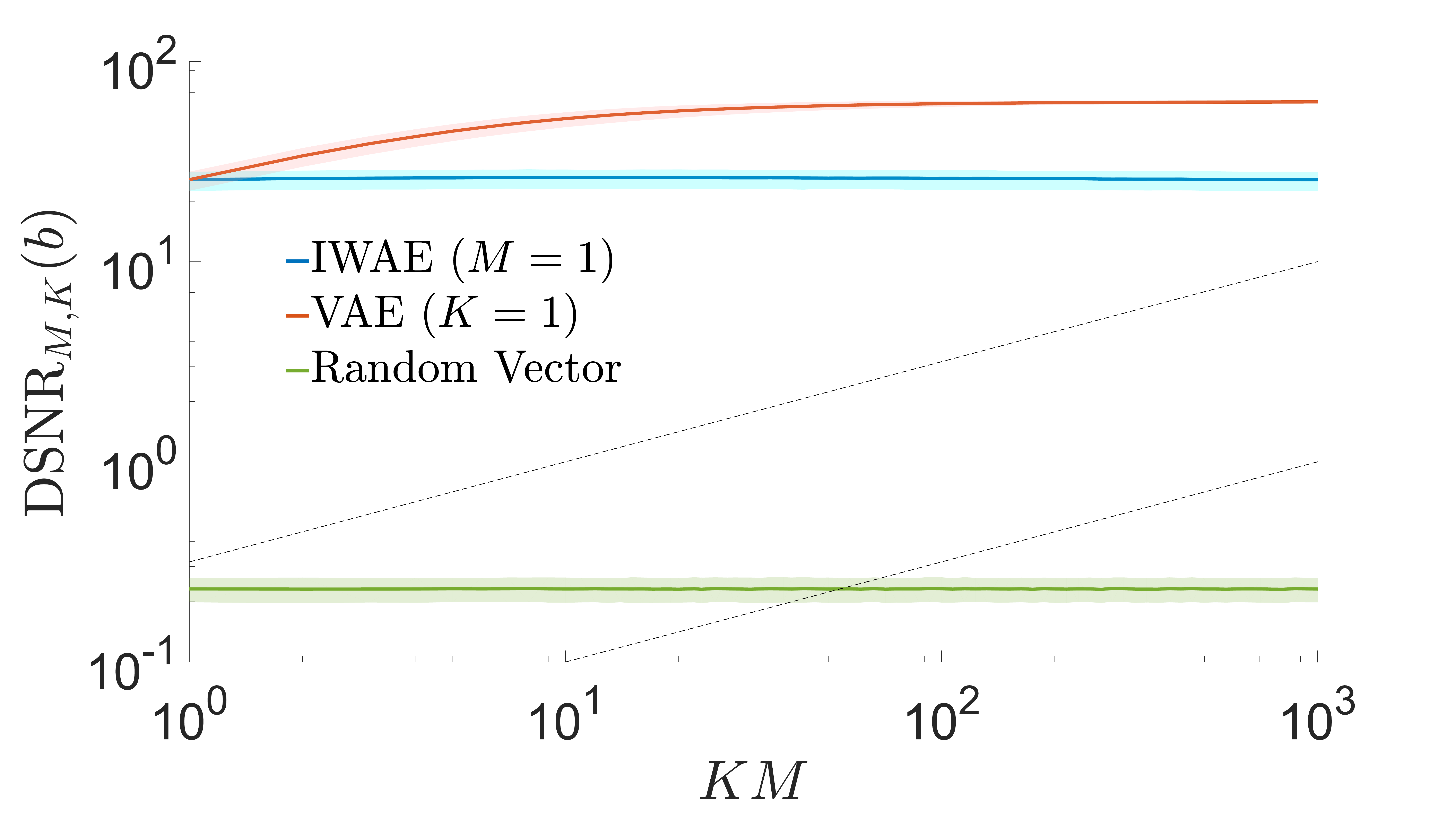}
		\caption{Convergence of \textsc{dsnr} for inference network\label{fig:snr/hv_snr_dir_end}}
	\end{subfigure} ~~~~~~~~~~
	\begin{subfigure}[b]{0.45\textwidth}
		\centering
		\includegraphics[width=\textwidth]{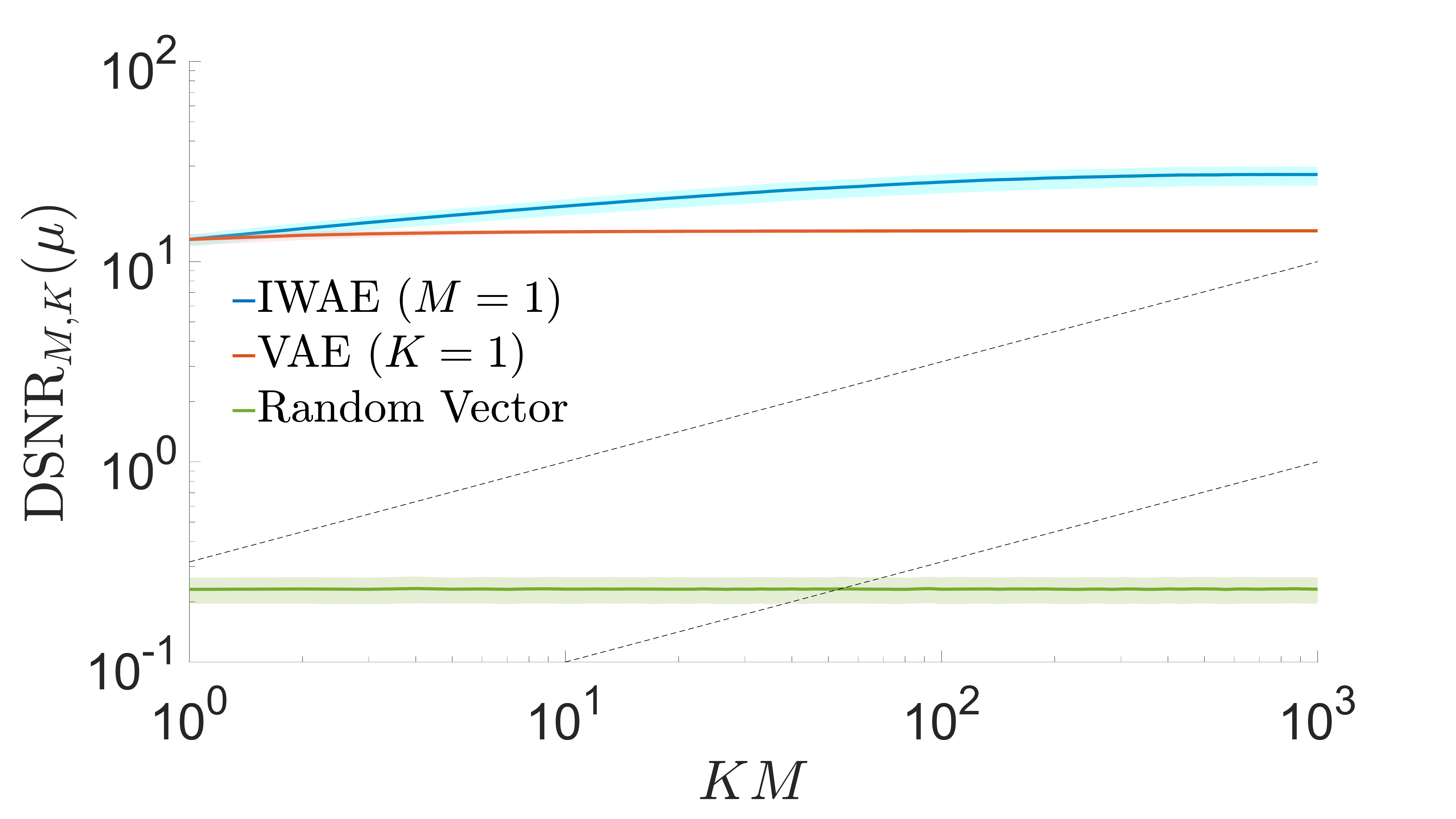}
		\caption{Convergence of \textsc{dsnr} for generative network\label{fig:snr/hv_snr_dir_mu_end}}
	\end{subfigure}
	\caption{Convergence of directional signal-to-noise ratio of gradient estimates where the
		true gradient is taken as $\E \left[\Delta_{1,1000}\right]$ as per
		Figure~\ref{fig:snr/extra_end}.
		\label{fig:snr/hv_extra_end}}
\end{figure}

\newpage

\section{Convergence of Deep Generative Model for Alternative Parameter Settings}
\label{sec:app:exp-algs}

Figure~\ref{fig-app:mnistexpt/convergence} shows the convergence of the introduced algorithms under different
settings to those shown in Figure~\ref{fig:mnistexpt/convergence}. Namely we consider $M=4, K=16$ for~\gls{PIWAE} and~\gls{MIWAE} 
and $\beta = 0.05$ for~\gls{CIWAE}.  These settings all represent tighter bounds than those of the main paper.
Similar behavior is seen in terms of the~\gls{IWAE}-64 metric for all algorithms.  \gls{PIWAE} produced
similar mean behavior for all metrics, though the variance was noticeably increased for $\log \hat{p}(x)$.
For~\gls{CIWAE} and~\gls{MIWAE}, we see that the parameter settings represent an explicit trade-off between
the generative network and the inference network:  $\log \hat{p}(x)$ was noticeably increased for both, matching
that of~\gls{IWAE}, while $-\textsc{KL}(Q_{\phi}(z \given x) || P_{\theta}(z \given x))$ was reduced.
Critically, we see here that, as observed for~\gls{PIWAE} in the main paper,~\gls{MIWAE} and~\gls{CIWAE} are able to
match the generative model performance of~\gls{IWAE} whilst improving the KL metric, indicating that they have learned
better inference networks.

\begin{figure*}[h]
	\centering
   	\begin{subfigure}[b]{0.33\textwidth}
        \centering
        \includegraphics[width=\textwidth]{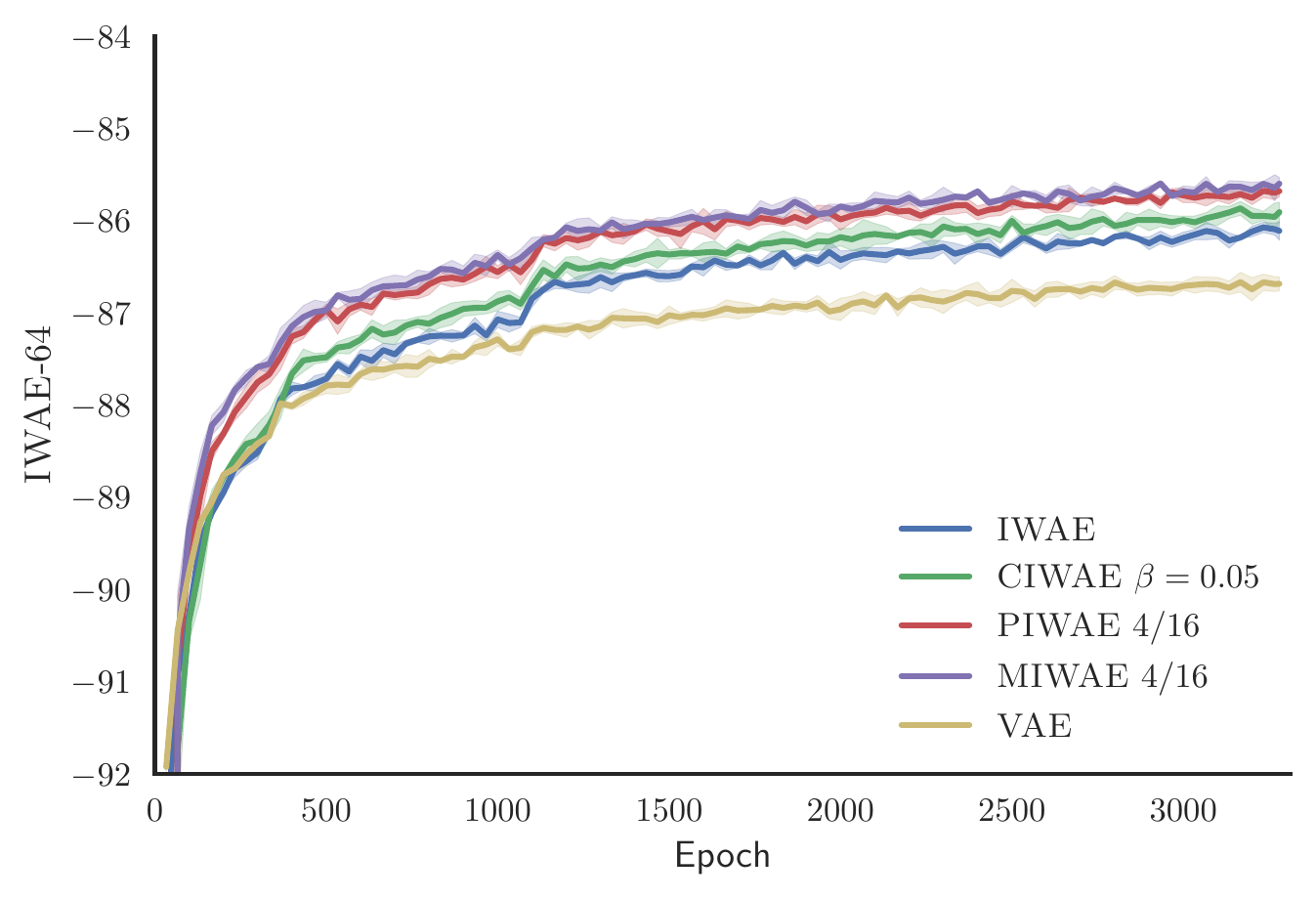}
        \caption{\textsc{IWAE}$_{64}$ \label{fig-app:mnistexpt/convergence/iwae64}}
    \end{subfigure}
	\begin{subfigure}[b]{0.33\textwidth}
		\centering
		\includegraphics[width=\textwidth]{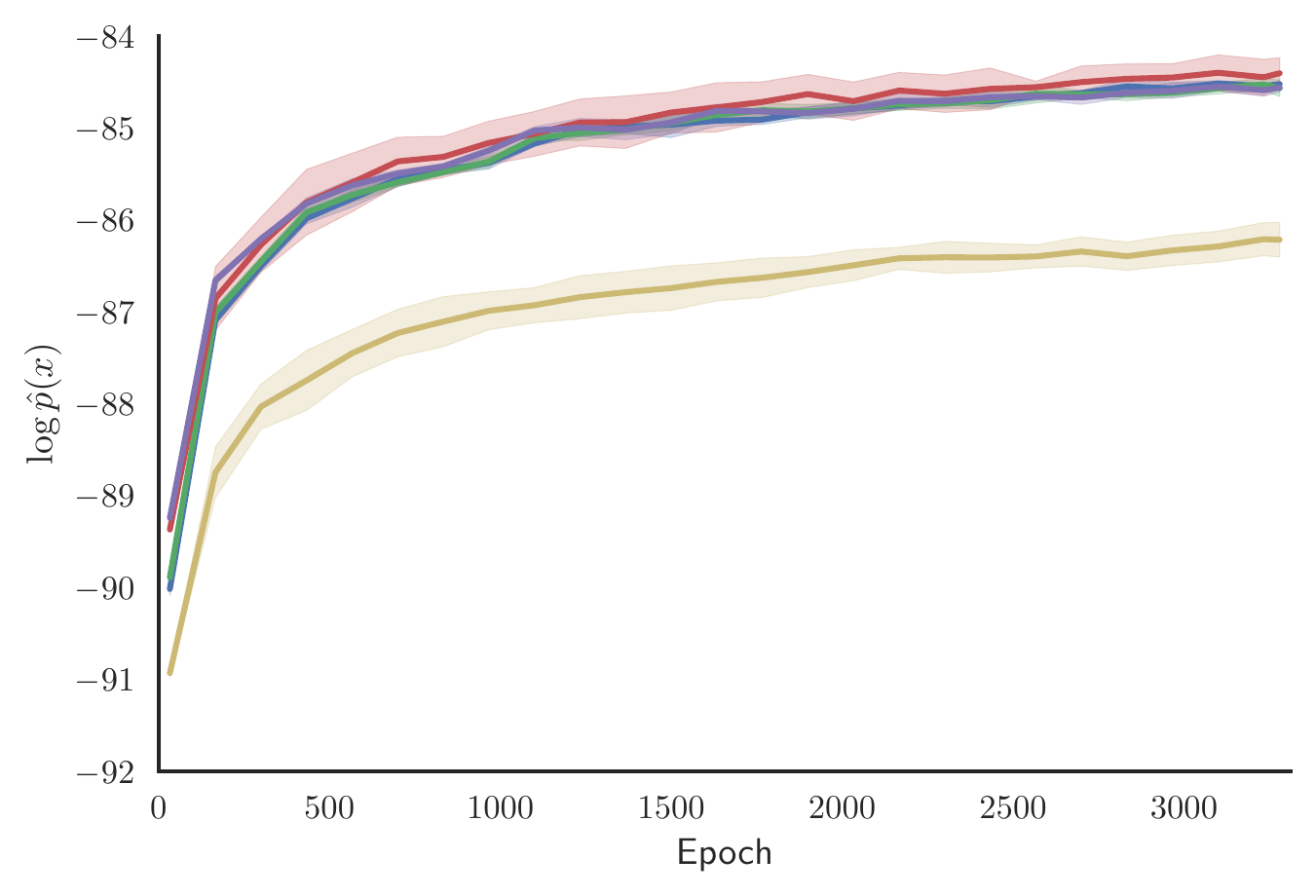}
		\caption{$\log \hat{p}(x)$ \label{fig-app:mnistexpt/convergence/logpx}}
	\end{subfigure}
	\begin{subfigure}[b]{0.33\textwidth}
		\centering
		\includegraphics[width=\textwidth]{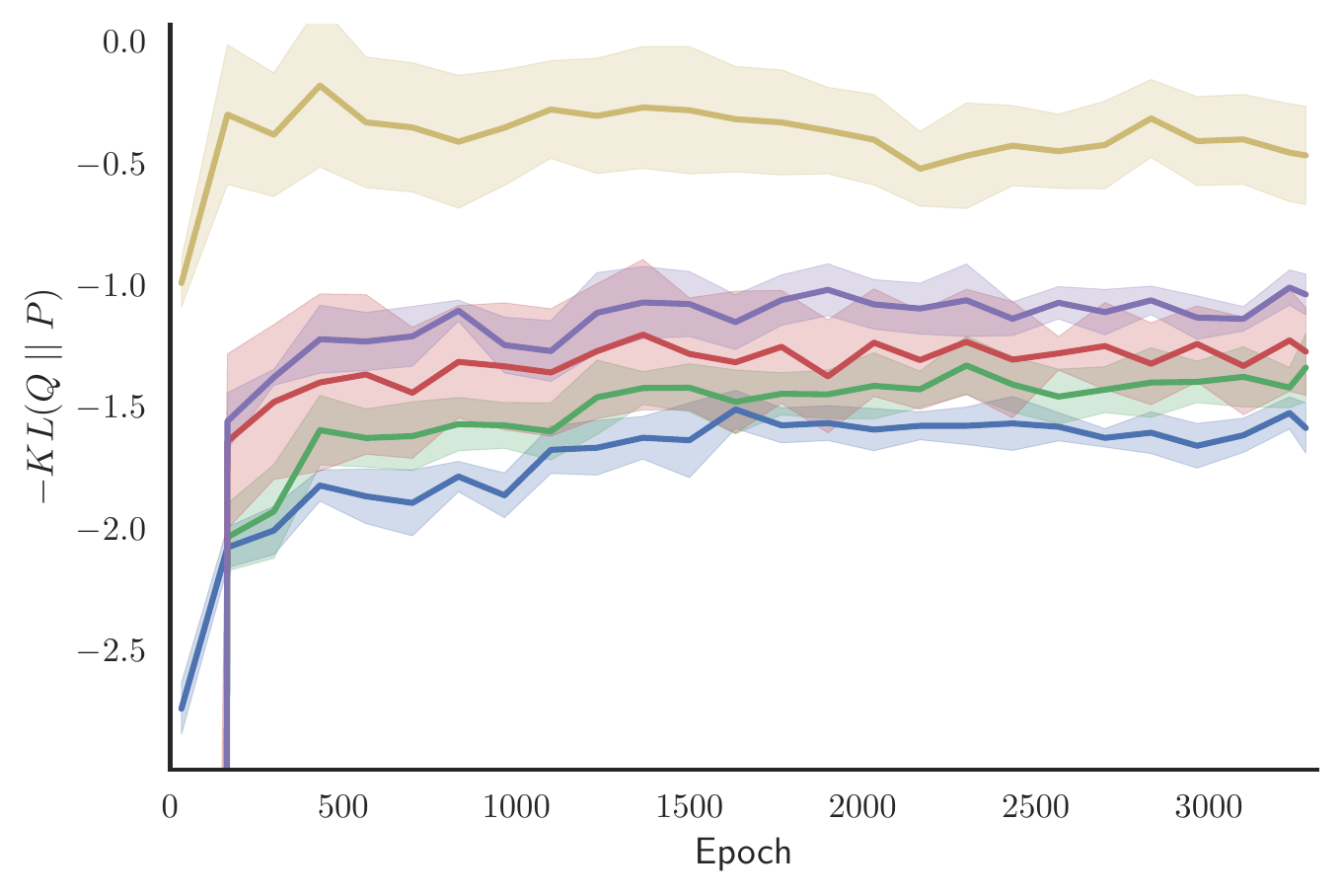}
		\caption{$-\mathrm{KL}(Q_{\phi}(z \given x) || P_{\theta}(z \given x))$ \label{fig-app:mnistexpt/convergence/kl}}
	\end{subfigure}
	\caption{Convergence of different evaluation metrics for each method.  Plotting conventions as per
		 Figure~\ref{fig:mnistexpt/convergence}.
		\vspace{-12pt}  \label{fig-app:mnistexpt/convergence}}
\end{figure*}

\section{Convergence of Toy Gaussian Problem}
\label{sec:app:toy-Gauss}

We finish by assessing the effect of the outlined changes in the quality
of the gradient estimates on the final optimization for our toy Gaussian problem.  Figure~\ref{fig:snr/hd_gaussian}
shows the convergence of running Adam~\citep{kingma2014adam} to optimize $\mu$, $A$, 
and $b$.  This suggests that the effects observed predominantly transfer to the overall
optimization problem.  Interestingly, setting $K=1$ and $M=1000$ gave the best performance
on learning not only the inference network parameters, but also the generative network
parameters.
\begin{figure*}[h]
	\includegraphics[width=\textwidth]{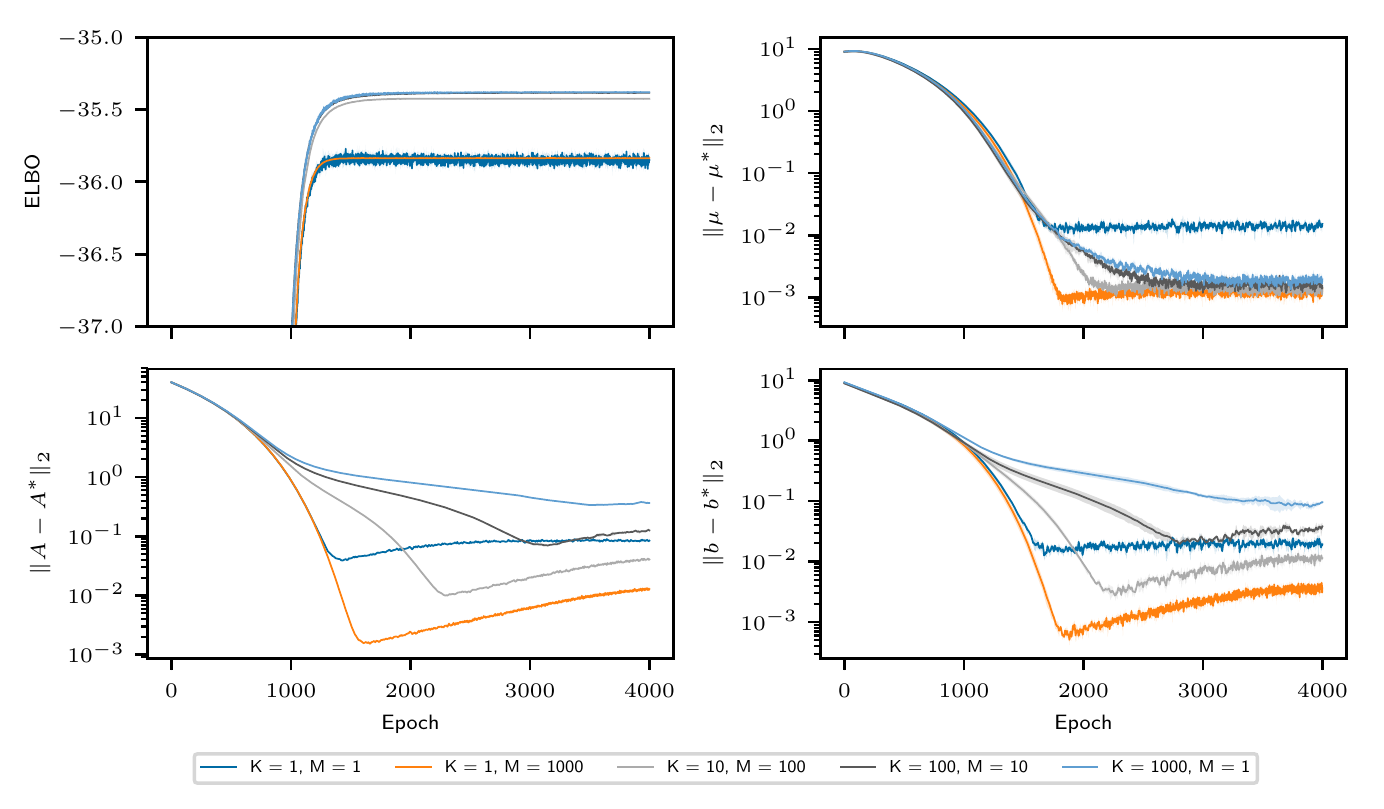}
	\caption{Convergence of optimization for different values of $K$ and $M$. 
		\emph{(Top, left)} $\ELBO_{\text{IS}}$ during training
		(note this represents a different metric for different $K$). \emph{(Top, right)} $L_2$ distance of the generative network parameters from the true maximizer. \emph{(Bottom)} $L_2$ distance of the inference network parameters from the true maximizer. Plots show means over $3$ repeats with $\pm 1$ standard deviation. Optimization is performed using the Adam algorithm with all parameters initialized by sampling from the uniform distribution on $[1.5, 2.5]$.}
	\label{fig:snr/hd_gaussian}
\end{figure*}

\section*{Acknowledgments}

TR and YWT are supported in part by the European Research Council under the European Union's Seventh Framework Programme (FP7/2007--2013) / ERC grant agreement no. 617071. 
TAL is supported by a Google studentship, project code DF6700.
MI is supported by the UK EPSRC CDT in Autonomous Intelligent Machines
and Systems.
CJM is funded by a DeepMind Scholarship.
FW is supported under DARPA PPAML through the U.S. AFRL
under Cooperative Agreement FA8750-14-2-0006, Sub Award number 61160290-111668.

\clearpage
\bibliography{refs}

%
%

\end{document}